\documentclass[jair,twoside,11pt,theapa]{article}
\usepackage{jair, rawfonts}

\jairheading{77}{2023}{1591-1636}{12/2022}{08/2023}
\ShortHeadings{Optimality Guarantees for Particle Belief Approximation of POMDPs}
{Lim, Becker, Kochenderfer, Tomlin, \& Sunberg}
\firstpageno{1591}

\usepackage{graphics} 
\usepackage{epsfig} 
\usepackage{mathptmx} 
\usepackage{times} 
\usepackage{amsmath} 
\usepackage{amssymb}  

\usepackage{algorithm}
\usepackage{algorithmicx}
\usepackage{booktabs}
\usepackage{color}
\usepackage{todonotes}
\usepackage{multicol}
\usepackage{commath}
\usepackage{enumitem}
\usepackage{amsthm}
\usepackage{subfigure}
\usepackage[capitalize]{cleveref}

\usepackage{tabularx}
\usepackage{xcolor}
\usepackage{pgfplots}
\pgfplotsset{compat = newest}
\usepackage{wrapfig}
\usepackage{mathtools}

\usetikzlibrary{arrows.meta}
\usetikzlibrary{backgrounds}
\usepgfplotslibrary{patchplots}
\usepgfplotslibrary{fillbetween}

\newtheorem{theorem}{Theorem}
\newtheorem{corollary}{Corollary}
\newtheorem{lemma}{Lemma}
\newtheorem{definition}{Definition}
\usepackage[noend]{algpseudocode}
\newcommand{\nc}[2]{\newcommand{#1}{#2}}
\newcommand{\rnc}[2]{\renewcommand{#1}{#2}}

\newcommand{\mb}[1]{\mathbf#1}

\newcommand{\undb}[2]{\underbrace{#1}_\text{#2}}
\newcommand{\lrp}[1]{\left(#1\right)}
\newcommand{\lrs}[1]{\left[#1\right]}
\newcommand{\lrc}[1]{\left\{#1\right\}}

\DeclareMathOperator{\E}{\mathbb{E}}
\DeclareMathOperator*{\argmax}{arg\,max}

\nc{\fa}{\forall}\nc{\te}{\exists}\nc{\pa}{\partial}\nc{\inft}{\infty}\nc{\imp}{\implies}\nc{\ds}{\displaystyle}\nc{\spc}{\text{\,\,\,\,}}\nc{\Spc}{\,\,\,\,\,\,\,\,}

\rnc{\l}{\ell }\nc{\R}{\mathbb{R} }\nc{\Ha}{\mathcal{H} }\nc{\La}{\mathcal{L} }\nc{\nb}{\vc{\nabla}}\nc{\tm}{\times}\nc{\Z}{\mathbb{Z}}\nc{\arrow}{\longrightarrow}\rnc{\mod}{\text{mod }}\nc{\kb}{(\SI{1.380E-23}{J K^{-1}})}\nc{\hb}{(\SI{1.055E-34}{J s})}\rnc{\c}{(\SI{3.00E8}{m s^{-1}})}\nc{\sq}{\Box}\rnc{\dag}{\dagger}\nc{\Sol}{\textbf{Solution:}}\nc{\done}{\;\blacksquare}\nc{\ua}{\uparrow}\nc{\da}{\downarrow}\nc{\rank}{\text{rank}}\nc{\Tr}{\text{tr}}\nc{\iidsim}{\stackrel{i.i.d}{\sim}}\rnc{\P}{\mathbb{P}}

\newcommand{\result}[4]{
    $#1\pm#2$ & {\color{green!#3!red}\rule{#4cm}{8pt}}
}
\newcommand{\resultbest}[4]{
    $\mathbf{#1\pm#2}$ & {\color{green!#3!red}\rule{#4cm}{8pt}}
}
\newcommand{\resultbad}[2]{
    $\textcolor{gray}{#1\pm#2}$ & {\color{green!0!red}\rule{0cm}{8pt}}
}
\newcommand{\noresult}{ & }

\DeclareSymbolFont{cmsymbols}{OMS}{cmsy}{m}{n}
\SetSymbolFont{cmsymbols}{bold}{OMS}{cmsy}{b}{n}
\DeclareSymbolFontAlphabet{\mathcal}{cmsymbols}

\usepackage[backend=biber,style=apa,natbib=true,doi=false,url=false,uniquename=false]{biblatex}
\AtEveryBibitem{
\ifentrytype{inproceedings}{
 \clearlist{address}
 \clearlist{publisher}
 \clearname{editor}
 \clearlist{organization}
 \clearfield{url}  
 \clearfield{doi}  
 \clearfield{pages}  
 \clearlist{location}
 \clearlist{month}
 }{}
 
 \ifentrytype{article}{
 \clearlist{address}
 \clearlist{publisher}
 \clearname{editor}
 \clearlist{organization}
 \clearfield{url}  
 \clearfield{doi}  
 \clearlist{location}
 }{}
 }

\newcommand{\shortcite}[1]{\citep{#1}}

\renewcommand{\cite}[1]{\citep{#1}}

\begin{document}

\title{Optimality Guarantees for Particle Belief Approximation of POMDPs}

\author{\name Michael H. Lim \email michaelhlim@berkeley.edu \\
        \addr University of California, Berkeley,\\
        Electrical Engineering and Computer Sciences Department
        \AND
        \name Tyler J. Becker \email Tyler.Becker-1@colorado.edu \\
        \addr University of Colorado Boulder,\\
        Aerospace Engineering Sciences Department
        \AND
        \name Mykel J. Kochenderfer \email mykel@stanford.edu \\
        \addr Stanford University,\\
        Aeronautics and Astronautics Department
        \AND
        \name Claire J. Tomlin \email tomlin@eecs.berkeley.edu \\
        \addr University of California, Berkeley,\\
        Electrical Engineering and Computer Sciences Department
        \AND
        \name Zachary N. Sunberg \email zachary.sunberg@colorado.edu \\
        \addr University of Colorado Boulder,\\
        Aerospace Engineering Sciences Department
        }


\maketitle

\begin{abstract}
    Partially observable Markov decision processes (POMDPs) provide a flexible representation for real-world decision and control problems.
    However, POMDPs are notoriously difficult to solve, especially when the state and observation spaces are continuous or hybrid, which is often the case for physical systems.
    While recent online sampling-based POMDP algorithms that plan with observation likelihood weighting have shown practical effectiveness, a general theory characterizing the approximation error of the particle filtering techniques that these algorithms use has not previously been proposed.
    Our main contribution is bounding the error between any POMDP and its corresponding finite sample particle belief MDP (PB-MDP) approximation.
    This fundamental bridge between PB-MDPs and POMDPs allows us to adapt any sampling-based MDP algorithm to a POMDP by solving the corresponding particle belief MDP, thereby extending the convergence guarantees of the MDP algorithm to the POMDP.
    Practically, this is implemented by using the particle filter belief transition model as the generative model for the MDP solver.
    While this requires access to the observation density model from the POMDP, it only increases the transition sampling complexity of the MDP solver by a factor of $\mathcal{O}(C)$, where $C$ is the number of particles.
    Thus, when combined with sparse sampling MDP algorithms, this approach can yield algorithms for POMDPs that have no direct theoretical dependence on the size of the state and observation spaces.
    In addition to our theoretical contribution, we perform five numerical experiments on benchmark POMDPs to demonstrate that a simple MDP algorithm adapted using PB-MDP approximation, Sparse-PFT, achieves performance competitive with other leading continuous observation POMDP solvers.
\end{abstract}


\section{Introduction}
\label{sec:introduction}
Maintaining safety and acting efficiently in the midst of uncertainty is an important aspect in a diverse set of challenges from transportation~\cite{holland2013optimizing,sunberg2017value} to autonomous scientific exploration \cite{bresina2002planning,frew2020field}, to healthcare~\cite{ayer2012mammography} and ecology~\cite{memarzadeh2018adaptive}.
The partially observable Markov decision process (POMDP) is a flexible framework for sequential decision making in uncertain environments.

One common method for solving POMDPs is online tree search, which is attractive for several reasons.
First, the approach scales to very large problems because it uses sampled trajectories, making it insensitive to the dimensionality of the state and observation spaces \cite{kearns2002sparse}.
Second, since online computation focuses on the current states and states likely to be encountered in the future, it can reduce the need for offline computation and end-to-end training \cite{deglurkar2023compositional}.
Third, tree search is applicable to a wide range of problems, for example hybrid continuous-discrete and problems with many local optima, because it only depends on a minimal set of problem structure requirements.

Recently proposed POMDP tree search algorithms \cite{sunberg2018pomcpow,garg2019despotalpha,lim2020sparse,lim2021voronoicdc,mern2021bayesian,hoerger2020,wu2021adaptive} have been shown empirically to work on continuous state and observation spaces.
Theoretical analysis, however, has lagged behind.
While there are algorithms that have performance guarantees \cite{lim2020sparse,lim2021voronoicdc} and algorithms that perform well empirically~\cite{sunberg2018pomcpow,garg2019despotalpha,lim2021voronoicdc,mern2021bayesian,hoerger2020,wu2021adaptive}, there has been little progress on a general theory describing why this family of algorithms can enjoys such good performance.
Though there have been some algorithm-specific results (outlined in \cref{sec:previous-analysis}), a considerable gap in the connection between POMDPs and practical approximations using particle methods still remains.

This manuscript formally justifies that optimality guarantees in a finite sample particle belief MDP (PB-MDP) approximation of a POMDP/belief MDP yield optimality guarantees in the original POMDP as well.
We accomplish this by showing that the $Q$-values of the POMDP and PB-MDP are close with high probability by using an intermediary theoretical algorithm called Sparse Sampling-$\omega$.
Specifically, we prove that the Sparse Sampling-$\omega$ $Q$-value estimates are close to both optimal $Q$-values of the POMDP and PB-MDP with high probability.
Since there exists an algorithm that approximates both $Q$-values accurately with high probability, the optimal $Q$-values of the POMDP and PB-MDP themselves must be close to each other with high probability.
This probability scales as $1-\mathcal{O}(C^{D}\exp(-t \cdot C))$, where $C$ is the number of particles; $D$ is the planning depth; and $t$ is a number determined by the POMDP reward function and probability distributions, number of particles, and desired accuracy. 
Notably, this convergence rate does not directly depend on the size of the state space nor the observation space, but rather depends on the R\'enyi divergence that links the probabilities concerning state and observation trajectories and the planning horizon $D$.

This fundamental bridge between PB-MDPs and POMDPs allows us to adapt any sampling-based MDP algorithm of choice to a POMDP by solving the corresponding particle belief MDP approximation and to preserve the convergence guarantees in the POMDP.
Practically, this means additionally assuming we have an explicit observation model $\mathcal{Z}$ and swapping out the state transition generative model with a particle filtering-based model.
This change only increases the computational complexity of transition generation by a factor of $\mathcal{O}(C)$, with $C$ the number of particles in a particle belief state.
This allows us to devise algorithms such as Sparse Particle Filter Tree (Sparse-PFT), which enjoys algorithmic simplicity, theoretical guarantees, and practicality, since it is equivalent to upper confidence trees (UCT) \cite{bjarnason2009lower,shah2020nonasymptotic}, with particle belief states.

The remainder of this paper proceeds as follows:
First, \cref{sec:background} reviews preliminary definitions and previous related work.
\cref{sec:pb-mdp-def} formalizes the notion of particle belief MDPs (PB-MDPs).
Then, \cref{sec:ssw} introduces Sparse Sampling-$\omega$ algorithm, and proves its coupled convergence towards the optimal $Q$-values of a POMDP and its corresponding PB-MDP in \cref{thm:ssw}.
\cref{sec:pomdp-guarantees} formally bridges the gap between POMDPs and PB-MDPs by leveraging the coupled convergence of Sparse Sampling-$\omega$.
In this section, we present two main theorems: \cref{thm:pomdp-qvalue} shows the optimal $Q$-value bounds between POMDP and PB-MDP, and \cref{thm:pomdp-grand} shows the near-optimality of planning with a PB-MDP to solve a POMDP by applying an online $Q$-value-estimating algorithm repeatedly in a closed loop with observations from the environment.
We also introduce Sparse-PFT, a practical example of generating a PB-MDP approximation algorithm from an MDP algorithm.
Finally, \cref{sec:numerical} empirically shows the performance of Sparse-PFT and other practical continuous observation POMDP algorithms over five different simulation experiments, and validates the improvements in performance of PB-MDP approximation with the increase in number of particles $C$ while keeping other hyperparameters fixed.

\section{Background and Related Work}
\label{sec:background}
\subsection{POMDPs}
\begin{figure}
    \centering
    \includegraphics[width=\textwidth]{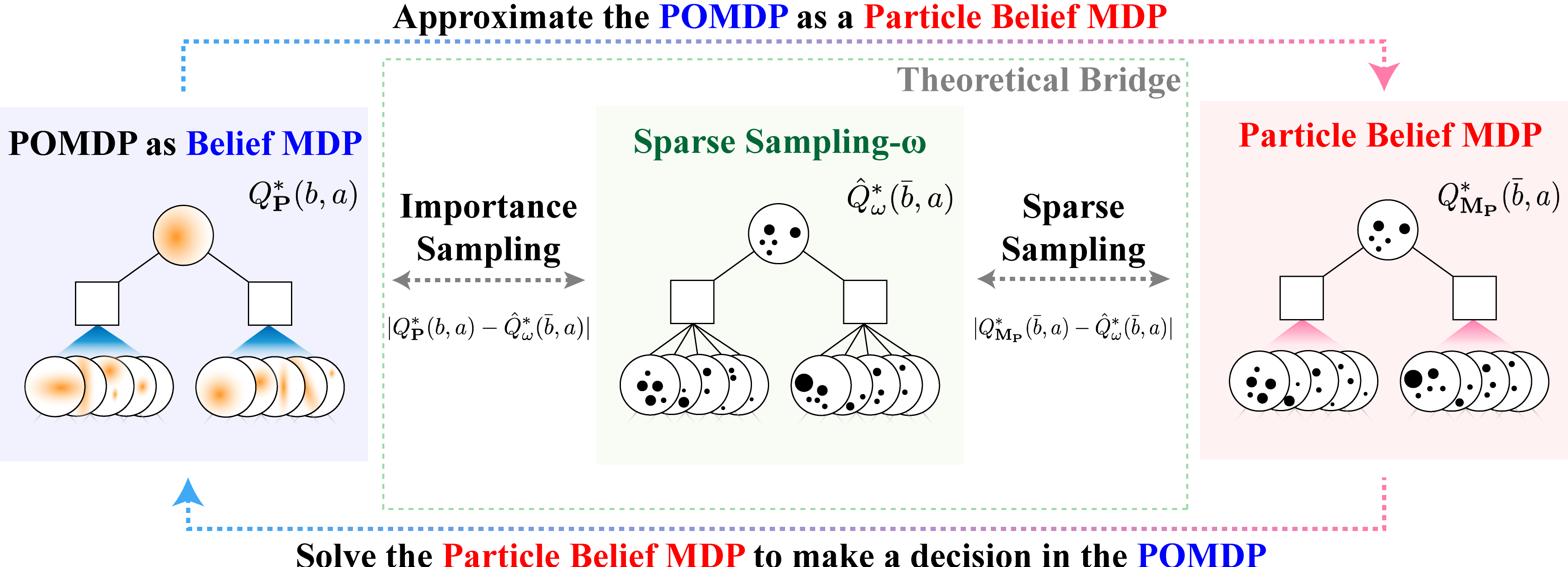}\label{fig:pomdp-pbmdp-bridge}
    \caption{Illustration of the proof of our main theorem, \cref{thm:pomdp-qvalue}: Since Sparse Sampling-$\omega$ algorithm $Q$-value estimator converges to both the optimal $Q$-values of POMDP and PB-MDP, such an existence of algorithm implies that the optimal $Q$-values of POMDP and PB-MDP are also close to each other with high probability.
    This enables us to approximate the POMDP problem as a PB-MDP, and then solve the PB-MDP with an MDP algorithm to make a decision in the original POMDP while retaining the guarantees and computational efficiencies of the original MDP algorithm.}
\end{figure}

The partially observable Markov decision process (POMDP) is a mathematical formalism that can represent a wide range of sequential decision making problems~\cite{kochenderfer2015decision}.
In a POMDP, an agent chooses actions based on observations to maximize the expectation of a cumulative reward signal.
A POMDP is defined by the 7-tuple $(S, A, O, \mathcal{T}, \mathcal{Z}, R, \gamma)$.
In this tuple, $S$, $A$, and $O$ are sets of all possible states, actions, and observations, respectively.
These sets can be discrete, e.g.\ $\{1,2\}$, continuous, e.g. $\mathbb{R}^2$, or hybrid.
The conditional probability distributions $\mathcal{T}$ and $\mathcal{Z}$ define state transitions and observation emissions, respectively.
The transition probability distribution is conditioned on the current state $s$ and action $a$ and is denoted $\mathcal{T}(s' \mid s, a)$.
The observation probability is conditioned on the previous action and current state\footnote{It is possible to use observation probability distributions additionally conditioned on the previous state, $\mathcal{Z}(o \mid s, a, s')$, in all algorithms discussed in this paper, but we use $\mathcal{Z}(o \mid a, s')$ for brevity.} and denoted $\mathcal{Z}(o \mid a, s')$.
The reward function, $R(s, a)$, maps states and actions to an expected reward, and $\gamma \in [0,1)$ is a discount factor.
The agent plans starting from $b_0$, the initial state distribution or the initial belief.
Some POMDP algorithms only require samples from the transition, observation, or reward models rather than explicit knowledge of $\mathcal{T}$, $\mathcal{Z}$, or $R$.
Such samples can be produced using a so-called \emph{generative model}~\cite{kearns2002sparse} denoted with $s', o, r \gets G(s, a)$. In some algorithms, only one or two of the outputs of $G$ are used and the others are discarded, e.g.\ the notation ``$o \gets G(s, a)$'' indicates that $s'$ and $r$ are discarded.

The objective of a POMDP is to find an optimal policy, $\pi^*$, that selects actions that maximizes the discounted sum of future rewards, with an appropriate tie-breaking method:
\begin{equation}\label{eq:objective}
\pi^* = \underset{\pi}{\text{argmax}} \, \E \left[\sum_{t=0}^\infty \gamma^t R(s_t, a_t)\right] \text{.}
\end{equation}
In general, the actions may be chosen based on the entire history of actions and observations,
\begin{equation}
h_t \equiv (b_0, a_0, o_1, a_1, \ldots, o_{t-1}, a_{t-1}, o_t) \text{.}
\end{equation}
However, because of the Markov property, it can be shown that optimal decisions can be made based only on the conditional distribution of the state given the history \cite{kaelbling1998planning}, known as the belief,
\begin{equation}
    b_t(s) \equiv \P(s_t=s \mid h_t) \text{.}
\end{equation}
This belief can be updated using Bayes's rule or an efficient approximation such as a Kalman filter or particle filter, and it is often more straightforward to determine actions based on beliefs rather than the history.
Since the belief and history fulfill the Markov property, a POMDP is a Markov decision process (MDP) on the belief or history space, commonly referred to as the belief MDP \cite{kaelbling1998planning}.

In order to maximize the objective in \cref{eq:objective}, the policy must take into account both the immediate reward from taking the action in the current state and whether that action will lead to states favorable for attaining rewards in the future.
For a history $h$ and a corresponding belief $b$, the history-action and belief-action value functions, defined as
\begin{equation}
    Q^\pi(h, a) \equiv Q^\pi(b, a) \equiv \E \left[\sum_{t=0}^\infty \gamma^t R(s_t, a_t) \, \middle| \, b_0=b, \, a_0=a,\, a_t = \pi(b_t)\right] \text{,}
\end{equation}
take both of these factors into account.
When $\pi$ is also used for the current step, the expected accumulated reward from a history or belief is denoted with $V^\pi(h) = V^\pi(b) = Q^\pi(b, \pi(b))$.
When $\pi$ is an optimal policy, these value functions are denoted with $Q^*$ and $V^*$.
If $Q^*$ can be calculated, an optimal policy $\pi^*$ can simply be extracted with $\pi^*(h) = \text{argmax}_a\,Q^*(h, a)$.
Thus, a common strategy for solving POMDPs involves iteratively improving estimates of $Q^*$ denoted simply with $Q$ for brevity.

Early research~\cite{smallwood1973optimal,kaelbling1998planning,shani2013survey} sought to find optimal solutions to POMDPs \emph{offline}; that is, they attempted to optimize actions for every possible belief before interacting with the environment.
However, since POMDPs are generally intractable~\cite{papadimitriou1987complexity}, it is often impossible to find a complete solution for a POMDP offline.
Instead, we seek to compute solutions online only for the part of the problem that may be reached in the immediate future.

\subsection{Importance Sampling and Particle Filtering}

In many real-world applications, updating the belief exactly based on a new action and observation is impractical.
Fortunately, Monte Carlo methods provide simple and effective tools for approximate reasoning about distributions such as beliefs.

We often need to reason about a random variable $X \sim \mathcal{P}$ based only on samples from another related random variable, $Y \sim \mathcal{Q}$. 
\emph{Importance sampling} allows us to, among other tasks, calculate the expectation of a function by observing that
\begin{equation}
    \E_{X \sim \mathcal{P}}[f(X)] = \int f(x) \mathcal{P}(x) dx = \int f(x) \frac{\mathcal{P}(x)}{\mathcal{Q}(x)} \mathcal{Q}(x) dx \approx \frac{1}{N} \sum_{i=1}^N \frac{\mathcal{P}(y_i)}{\mathcal{Q}(y_i)} f(y_i) \text{,}
\end{equation}
where $\{y_i\}_{i=1}^N$ are samples from distribution $\mathcal{Q}$.
The convergence property of this approximation relevant to the present work is described formally in \cref{sec:is}.

The particle filter is an application of Monte Carlo estimation to the task of Bayesian belief updating~\cite{thrun2005probabilistic,kochenderfer2015decision}.
The simplest form is an unweighted particle filter, in which the belief is represented by a collection of $N$ states, known as \textit{particles}, $\tilde{b} = \{s_i\}_{i=1}^N$.
The density is approximated by $\tilde{b}(s) \approx \sum_{i=1}^N \delta(s_i = s)$, where $\delta(\cdot)$ is a Dirac or Kronecker delta function depending on the form of the state space.
At each step of the POMDP, after an action $a$ is taken and an observation $o$ is received, a new state $s'_i$ and observation $o_i$ is simulated once or more for each of the particles to create the new belief, $\tilde{b}' = \{s'_i : o_i=o\}$.
In most cases, few particles will match $o$ so it is difficult to maintain a large number of particles in the belief.
Various domain-specific techniques can be used to reduce this problem, but it is difficult to solve completely in the unweighted particle filter.
 
The weighted particle filter is usually much more effective.
The belief is represented by a collection of state particles and corresponding weights, 
$\tilde{b} = \{(s_i, w_i)\}_{i=1}^N$.
The density is approximated with $\tilde{b}(s) \approx \frac{\sum_{i=1}^N w_i \delta(s_i = s)}{\sum_{i=1}^N w_i}$.
A belief update consists of simulating each particle once or more and then calculating the new weight according to the importance sampling correction: $w'_i = w_i\cdot Z(o \mid s, a, s'_i)$.
Typically, the weights of a few particles grow much larger than the others, so a resampling step creates many particles from those with large weights and eliminates those with very small weights~\cite{thrun2005probabilistic,kochenderfer2015decision}.

\subsection{Online POMDP Solvers}
Monte Carlo tree search (MCTS) is a common solution technique for Games, MDPs, and POMDPs \cite{browne2012survey,silver2010pomcp}.
In an MDP context, MCTS constructs a tree consisting of state and state-action nodes.
In a POMDP context, each node corresponds to an action- or observation-terminated history node, estimating $Q(h, a)$ at each action-terminated history node.
The most common variant is called partially observable upper confidence trees (PO-UCT) or partially observable Monte Carlo planning (POMCP)\footnote{Strictly speaking, POMCP also includes a specialized unweighted particle filter update that re-uses simulations from the planning step, but the term is often used informally as a synonym for PO-UCT.}~\cite{silver2010pomcp} and constructs the tree by using Upper Confidence Bound (UCB), asymmetrically favoring regions of the history and action spaces that are likely to be visited when the optimal policy is executed~\cite{kocsis2006bandit,bjarnason2009lower,shah2020nonasymptotic}.

In addition to the PO-UCT algorithm described above, there are several other approaches to solve POMDPs through online planning.
Early solvers attempted to use exact Bayesian belief updates on discrete state spaces~\cite{ross2008online}, however, these are much less scalable than PO-UCT.
Two other popular solvers with scalability similar to UCT are determinized sparse partially observable trees (DESPOT)~\cite{ye2017despot} and adaptive belief trees (ABT)~\cite{kurniawati2016online}.
DESPOT uses a small number of determinized scenarios instead of independent random simulations to reduce variance and relies on heuristic tree search guided by upper and lower bounds rather than Monte Carlo tree search.
ABT is designed to efficiently adapt to changes in the environment without discarding previous computation.

Since PO-UCT, DESPOT, and ABT all rely on unweighted particle belief representations, they will fail to find optimal policies in continuous observation spaces because the probability of generating the same observation twice, and hence creating beliefs with multiple particles, is zero~\cite{sunberg2018pomcpow,lim2020sparse}.
Partially observable Monte Carlo planning with observation widening (POMCPOW) approaches the continuous observation challenge by introducing a weighted particle filter and the continuous action challenge with progressive widening~\cite{sunberg2018pomcpow}.
DESPOT-$\alpha$~\cite{garg2019despotalpha} incorporates a similar weighting scheme and uses the $\alpha$-vector concept to generalize value estimates between sibling nodes.
Adaptive online packing-guided search (AdaOPS) \cite{wu2021adaptive} fuses similar observation branches in the search tree to improve performance.
Lazy Belief Extraction for Continuous Observation POMDPs (LABECOP) \cite{hoerger2020} builds the planning tree by re-weighting particles and extracting belief sequence values efficiently.

\subsection{Theoretical Analysis of Particle-based POMDP Algorithms} \label{sec:previous-analysis}

Several previous studies have analyzed particle-based POMDP algorithms from a theoretical perspective.
\citet{silver2010pomcp} claim that POMCP value estimates converge to the optimal value for discrete POMDPs on the basis that it equivalent to applying UCT to the history MDP corresponding to the POMDP.
However, as mentioned above, POMCP does not converge in continuous observation POMDPs~\cite{sunberg2018pomcpow,lim2020sparse}.
\citet{ye2017despot} analyzed the approximation of a POMDP with a finite set of \emph{scenarios}, which essentially correspond to random seeds that are fixed across different possible action sequences, and bounded the performance of the DESPOT algorithm that uses these scenarios.
However, these bounds depend on the size of the observation space, $|O|$, and thus cannot be applied to continuous observation spaces.
This analysis was expanded by \citet{luo2019importance} to cover a case in which scenarios are selected from an importance distribution.

\citet{bai2014integrated} also provide convergence guarantees for their Monte Carlo value iteration (MCVI) algorithm which uses simulations in a manner somewhat akin to particle filtering.
Their analysis extends to continuous observation spaces, but the algorithm is best suited for offline use, unlike the algorithms we focus on.
According to \citet{bai2014integrated}, MCVI spends hours computing a policy graph that can be executed quickly online.

\citet{lim2020sparse} presented the first theoretical analysis of online POMDP tree search algorithms that use weighted particle filtering.
However, the partially observable weighted sparse sampling (POWSS) algorithm analyzed in that work is not efficient enough to be practically useful.
\citet{wu2021adaptive} provide analytical performance guarantees for a simplified version of AdaOPS, a recent particle belief tree search POMDP solver included in our numerical analysis in \cref{sec:numerical}.
However, the full AdaOPS algorithm used in the numerical experiments is more complex than the simplified version used in the theoretical portion of the work.

\citet{du2021particle} analyzed the number of particles needed to control partially observable linear systems.
Finally, there is a large body of work on particle filters without consideration of decision making.
Some results from this field are summarized by \citet{crisan2002survey}.

In contrast to these works that provide guarantees for individual algorithms or limited cases, the analysis in this paper provides a general bound for particle belief approximation of a broad class of POMDPs, giving justification for MDP algorithms to be adapted to solve POMDPs efficiently.

\section{Particle Belief MDPs (PB-MDPs)}\label{sec:pb-mdp-def}

In this section, we define the corresponding particle belief MDP (PB-MDP) for a given POMDP.
Deriving the corresponding particle belief MDP of a POMDP is equivalent to approximating the belief MDP with a finite number of particles.

\begin{definition}[Particle Belief MDP]
The corresponding particle belief MDP for a given POMDP problem $\mathbf{P} = (S, A, O, \mathcal{T}, \mathcal{Z}, R, \gamma)$ is the MDP $\mathbf{M}_{\mathbf{P}} = (\Sigma, A, \tau, \rho, \gamma)$ defined by the following elements:
\end{definition}
\begin{itemize}
\item $\Sigma$: State space over particle beliefs $\bar{b}_{d}$. An element in this set, $\bar{b}_d \in \Sigma$, is a particle collection, $\bar{b}_d = \{(s_{d,i},w_{d,i})\}_{i=1}^C$, where $s_{d,i}\in S$, $w_{d,i} \in \R^+$. \footnote{The $d$ subscript, the number of steps, is included for subscript order consistency with the rest of the paper, but is not meaningful in this context.} For the sake of brevity in the rest of the paper, we drop$\phantom{\}}_{i=1}^C$ and render a particle belief as $\{s_{d,i},w_{d,i}\}$.
      The beliefs are not assumed to be permutation invariant, meaning that particle beliefs with different particle orders are considered different elements in $\Sigma$. This simplifies derivation of the transition distribution (see \cref{eqn:pb-trans-dist}) because each particle transition is independent.
  \item $A$: Action space. Remains the same as the original action space.
  \item $\tau$: Transition density $\tau(\bar{b}_{d+1}\mid \bar{b}_{d},a)$: We define the likelihood weights $w_{d,i}$ of particles $s_{d,i}$ to be updated through unnormalized Bayes rule:
  \begin{align}
    w_{d+1,i} &= w_{d,i}\cdot \mathcal{Z}(o\mid a,s_{d+1,i}).
  \end{align}
  Then, the transition probability from $\bar{b}_{d}$ to $\bar{b}_{d+1}$ by taking the action $a$ can be defined as:
  \begin{align}
    \tau(\bar{b}_{d+1}\mid \bar{b}_{d},a) &\equiv \int_{O} \P(\bar{b}_{d+1}\mid \bar{b}_{d},a,o)\P(o\mid \bar{b}_{d},a)\, do.
  \end{align}
  The first term in the integrand product $\P(\bar{b}_{d+1}\mid \bar{b}_{d},a,o)$ is the conditional transition density given some observation $o$.
  Since each new state particle is generated independently and the likelihood weight updates are deterministic given $s_{d,i},s_{d+1,i}, a$ and $o$, this term can be written in terms of $\mathcal{T}$ and $\mathcal{Z}$:
  \begin{align}
    \P(\bar{b}_{d+1}\mid \bar{b}_{d},a,o) &= \P(\{s_{d+1,i}, w_{d+1,i}\} \mid \{s_{d,i}, w_{d,i}\}, a,o) \\
    &= \prod_{i=1}^C \P(s_{d+1,i}, w_{d+1,i} \mid s_{d,i}, a, o) \label{eqn:pb-trans-dist} \\
    &= \begin{cases}
        \prod_{i=1}^C \mathcal{T}(s_{d+1,i}\mid s_{d,i},a) & \text{if }w_{d+1,i} = w_{d,i}\cdot \mathcal{Z}(o\mid a,s_{d+1,i})\; \forall i\\
        0 & \text{otherwise}.\\
    \end{cases}
  \end{align}
  The second term in the integrand product $\P(o\mid \bar{b}_{d},a)$ is the observation likelihood given a particle belief and an action. 
  This is equivalent to weighted sum of observation likelihoods conditioning on the observation having been generated from the respective $i$-th particle:
  \begin{align}
    \P(o\mid \bar{b}_{d},a) &= \P(o\mid \{s_{d,i},w_{d,i}\},a) = \frac{\sum_{i=1}^C w_{d,i}\cdot\P(o\mid s_{d,i},a)}{\sum_{i=1}^C w_{d,i}}\\
    &= \frac{\sum_{i=1}^C w_{d,i}\cdot[\int_S \mathcal{Z}(o\mid a,s')\mathcal{T}(s'\mid s_{d,i},a) ds']}{\sum_{i=1}^C w_{d,i}}.
  \end{align}
  Note that this density $\tau$ is usually impossible or very difficult to calculate explicitly.
  However, it is rather easy to sample from it using generative models.
  \item $\rho$: Reward function $\rho(\bar{b}_{d},a)$: 
  \begin{align}
    \rho(\bar{b}_{d},a) &= \frac{\sum_i w_{d,i} \cdot R(s_{d,i},a)}{\sum_i w_{d,i}}.
  \end{align}
  Note that if $R$ is bounded by $R_{\max}$, $\rho$ is also bounded with $||\rho||_{\infty} \leq R_{\max}$, since the normalized weights sum to 1.
  \item $\gamma$: Discount factor. Remains the same as the original discount factor.
\end{itemize}

The significance of defining a corresponding particle belief MDP is that we can directly adapt any sampling-based MDP algorithm to approximately solve a POMDP by only changing the transition generative model.
The transition generative model will now be a sampler based on particle filtering, as the particle belief MDP deals with particle belief states.
Furthermore, this allows $Q$-value convergence guarantees of the MDP algorithms to translate nicely into solving the POMDP, as we will prove later in this paper that the optimal $Q$-values of the POMDP $Q_{\mathbf{P}}^*$ and PB-MDP $Q_{\mathbf{M}_{\mathbf{P}}}^*$ are close with high probability.

\section{Sparse Sampling-$\omega$}
\label{sec:ssw}
In order to show that the optimal $Q$-values of the POMDP, $Q_{\mathbf{P}}^*$, and PB-MDP, $Q_{\mathbf{M}_{\mathbf{P}}}^*$, are approximately equivalent, we first introduce an algorithm called Sparse Sampling-$\omega$ (sparse sampling with weights), which will serve as a theoretical bridge between POMDP and PB-MDP.
Sparse Sampling-$\omega$ is a sparse sampling solver that uses particle belief states with particle likelihood weighting to deal with observation uncertainty.
As is evident from the name, Sparse Sampling-$\omega$ takes inspiration from sparse sampling \cite{kearns2002sparse} for continuous state MDPs, using particle belief states.
Note that Sparse Sampling-$\omega$ is purely a theoretical intermediary tool to bridge POMDPs and PB-MDPs, and fully expanding the state and action nodes is extremely computationally inefficient.
Rather, this theoretically well-behaved algorithm is what lets us effectively bridge the gap between $Q_{\mathbf{P}}^*$ and $Q_{\mathbf{M}_{\mathbf{P}}}^*$.

\subsection{Algorithm Definition}
\begin{algorithm}[t]
\algrenewcommand\algorithmicprocedure{\textbf{algorithm}}
\caption{Sparse Sampling-$\omega$}\label{alg:ssw}
\begin{multicols}{2}
\small
\textbf{Global Variables:} $\gamma,G,C,D$.\\

\textbf{Procedure:} \textsc{GenPF}($\bar{b},a$)\\
\textbf{Input:} particle belief set $\bar{b}=\{(s_i,w_i)\}$, action $a$.\\
\textbf{Output:} New updated particle belief set $\bar{b}'=\{(s_i',w_i')\}$, mean reward $\rho$.
\begin{algorithmic}[1]
    \State $s_o \leftarrow$ sample $s_i$ from $\bar{b}$ w.p. $\frac{w_i}{\sum_i w_i}$
    \State $o \leftarrow G(s_o,a)$
    \For{$i = 1,\ldots,C$}
        \State $s_i', r_i \leftarrow G(s_i, a)$ \label{eq:independent}
        \State $w_i' \leftarrow w_i \cdot \mathcal{Z}(o|a,s_i')$
    \EndFor
    \State $\bar{b}' \gets \left\{(s_i', w_i')\right\}_{i=1}^C$
    \State $\rho \leftarrow \sum_i w_i r_i/\sum_i w_i $
    \State \Return $\bar{b}', \rho$
\end{algorithmic}
\small
\columnbreak
\textbf{Procedure:} \textsc{EstimateV}($\bar{b}, d$)\\
\textbf{Input:} particle belief set $\bar{b}=\{(s_i,w_i)\}$, depth $d$.\\
\textbf{Output:} A scalar $\hat{V}_d^*(\bar{b})$ that is an estimate of $V^*(\bar{b})$.
\begin{algorithmic}[1]
    \If{$d\geq D$}
        \State \Return 0
    \EndIf
    \For{$a\in A$}
        \State $\hat{Q}_d^*(\bar{b},a) \leftarrow \textsc{EstimateQ}(\bar{b}, a, d)$
    \EndFor
    \State \Return $\hat{V}_d^*(\bar{b}) \leftarrow \max_{a\in A} \hat{Q}_d^*(\bar{b},a)$
\end{algorithmic}
\vspace{1em}
\textbf{Procedure:} \textsc{EstimateQ}($\bar{b}, a, d$)\\
\textbf{Input:} particle belief set $\bar{b}=\{(s_i,w_i)\}$, action $a$, depth $d$.\\
\textbf{Output:} A scalar $\hat{Q}_d^*(\bar{b},a)$ that is an estimate of $Q_d^*(b,a)$.
\begin{algorithmic}[1]
    \For{$i = 1,\ldots,C$}
        \State $\bar{b}'_i, \rho \leftarrow \textsc{GenPF}(\bar{b},a)$ \label{line:diff}
        \State $\hat{V}_{d+1}^*(\bar{b}'_i) \leftarrow \textsc{EstimateV}(\bar{b}'_i, d+1)$
    \EndFor
    \State \Return $\hat{Q}_d^*(\bar{b},a) \leftarrow \rho + \frac{1}{C}\sum_{i=1}^{C} \gamma\cdot\hat{V}_{d+1}^*(\bar{b}'_i)$
\end{algorithmic}

\end{multicols}
\end{algorithm}

The Sparse Sampling-$\omega$ algorithm is defined with the procedures listed in \cref{alg:ssw}.
The global variables are the discount factor $\gamma$, the generative model $G$, the observation width and number of particles $C$, and the planning depth $D$.
\textsc{GenPF} is the helper function to generate the next-step particle belief set, where the particles are evolved according to the transition density $\mathcal{T}$ and the weights are updated through the observation density $\mathcal{Z}$.
In \textsc{GenPF}, the sampled states $s'_i$ are inserted into each next-step particle belief set $\overline{bao_j}$ with the new weights $w_i' = w_i \cdot \mathcal{Z}(o_j\mid a,s_i')$, which are the adjusted probability of hypothetically sampling observation $o_j$ from state $s_i'$.
Furthermore, the reward returned by \textsc{GenPF} is the particle likelihood weighted reward $\rho = \sum_i w_i r_i/\sum_i w_i $ of the current particle belief state, which is a constant output for a fixed pair of $\bar{b},a$.

The main planning functions in Sparse Sampling-$\omega$ are the \textsc{EstimateV} and \textsc{EstimateQ} procedures.
We use particle belief set $\bar{b}$ at every step $d$, which contain pairs $(s_i,w_i)$ that correspond to the generated sample and its corresponding weight.
\textsc{EstimateV} is a subroutine that returns the value function $V$, for an estimated state or belief, by calling \textsc{EstimateQ} for each action and returning the maximum.
Similarly, \textsc{EstimateQ} performs sampling and recursively calls \textsc{EstimateV} to estimate the $Q$-function at a given step with a weighted average. 
In \textsc{EstimateQ}, Sparse Sampling-$\omega$ samples the next particle belief state using \textsc{GenPF}. 

Consequently, the Sparse Sampling-$\omega$ policy action can be obtained by calling the value estimation function \textsc{EstimateV}$(\bar{b}_0,0)$ at the root node and taking an action that maximizes the $Q$-value.
The particle belief set is initialized by drawing samples from $b_0$ and setting weights to $1/C$, as the samples were drawn directly from $b_0$.
Sparse Sampling-$\omega$ is not computationally efficient as it fully expands the sparsely sampled tree with full particle belief states.
It serves only to demonstrate theoretical convergence and is only practically applicable to very small toy POMDP problems.

Sparse Sampling-$\omega$ is identical to the sparse sampling algorithm \cite{kearns2002sparse} planning on a particle belief MDP.
It also is a slight modification of the previously-published POWSS algorithm~\cite{lim2020sparse}.
Specifically, whereas POWSS generates exactly one observation and corresponding new belief for each particle in a belief, Sparse Sampling-$\omega$ randomly selects a state to generate the observation each time \textsc{GenPF} is called in \cref{line:diff} of \cref{alg:ssw}.
This means that Sparse Sampling-$\omega$ performs a Monte Carlo sampling estimate of the next step value, while POWSS performs an importance weighted summation over the estimates.

Most importantly, this duality of being a modification of POWSS algorithm maintaining similar convergence guarantees for POMDPs while simultaneously being an adaptation of the sparse sampling algorithm for particle belief MDP makes it the ideal candidate to bridge POMDPs and PB-MDPs together.
As an added benefit, the definition of Sparse Sampling-$\omega$ is much simpler than the original POWSS algorithm, while still allowing us to use similar analysis techniques used in both POWSS and sparse sampling.

\subsection{Theoretical Analysis}
\label{sec:ssw-convergence}
In this section, we will prove that Sparse Sampling-$\omega$ algorithm can be made to approximate both optimal $Q$-values of the POMDP $Q_{\mathbf{P}}^*$ and PB-MDP $Q_{\mathbf{M}_{\mathbf{P}}}^*$ arbitrarily closely by increasing the observation width $C$.
\cref{thm:ssw} proves that the Sparse Sampling-$\omega$ algorithm approximates these $Q$-values with high probability by combining results from self-normalized importance sampling estimators and POWSS optimality proofs \cite{lim2020sparse} to prove the optimality in $Q_{\mathbf{P}}^*$, and sparse sampling proof \cite{kearns2002sparse} to prove the optimality in $Q_{\mathbf{M}_{\mathbf{P}}}^*$.

\subsubsection{Importance Sampling} \label{sec:is}
We begin the theoretical portion of this work by stating an important property about self-normalized importance sampling estimators (SN estimators). 
We have previously published this property \cite{lim2020sparse} but present it again here because of its importance to our analysis.
One goal of importance sampling is to estimate an expected value of a function $f(x)$ where $x$ is drawn from a distribution $\mathcal{P}$ while the estimator only has access to another distribution $\mathcal{Q}$ along with the importance weights $w_{\mathcal{P}/\mathcal{Q}}(x) \propto \mathcal{P}(x)/\mathcal{Q}(x)$.
This technique is crucial for Sparse Sampling-$\omega$ because we wish to estimate the value for beliefs conditioned on observation sequences while only being able to sample from the marginal distribution of states for a given action sequence.

We define the following quantities:
\begin{align}
  \tilde{w}_{\mathcal{P}/\mathcal{Q}}(x) &\equiv \frac{w_{\mathcal{P}/\mathcal{Q}}(x)}{\sum_{i=1}^N w_{\mathcal{P}/\mathcal{Q}}(x_i)} \tag{SN Importance Weight}\\
  d_\alpha(\mathcal{P}||\mathcal{Q}) &\equiv \E_{x \sim \mathcal{Q}}[w_{\mathcal{P}/\mathcal{Q}}(x)^\alpha] \tag{R\'enyi Divergence}\\
  \tilde{\mu}_{\mathcal{P}/\mathcal{Q}} &\equiv \sum_{i=1}^N \tilde{w}_{\mathcal{P}/\mathcal{Q}}(x_i) f(x_i) \tag{SN Estimator}\text{.}
\end{align}

Of particular importance is the infinite R\'enyi Divergence, $d_\infty$, which can be rewritten as an almost sure bound on the ratio of $\mathcal{P}$ and $\mathcal{Q}$:
\begin{equation}
d_{\infty}(\mathcal{P}||\mathcal{Q}) = \underset{x \sim \mathcal{Q}}{\text{ess}\sup} \, w_{\mathcal{P}/\mathcal{Q}}(x) \text{.}
\end{equation}
Assuming $d_{\infty}(\mathcal{P}||\mathcal{Q})$ is finite, we prove an estimator concentration bound in the following theorem.

\begin{theorem}[SN $d_{\infty}$-Concentration Bound]\label{thm:sn}
  Let $\mathcal{P}$ and $\mathcal{Q}$ be two probability measures on the measurable space $(\mathcal{X},\mathcal{F})$ with $\mathcal{P}$ absolutely continuous w.r.t. $\mathcal{Q}$ and $d_{\infty}(\mathcal{P}||\mathcal{Q}) < +\infty$. Let $x_1,\ldots, x_N$ be $N$ independent identically distributed random variables  with distribution $\mathcal{Q}$, and $f: \mathcal{X} \to \R$ be a bounded function ($\norm{f}_\infty < + \infty$). Then, for any $\lambda > 0$ and $N$ large enough such that $\lambda > \norm{f}_{\infty}  d_{\infty}(\mathcal{P}||\mathcal{Q})/\sqrt{N}$, the following bound holds with probability at least $1-3\exp(-N\cdot t^2(\lambda,N))$:
  \begin{align} 
    &|\E_{x\sim \mathcal{P}}[f(x)] - \tilde{\mu}_{\mathcal{P}/\mathcal{Q}}| \leq \lambda, \\
    &t(\lambda,N) \equiv \frac{\lambda}{\norm{f}_{\infty}d_{\infty}(\mathcal{P}||\mathcal{Q})}-\frac{1}{\sqrt{N}}.
  \end{align}
\end{theorem} 
\cref{thm:sn} builds upon the derivation in Proposition D.3 of Metelli \textit{et al.}~(\citeyear{Metelli2018}), which provides a polynomially decaying bound by assuming $d_2$ is bounded. Here, we compromise by further assuming that $d_{\infty}$ exists and is bounded to get an exponentially decaying bound.
The proof of \cref{thm:sn} is given in \cref{app:theorem1}, and the intuitive explanation of the $d_{\infty}$ assumption in the POMDP planning context is given in \cref{sec:assumptions}.

This exponential decay is important for the proofs in this section. 
We need to ensure that all nodes of the Sparse Sampling-$\omega$ tree at all depths $d$ reach convergence. 
The branching of the tree induces a factor proportional to $C^{D}$. 
Theorem 1 applied with $N=C$ will not only help offset the $C^{D}$ factor even with increasing depths, but also be consistent with Hoeffding-type bound exponential error rate that we also use to bound intermediate estimator errors.

\subsubsection{Assumptions for Analyzing Sparse Sampling-$\omega$.}\label{sec:assumptions}
The following assumptions are needed for the Sparse Sampling-$\omega$ coupled convergence proof:
\begin{enumerate}[label=(\roman*)]
    \item $S$ and $O$ are continuous spaces, and the action space has a finite number of elements, $|A|<+\inft$. \label{req:space}
    \item The densities $\mathcal{Z},\mathcal{T},b_0$ have the property that, for any observation sequence $\{o_{n,j}\}_{n=1}^d$, the R\'enyi divergence of the target distribution $\mathcal{P}^d$ and sampling distribution $\mathcal{Q}^d$ (\cref{eq:pqdef,eq:pqdef2}) is bounded above by $d_{\inft}^{\max}$ for all $d = 0,\ldots,D-1$: \label{req:Renyi}
    \begin{align}
        d_{\inft}(\mathcal{P}^d||\mathcal{Q}^d) = \text{ess sup}_{x \sim \mathcal{Q}^d}\, w_{\mathcal{P}^d/\mathcal{Q}^d}(x) \leq d_{\inft}^{\max}
    \end{align}
    \item The reward function $R$ is bounded by a finite constant $R_{\max}$, and hence the value function is bounded by $V_{\max} \equiv \frac{R_{\max}}{1-\gamma}$. \label{req:r}
    \item We can sample from the generating function $G$ and evaluate the observation density $\mathcal{Z}$. \label{req:generate}
    \item The POMDP terminates after no more than $D < \infty$ steps. \label{req:finite}
    \item We restrict our analysis to all the beliefs $b \in B$ that are realizable from the initial belief $b_0$ through Bayesian updates with action sequences $\{a_n\}$ and observation sequences $\{o_n\}$. \label{req:realizable}
\end{enumerate}
Intuitively, condition \ref{req:Renyi} means that the ratio of the observation probability conditioned on the true state to the marginal observation probability cannot be too high.
Additionally, the results still hold even when either of $S$ or $O$ are discrete, so long as it does not violate condition \ref{req:Renyi}, by appropriately switching the integrals to sums.

Although our analysis is restricted to the case when $\gamma<1$ and the problem has a finite horizon, we believe that similar results can be derived for either when $\gamma = 1$ for a finite horizon or for infinite horizon problems when $\gamma < 1$ by using the common argument that eventually future discounted rewards will be small~\cite{silver2010pomcp,kearns2002sparse}.
Furthermore, while the results from this section repeat steps taken in proving POWSS \cite{lim2020sparse}, we significantly modify the details for Sparse Sampling-$\omega$.

\subsubsection{Particle Likelihood Weighting Accuracy.} \label{sec:plw_accuracy}
As a precursor to \cref{thm:ssw}, we establish a general result about function estimation using state particles with likelihood weights. 
This is useful because the inductive proof for showing Sparse Sampling-$\omega$ convergence in \cref{lemma:induction} relies heavily upon an SN estimator concentration inequality as well as a Hoeffding-type inequality.

\begin{lemma}[Particle Likelihood SN Estimator Convergence]\label{lemma:pwl_sn}
  Suppose a function $f$ is bounded by a finite constant $\norm{f}_{\infty} \leq f_{\max}$, and a particle belief state $\bar{b}_d=\{s_{d,i},w_{d,i}\}$ at depth $d$ represents $b_d$ with particle likelihood weighting that is recursively updated as $w_{d,i} = w_{d-1,i} \cdot \mathcal{Z}(o_{d}\mid~a,s_d)$.
  Then, for all $d = 0,\ldots,D-1$, the following weighted average is the SN estimator of $f$ under the belief $b_d$ corresponding to the actions $\{a_n\}_{n=0}^{d-1}$ and observations $\{o_n\}_{n=1}^d$, for all beliefs $b_d \in B$ that are realizable given the initial belief $b_0$:
  \begin{align}
      \tilde{\mu}_{\bar{b}_d}[f] &= \frac{\sum_{i = 1}^C w_{d,i}f(s_{d,i})}{\sum_{i = 1}^C w_{d,i}},
  \end{align}
  and the following concentration bound holds with probability at least $1-3\exp(-C\cdot t_{\max}^2(\lambda,C))$,
  \begin{align}
    &|\E_{s\sim {b_d}}[f(s)] - \tilde{\mu}_{\bar{b}_d}[f] | \leq \lambda,\\
    &t_{\max}(\lambda,C) \equiv \frac{\lambda}{f_{\max} d_{\inft}^{\max}} - \frac{1}{\sqrt{C}}.
  \end{align}
\end{lemma}
\begin{proof}
We only outline the important steps here, and defer the detailed proof of this lemma to \cref{app:lemma1}.
The key of this proof lies in the fact that the state particles trajectories $\{s_{n,1}\},\ldots,$ $\{s_{n,C}\}$ are independent identically distributed random variable sequences of depth $d$, as \textsc{GenPF} independently generates each state sequence $i$ according to the transition density $\mathcal{T}$.
While \textsc{GenPF} generates highly correlated observation sequences and histories $\{o_{n}\}_{n=1}^d$, the dependence on observation sequence for a given particle belief state is only through the particle likelihood weights.

We abbreviate some terms of interest with the following notation:
\begin{align}
  \mathcal{T}_{1:d}^i &\equiv \prod_{n=1}^d \mathcal{T}(s_{n,i}\mid s_{n-1,i},a_n);\quad \mathcal{Z}_{1:d}^{i} \equiv \prod_{n=1}^d \mathcal{Z}(o_{n}\mid a_{n},s_{n,i}) \text{,}
\end{align}
where $d$ is the depth, and $i$ is the index of the state sample.
Intuitively, $\mathcal{T}_{1:d}^i$ is the transition density of the $i$th state sequence, $\{s_{n,i}\}_{n=1}^d$, and $\mathcal{Z}_{1:d}^{i}$ is the conditional density of observation sequence $\{o_n\}$ given the $i$th state sequence from the root node to depth $d$.
Additionally, $b_d^i$ denotes $b_d(s_{d,i})$ and $w_{d,i}$ the weight of $s_{d,i}$.

Then, we apply importance sampling to our system for all depths $d=0,\ldots,D-1$.
Here, $\mathcal{P}^d$ is the normalized measure of the state sequence $\{s_{n,i}\}_{n=0}^d$ conditioned on the observation sequence $\{o_{n}\}_{n=1}^d$ and action sequence $\{a_n\}_{n=0}^{d-1}$ up to the node at depth $d$, and $\mathcal{Q}^d$ is the measure of the state sequence conditioned only on the action sequence. 
For simplicity, we use $\mathcal{Z}_{1:d}$ to denote the product of observation likelihoods $\prod_{n=1}^d\mathcal{Z}(o_n\mid a_{n-1},s_n)$ and $\mathcal{T}_{1:d}$ to denote the product of transition densities $\prod_{n=1}^d\mathcal{T}(s_n\mid s_{n-1}, a_{n-1})$.
Then, for an arbitrary action sequence $\{a_n\}$, the following describes the densities necessary to define importance weighting:
\begin{align}
  &\mathcal{P}^d = \mathcal{P}_{\{a_n, o_{n}\}}^d(\{s_{n,i}\}) = \frac{(\mathcal{Z}_{1:d}^{i}) (\mathcal{T}_{1:d}^i) b_{0}^i}{\int_{S^{d+1}} (\mathcal{Z}_{1:d}) (\mathcal{T}_{1:d}) b_{0} ds_{0:d}}\label{eq:pqdef}\\
  &\mathcal{Q}^d = \mathcal{Q}_{\{a_n\}}^d(\{s_{n,i}\}) = (\mathcal{T}_{1:d}^i) b_{0}^i \label{eq:pqdef2}\\
  &w_{\mathcal{P}^d/\mathcal{Q}^d}(\{s_{n,i}\}) = \frac{(\mathcal{Z}_{1:d}^{i}) }{\int_{S^{d+1}} (\mathcal{Z}_{1:d}) (\mathcal{T}_{1:d}) b_{0} ds_{0:d}} \text{.}
\end{align}
Here, the integral to calculate the normalizing constant is taken over $S^{d+1}$, the Cartesian product of the state space $S$ over $d+1$ steps.
Now, we can show that the recursive likelihood updating scheme in \cref{lemma:pwl_sn} produces valid likelihood weights up to a normalization by simply expanding the weight $w_{d,i}$:
\begin{align}
  w_{d,i} &= w_{d-1,i}\cdot \mathcal{Z}(o_{d}\mid a_{d-1},s_{d,i}) = w_{d-2,i} \prod_{n=d-1}^d \mathcal{Z}(o_{n}\mid a_{n-1},s_{n,i}) = \ldots = \mathcal{Z}_{1:d}^{i} \propto w_{\mathcal{P}^d/\mathcal{Q}^d}(\{s_{n,i}\}).
\end{align}
Consequently, we conclude that the weighted average with particle likelihood weights indeed corresponds to the proper SN estimator:
\begin{align}
  \tilde{\mu}_{\bar{b}_d}[f] &= \frac{\sum_{i = 1}^C w_{d,i}\cdot f(s_{d,i})}{\sum_{i = 1}^C w_{d,i}} = \frac{\sum_{i = 1}^C w_{\mathcal{P}^{d}/\mathcal{Q}^{d}}(\{s_{n,i}\}) \cdot f(s_{d,i})}{\sum_{i = 1}^C w_{\mathcal{P}^{d}/\mathcal{Q}^{d}}(\{s_{n,i}\}) }.
\end{align}
We can apply the SN concentration inequality in \cref{thm:sn} to obtain the concentration bound. 
\end{proof}
Note that proving this lemma allows us to apply the particle likelihood weighting SN inequality whenever we encounter weighted averages with particle likelihood weights for a realizable particle belief.
Also, this result does not depend on any specific choice of observation sequence $\{o_n\}$.

\subsubsection{Coupled Convergence of Sparse Sampling-$\omega$.}
The theorem below describes Sparse Sampling-$\omega$'s coupled convergence to both optimal $Q$-values of the POMDP $Q_{\mathbf{P}}^*$ and PB-MDP $Q_{\mathbf{M}_{\mathbf{P}}}^*$, as $C$ is increased.

\begin{theorem}[Sparse Sampling-$\omega$ Coupled Optimality]\label{thm:ssw}
  Suppose conditions \ref{req:space}-\ref{req:realizable} are satisfied. Then, for any $\lambda >0$ and $0<\delta\leq 1$, choosing particle count constant $C$ that satisfies:
  \begin{align}
    C &= \max\lrc{\lrp{\frac{4V_{\max} d_{\inft}^{\max}}{\lambda}}^2, \frac{64V_{\max}^2 }{\lambda^2}\lrp{D\log\frac{24|A|^{\frac{D+1}{D}}V_{\max}^2D }{\lambda^2} + \log\frac{1}{\delta}}},
  \end{align}
  the $Q$-function estimates $\hat{Q}^*_{\omega,d}(\bar{b}_d,a)$ obtained for all depths $d=0,\ldots,D-1$, realized beliefs or histories $b_d$ encountered in the Sparse Sampling-$\omega$ tree, and actions $a$ are jointly near-optimal with respect to $Q_{\mathbf{P},d}^*$ and $Q_{\mathbf{M}_{\mathbf{P}},d}^*$ with probability at least $1 - \delta$: 
  \begin{align}
    |Q_{\mathbf{P},d}^*(b_d,a) - \hat{Q}^*_{\omega,d}(\bar{b}_d,a) | & \leq \frac{\lambda}{1-\gamma}, \label{eq:ssw-powss}\\ 
    |Q_{\mathbf{M}_{\mathbf{P}},d}^*(\bar{b}_d,a) - \hat{Q}^*_{\omega,d}(\bar{b}_d,a) | & \leq \frac{\lambda}{1-\gamma}. \label{eq:ssw-ss}
  \end{align}
\end{theorem}

To prove \cref{thm:ssw}, we follow a similar proof strategy from our previous proof for POWSS \shortcite{lim2020sparse} to show that \cref{eq:ssw-powss} holds, and a similar strategy of the original sparse sampling proof \shortcite{kearns2002sparse} to show that \cref{eq:ssw-ss} holds. 
In essence, this Sparse Sampling-$\omega$ convergence guarantee builds on POWSS and sparse sampling convergence guarantees, providing coupled convergence results to optimal $Q$-values of the POMDP $Q_{\mathbf{P}}^*$ and PB-MDP $Q_{\mathbf{M}_{\mathbf{P}}}^*$.

First, we use induction in \cref{lemma:induction} to prove a concentration inequality for the value function at all nodes in the tree, starting at the leaves and proceeding up to the root.
Consequently, proving \cref{lemma:induction} allows us to prove \cref{thm:ssw}, with some justifications of how the parameter $C$ can actually be explicitly chosen with the choice of $\lambda,\delta$.
The detailed proof for \cref{thm:ssw} is in \cref{app:theorem2}.

\begin{lemma}[Sparse Sampling-$\omega$ Estimator $Q$-Value Coupled Convergence]\label{lemma:induction}
  For all $d = 0,\ldots,D-1$ and $a$, the following bounds hold with probability at least $1 - 6|A|(4|A|C)^{D}\exp(-C\cdot \tilde{t}^2)$:
  \begin{align} 
    &|Q^*_{\mathbf{P},d}(b_d,a) - \hat{Q}_{\omega,d}^*(\bar{b}_d,a)| \leq \alpha_d,\; \alpha_{d} = \lambda + \gamma\alpha_{d+1},\; \alpha_{D-1} = \lambda,\label{eq:p_to_ssw}\\
    &|Q^*_{\mathbf{M}_{\mathbf{P}},d}(\bar{b}_d,a)- \hat{Q}_{\omega,d}^*(\bar{b}_d,a)| \leq \beta_d,\; \beta_{d} = \gamma(\lambda + \beta_{d+1}),\; \beta_{D-1} = 0, \label{eq:ssw_to_m}\\
    &t_{\max}(\lambda,C) = \frac{\lambda}{4V_{\max} d_{\inft}^{\max}} - \frac{1}{\sqrt{C}},\; \tilde{t} = \min\{t_{\max},\lambda/4\sqrt{2}V_{\max}\}
  \end{align}
\end{lemma}

\begin{proof} 
We outline how we use the particle likelihood SN estimator inequality and Hoeffding inequality to bound the $Q$-values, and defer the detailed proof to \cref{app:lemma2}.

The optimal $d$-step $Q$-values for the POMDP $Q_{\mathbf{P}}^*$ and the corresponding PB-MDP $Q_{\mathbf{M}_{\mathbf{P}}}^*$ are
\begin{align}
  Q^*_{\mathbf{P},d}(b_d,a) &= \E_{\mathbf{P}}[R(s_{d},a) + \gamma V_{\mathbf{P},d+1}^*(b_dao)\mid b_d] \\
  &= \int_S R(s_{d},a) b_{d} \cdot ds_{d} + \gamma \int_S \int_S \int_O V_{\mathbf{P},d+1}^*(b_dao)(\mathcal{Z}_{d+1}) (\mathcal{T}_{d,d+1})b_d \cdot ds_{d:d+1}do, \\
  Q^*_{\mathbf{M}_{\mathbf{P}},d}(\bar{b}_d,a) &= \rho(\bar{b},a) + \gamma\E_{\mathbf{M}_{\mathbf{P}}}[V_{\mathbf{M}_{\mathbf{P}},d+1}^*(\bar{b}_{d+1})\mid \bar{b}_d, a]\\
  &= \frac{\sum_i w_{d,i} \cdot R(s_{d,i},a)}{\sum_{i=1}^C w_{d,i}} + \gamma\int_\Sigma V_{\mathbf{M}_{\mathbf{P}},d+1}^*(\bar{b}_{d+1})\tau(\bar{b}_{d+1}\mid \bar{b}_{d},a) d\bar{b}_{d+1}.
\end{align}
The Sparse Sampling-$\omega$ value estimates are mathematically equal to
\begin{align}
    \hat{V}^*_{\omega,d}(\bar{b}_d) &= \max_{a\in A}\hat{Q}^*_{\omega,d}\left(\bar{b}_d,a\right)\\
    \hat{Q}^*_{\omega,d}(\bar{b}_d,a) &= \frac{\sum_{i = 1}^C w_{d,i}r_{d,i}}{\sum_{i = 1}^C w_{d,i}}+\frac{1}{C}\sum_{i=1}^{C}\gamma\cdot \hat{V}^*_{\omega,d+1}\left(\bar{b}_{d+1}'^{[I_i]}\right),
\end{align}
where $\{I_i\}$ are $C$ independent identically distributed random variables with finite discrete distribution $p_{w,d}$ with probability mass $p_{w,d}(I=j) = (w_{d,j}/\sum_k w_{d,k})$, and particle belief state $\bar{b}_{d+1}'^{[I_i]}$ is updated by an observation generated from $s_{d,I_i}$.
This reflects the fact that \textsc{GenPF} randomly selects a state particle $s_o$ with probability $w_{d,o}/\sum_k w_{d,k}$ $C$ times independently to generate a new observation for the particle belief state after next step.

\textbf{POMDP Value Convergence:} First, we show that \cref{eq:p_to_ssw} is satisfied, which is an adapted and substantially modified proof of POWSS convergence \cite{lim2020sparse}.
Using the triangle inequality for a given step $d$ of the inductive proof, we split the difference into two terms, the reward estimation error (A) and the next-step value estimation error (B):
\begin{align}
  |Q_{\mathbf{P},d}^*(b_d,a) - \hat{Q}^*_{\omega,d}(\bar{b}_d,a) | \leq &\undb{\left| \E_{\mathbf{P}}[R(s_d,a)\mid b_d] - \frac{\sum_{i = 1}^C w_{d,i}r_{d,i}}{\sum_{i = 1}^C w_{d,i}}\right|}{(A)} \\
  &+ \gamma \undb{\left| \E_{\mathbf{P}}[V_{\mathbf{P},d+1}^*(b_dao)\mid b_d] - \frac{1}{C}\sum_{i=1}^{C} \hat{V}^*_{\omega,d+1}(\bar{b}_{d+1}'^{[I_i]}) \right|}{(B)} \notag
\end{align}

The reward estimation error (A) is exactly the particle likelihood importance sampling error for estimating the reward function $R(\cdot, a)$, which can be bounded by applying \cref{lemma:pwl_sn}. This also proves the base case.

To bound the next-step value estimation error (B), we introduce particle likelihood SN estimators and Monte Carlo average estimators to bridge the following quantities (detailed definitions and bounds of each terms are in \cref{app:theorem2}):
{\small
\begin{align}
    &\undb{\left| \E_{\mathbf{P}}[V_{\mathbf{P},d+1}^*(b_dao)\mid b_d] - \frac{1}{C}\sum_{i=1}^{C} \hat{V}^*_{\omega,d+1}(\bar{b}_{d+1}'^{[I_i]}) \right|}{(B)}\\
    &\leq \undb{\left| \E_{\mathbf{P}}[V_{\mathbf{P},d+1}^*(b_dao)\mid b_d] - \frac{\sum_{i = 1}^C w_{d,i}\mathbf{V}_{\mathbf{P},d+1}^*(b_d,a)^{[i]}}{\sum_{i = 1}^C w_{d,i}} \right|}{(1) Importance sampling error}  + \undb{\left| \frac{\sum_{i = 1}^C w_{d,i}\mathbf{V}_{\mathbf{P},d+1}^*(b_d,a)^{[i]}}{\sum_{i = 1}^C w_{d,i}} - \frac{1}{C}\sum_{i=1}^{C}\mathbf{V}_{\mathbf{P},d+1}^*(b_d,a)^{[I_i]} \right|}{(2) MC weighted sum approximation error} \notag\\
    &\quad + \undb{\left| \frac{1}{C}\sum_{i=1}^{C}\mathbf{V}_{\mathbf{P},d+1}^*(b_d,a)^{[I_i]} - \frac{1}{C}\sum_{i=1}^{C}V_{\mathbf{P},d+1}^*(b_dao^{[I_i]}) \right|}{(3) MC next-step integral approximation error} + \undb{\left|\frac{1}{C}\sum_{i=1}^{C}V_{\mathbf{P},d+1}^*(b_dao^{[I_i]}) - \frac{1}{C}\sum_{i=1}^{C} \hat{V}^*_{\omega, d+1}(\bar{b}_{d+1}'^{[I_i]}) \right|}{(4) Inductive function estimation error}. \notag
\end{align}
}

\textbf{PB-MDP Value Convergence:} Second, we show that \cref{eq:ssw_to_m} is satisfied, which is an adapted and substantially modified proof of sparse sampling convergence \cite{kearns2002sparse}.
Once again, we split the difference between the SN estimator and the $Q_{\mathbf{M}_{\mathbf{P}}}^*$ function into two terms, the reward estimation error (A) and the next-step value estimation error (B):
\begin{align}
  |Q^*_{\mathbf{M}_{\mathbf{P}},d}(\bar{b}_d,a)- \hat{Q}_{\omega,d}^*(\bar{b}_d,a) | \leq& \undb{\left| \rho(\bar{b}_d,a) - \rho(\bar{b}_d,a)\right|}{(A) = 0} \\
  &+ \gamma \undb{\left| \E_{\mathbf{M}_{\mathbf{P}}}[V_{\mathbf{M}_{\mathbf{P}},d+1}^*(\bar{b}_{d+1})\mid \bar{b}_{d}, a] - \frac{1}{C}\sum_{i=1}^{C} \hat{V}^*_{\omega,d+1}(\bar{b}_{d+1}'^{[I_i]}) \right|}{(B)} \notag
\end{align}

Since our particle belief MDP induces no reward estimation error, the term (A) is always 0 and proving the base case $d=D-1$ is trivial as (A) and (B) are both 0.
Then, we show that the difference (B) is bounded for all $d=0,\ldots,D-1$.
We use the triangle inequality repeatedly to separate it into two terms; (1) the MC transition approximation error, and (2) the inductive function estimation error (detailed definitions and bounds of each terms are in \cref{app:theorem2}):
{\small
\begin{align}
    &\undb{\left| \E_{\mathbf{M}_{\mathbf{P}}}[V_{\mathbf{M}_{\mathbf{P}},d+1}^*(\bar{b}_{d+1})\mid \bar{b}_{d}, a] - \frac{1}{C}\sum_{i=1}^{C} \hat{V}^*_{\omega,d+1}(\bar{b}_{d+1}'^{[I_i]}) \right|}{(B)}\\
    &\leq \undb{\left| \E_{\mathbf{M}_{\mathbf{P}}}[V_{\mathbf{M}_{\mathbf{P}},d+1}^*(\bar{b}_{d+1})\mid \bar{b}_{d}, a] - \frac{1}{C}\sum_{i=1}^{C} V_{\mathbf{M}_{\mathbf{P}},d+1}^*(\bar{b}_{d+1}'^{[I_i]}) \right|}{(1) MC transition approximation error} + \undb{\left| \frac{1}{C}\sum_{i=1}^{C} V_{\mathbf{M}_{\mathbf{P}},d+1}^*(\bar{b}_{d+1}'^{[I_i]}) - \frac{1}{C}\sum_{i=1}^{C} \hat{V}^*_{\omega,d+1}(\bar{b}_{d+1}'^{[I_i]}) \right|}{(2) Inductive function estimation error}. \notag
\end{align}
}

Combining the probability bounds used in both of these procedures results in a worst case $\mathcal{O}(\exp(-t \cdot C))$ probability factor, where $t$ is some constant, as both the SN concentration bound and the Hoeffding bound are exponentially decaying.
Since this upper bound on the estimation error needs to hold for all steps $d=0,\ldots, D-1$, we must apply the worst case union bound on the probability to ensure that every node in the tree achieves the desired concentration bound.
This results in a worst case probability factor that is $\mathcal{O}(C^D)$.
Therefore, we can obtain the $Q$-value estimator concentration inequality, with convergence rate $\mathcal{O}(C^{D}\exp(-t\cdot C))$.
\end{proof}

\section{Particle Belief MDP Approximation Guarantees}
\label{sec:pomdp-guarantees}
In this section, we establish the theoretical guarantees for using any approximately optimal MDP planning algorithm to solve the POMDP problem $\mathbf{P}$ by planning in the particle belief MDP $\mathbf{M}_{\mathbf{P}}$.
\cref{thm:pomdp-qvalue} shows that the $Q$-values $Q_{\mathbf{P}}^*$ and $Q_{\mathbf{M}_{\mathbf{P}}}^*$ are close to each other with high probability, and \cref{thm:pomdp-grand} shows that using any approximately optimal MDP planning algorithm $\mathcal{A}$ in the particle belief MDP $\mathbf{M}_{\mathbf{P}}$ as a policy yields near-optimal value in the original POMDP if applied repeatedly in a closed loop with the environment and an exact belief updater.

\subsection{Particle Belief MDP $Q$-Value Approximation Optimality}
We introduce \cref{thm:pomdp-qvalue}, which probabilistically bridges the POMDP $\mathbf{P}$ and its corresponding particle belief MDP $\mathbf{M}_{\mathbf{P}}$.
In essence, this theorem claims that the two optimal $Q$-values $Q_{\mathbf{P}}^*$ and $Q_{\mathbf{M}_{\mathbf{P}}}^*$ are close with high probability, because creating a very accurate $Q$-value estimator via Sparse Sampling-$\omega$ that is close to both $Q_{\mathbf{P}}^*$ and $Q_{\mathbf{M}_{\mathbf{P}}}^*$ happens with high probability.

\begin{theorem}[Particle Belief MDP $Q$-Value Approximation Optimality]\label{thm:pomdp-qvalue}
    Given a finite horizon POMDP $\mathbf{P}$ and its corresponding particle belief MDP $\mathbf{M}_{\mathbf{P}}$, there exists a number of particles $C$ for which the optimal $Q$-value of the POMDP problem $Q_{\mathbf{P}}^*(b,a)$ can be approximated by the optimal $Q$-value of the particle belief MDP problem $Q_{\mathbf{M}_{\mathbf{P}}}^*(\bar{b},a)$ with arbitrary precision. 
    Namely, under the regularity conditions \ref{req:space}-\ref{req:realizable}, the following bound holds for a given realizable belief $b$, corresponding sampled particle belief $\bar{b}$, and all available actions $a$ with probability at least $1-\delta_{\mathbf{M}_{\mathbf{P}}}$ for a desired accuracy $\epsilon_{\mathbf{M}_{\mathbf{P}}}$:
    \begin{align}
        |Q_{\mathbf{P}}^*(b,a) - Q_{\mathbf{M}_{\mathbf{P}}}^*(\bar{b},a)| \leq \epsilon_{\mathbf{M}_{\mathbf{P}}}.
    \end{align}
\end{theorem}

\begin{proof} 
The main idea of the proof is that we bridge the two $Q$-values, $Q_{\mathbf{P}}^*$ and $Q_{\mathbf{M}_{\mathbf{P}}}^*$, via approximation through Sparse Sampling-$\omega$ with $C$ particles. 
From \cref{thm:ssw}, we have established that there exists an algorithm, Sparse Sampling-$\omega$, which is jointly optimal in both senses of POMDP $\mathbf{P}$ and its corresponding particle belief MDP $\mathbf{M}_{\mathbf{P}}$.
Then, if we were to hypothetically perform Sparse Sampling-$\omega$ of depth $D$, the sum of the errors between the three types of $Q$-values at the root node, $Q_{\mathbf{P}}^*$, $Q_{\mathbf{M}_{\mathbf{P}}}^*$ and $\hat{Q}_{\omega}^*$, are jointly bounded with probability at least $1-\delta_{\mathbf{M}_{\mathbf{P}}}$ through \cref{thm:ssw}, where $\delta_{\mathbf{M}_{\mathbf{P}}}=\delta$ for notational clarity in this context.
We use the fact that $Q_{\mathbf{P}}^*$ and $Q_{\mathbf{M}_{\mathbf{P}}}^*$ are the optimal $Q$-values at $d=0$ for the POMDP and PB-MDP, respectively:
\begin{align}
  |Q_{\mathbf{P}}^*(b,a) - Q_{\mathbf{M}_{\mathbf{P}}}^*(\bar{b},a)| &\leq |Q_{\mathbf{P}}^*(b,a) - \hat{Q}_{\omega}^*(\bar{b},a)| + |\hat{Q}_{\omega}^*(\bar{b},a) - Q_{\mathbf{M}_{\mathbf{P}}}^*(\bar{b},a)|\\
  &\equiv |Q^*_{\mathbf{P},0}(b_0,a) - \hat{Q}_{\omega,0}^*(\bar{b}_0,a)| + |Q^*_{\mathbf{M}_{\mathbf{P}},0}(\bar{b}_0,a)- \hat{Q}_{\omega,0}^*(\bar{b}_0,a)|\\
  &\leq \frac{2\lambda}{1-\gamma} \;\;\equiv \epsilon_{\mathbf{M}_{\mathbf{P}}}.
\end{align}
Since this bound holds with high probability for creating any hypothetical Sparse Sampling-$\omega$ tree, this must mean that $|Q_{\mathbf{P}}^*(b,a) - Q_{\mathbf{M}_{\mathbf{P}}}^*(\bar{b},a)| \leq \epsilon_{\mathbf{M}_{\mathbf{P}}}$ in general with high probability.

The convergence rate of $\delta_{\mathbf{M}_{\mathbf{P}}}$ is $\mathcal{O}(C^{D}\exp(-\tilde{t} \cdot C))$.
This means that as we increase the number of particles, we can expect better performance by approximately solving a POMDP via particle belief approximation.
\end{proof}

\subsection{Particle Belief MDP Planning Optimality}
\begin{figure}[t]
    \centering
    \includegraphics[width=\textwidth]{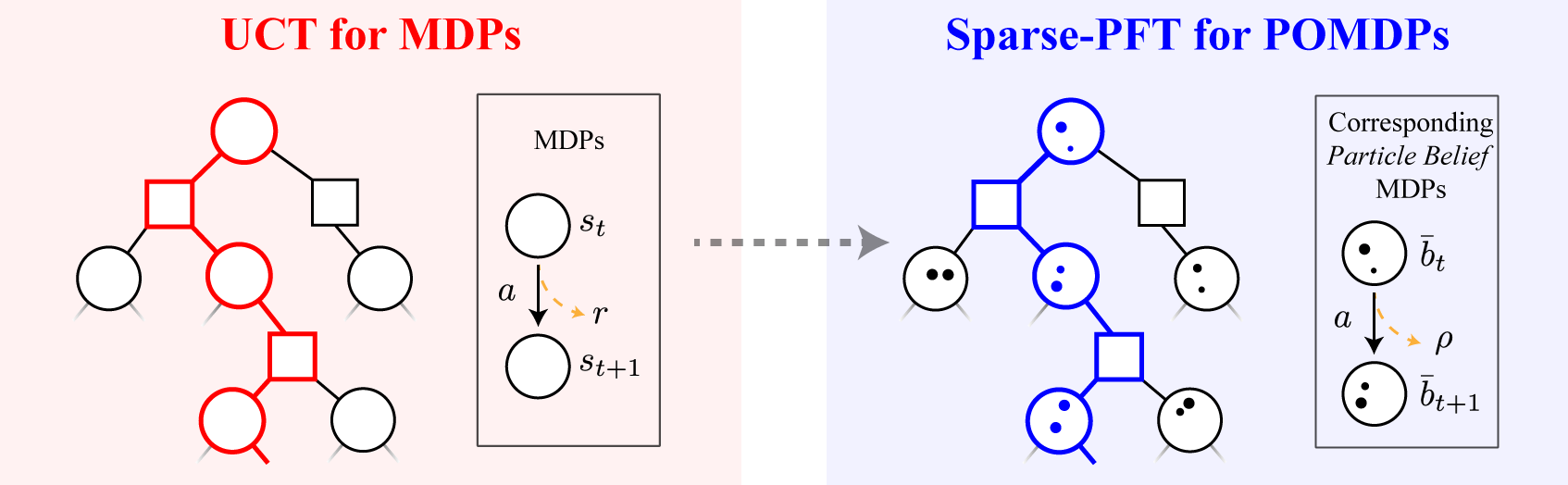}
    \caption{Illustration of promoting an MDP algorithm (UCT) into a POMDP algorithm (Sparse-PFT). 
    This practically only involves changing the transition generative model into a particle filtering-based transition generative model that deals with particle belief states.}
    \label{fig:uct_to_pft}
\end{figure}

\begin{corollary}[Particle Belief MDP Planning Optimality]\label{cor:planning}
    Under regularity conditions necessary for both the particle belief MDP and an MDP planning algorithm $\mathcal{A}$, if the optimal planner can approximate $Q$-values with arbitrary precision $\epsilon_{\mathcal{A}}$ with probability at least $1-\delta_{\mathcal{A}}$ in the corresponding particle belief MDP of a given POMDP, then the planning algorithm can approximate the POMDP $Q$-values within $\epsilon_{\mathbf{M}_{\mathbf{P}}} + \epsilon_{\mathcal{A}}$ with probability at least $1-\delta_{\mathbf{M}_{\mathbf{P}}}-\delta_{\mathcal{A}}$:
    \begin{align}
        |Q_{\mathbf{P}}^*(b,a) - \hat{Q}_{\mathbf{M}_{\mathbf{P}}}^{\mathcal{A}}(\bar{b},a)| \leq \epsilon_{\mathbf{M}_{\mathbf{P}}} + \epsilon_{\mathcal{A}}.
    \end{align}
\end{corollary}
\begin{proof}
  This is a straightforward application of triangle inequality for the $Q$-value estimation accuracy and worst case union bound for the probability.
\end{proof}
Note that it would also be possible to devise an expected value version of the bounds by converting the probability statement into an expected value statement.

Essentially, \cref{cor:planning} means that we can use any approximately optimal MDP planning algorithm to solve the POMDP problem by planning in the particle belief MDP instead, and still retain similar optimality guarantees.
The most remarkable thing about this result is that it does not directly depend on the size of the state space nor the observation space.
However, the dependence may indirectly come through the observation density and thus the R\'enyi divergence factor, and in practice, the generative model sampling complexity often depends on the dimensionality of the state space.
Moreover, even though this approach is insensitive to the state and observation space size, the guarantees and practical algorithms are highly sensitive to the planning horizon $D$.

In most practical cases, this method would usually incur an additional $\mathcal{O}(C)$ compute time factor in a given transition sampling step as single particle belief state generation now needs to propagate $C$ particles forward instead of a single particle/state.
Moreover, if the algorithm requires storing the beliefs, the memory requirements are increased by an $\mathcal{O}(C)$ factor compared with the MDP algorithm.
Fortunately, as demonstrated in \cref{sec:numerical}, a modest number of particles often gives adequate performance in practice.

Proving \cref{cor:planning} allows us to prove \cref{thm:pomdp-grand} with additional results from Kearns \textit{et al.} (\citeyear{kearns2002sparse}) and Singh and Yee~(\citeyear{Singh1994}). 
Through the near-optimality of the $Q$-functions, we conclude that the value obtained by employing a near-optimal MDP policy in the PB-MDP is also near-optimal in the original POMDP with further assumptions on the closed-loop POMDP system.
In this context, we mean a near-optimal MDP planning algorithm $\mathcal{A}$ to be an algorithm with small values of $\epsilon_{\mathcal{A}},\delta_{\mathcal{A}}$ that would satisfy the conditions required in the proof of the theorem.
Examples of such algorithms include sparse sampling \cite{kearns2002sparse} and others \cite{bjarnason2009lower,couetoux2011double}.
The detailed proof for \cref{thm:pomdp-grand} is in \cref{app:theorem4}.

\begin{theorem}[Particle Belief MDP Approximate Policy Convergence] \label{thm:pomdp-grand}
    Suppose a near-optimal MDP planning algorithm $\mathcal{A}$ is used to plan with particle belief MDP $\mathbf{M}_{\mathbf{P}}$ repeatedly in a closed loop with POMDP environment $\mathbf{P}$ and an exact Bayesian belief updater to process observations from the environment. Further assume that regularity conditions \ref{req:space}-\ref{req:realizable} are met for $\mathbf{M}_{\mathbf{P}}$ and that $\mathcal{A}$ can approximate the $Q$-values of $\mathbf{M}_{\mathbf{P}}$ with arbitrary precision $\epsilon_\mathcal{A}$ with probability at least $1-\delta_\mathcal{A}$.
    Then, for any $\epsilon > 0$, we can choose $C$ such that the value obtained by planning with $\mathcal{A}$ in $\mathbf{M}_{\mathbf{P}}$ is within $\epsilon$ of the optimal POMDP value function at $b_0$:
    \begin{align}
        V_{\mathbf{P}}^*(b_0)-V_{\mathbf{M}_{\mathbf{P}}}^{\mathcal{A}}(b_0) &\leq \epsilon.
    \end{align}
\end{theorem}

\subsection{Sparse Particle Filter Tree (Sparse-PFT)}

\begin{algorithm}[t]
\algrenewcommand\algorithmicprocedure{\textbf{algorithm}}
\caption{Sparse-PFT Algorithm}\label{alg:spt}
\begin{multicols}{2}
\small
\textbf{Global Variables:} $\gamma,n,c_\textsc{UCB},\beta_\textsc{UCB},G,D$.\\

\textbf{Procedure:} \textsc{Plan}($b$)\\
\textbf{Input:} Belief $b$.\\
\textbf{Output:} An action $a$.
\begin{algorithmic}[1]
    \For{$i = 1,\ldots,C$}
        \State $s \leftarrow \text{sample from } b$
        \State $\bar{b} \leftarrow \bar{b} \cup \{(s, 1/C)\}$
    \EndFor
    \For{$i = 1,\ldots,n$}
        \State \textsc{Simulate}($\bar{b}, 0$)
    \EndFor
    \State \Return $a \leftarrow \argmax_{a\in C(b)} Q(b,a)$
\end{algorithmic}
\columnbreak
\textbf{Procedure:} \textsc{Simulate}($\bar{b}, d$)\\
\textbf{Input:} particle belief set $\bar{b}=\{(s_i,w_i)\}$, depth $d$.\\
\textbf{Output:} A scalar $q$ that is the total discounted reward of one simulated trajectory sample.
\begin{algorithmic}[1]
    \If{$d=D$}
        \State \Return 0
    \EndIf
    \State $a \leftarrow \argmax_{a\in C(\bar{b})} Q(\bar{b},a) + c_\textsc{UCB}\frac{N(\bar{b})^{\beta_\textsc{UCB}}}{\sqrt{N(\bar{b},a)}}$ \label{line:ucb}
    \If{$|C(\bar{b},a)| = C$}
        \State $\bar{b}', \rho \leftarrow $ sample from $C(\bar{b},a)$
    \Else 
        \State $\bar{b}', \rho \leftarrow \textsc{GenPF}(\bar{b}, a)$
        \State $C(\bar{b},a) \leftarrow C(\bar{b},a) \cup \{(\bar{b}', \rho)\}$
    \EndIf
    \If{$N(\bar{b}) = 0$} 
        \State $q \leftarrow \rho + \gamma \cdot \textsc{Rollout}(\bar{b}', d-1)$
    \Else 
        \State $q \leftarrow \rho + \gamma \cdot \textsc{Simulate}(\bar{b}', d-1)$ 
    \EndIf
    \State $N(\bar{b}) \leftarrow N(\bar{b}) + 1$
    \State $N(\bar{b}, a) \leftarrow N(\bar{b}, a) + 1$
    \State $Q(\bar{b}, a) \leftarrow Q(\bar{b}, a) + \frac{q - Q(\bar{b}, a)}{N(\bar{b}, a)}$
    \State \Return $q$
\end{algorithmic}
\end{multicols}
\end{algorithm}

By utilizing the results in \cref{thm:pomdp-qvalue} and \cref{thm:pomdp-grand}, we can promote a variant of sampling-based MDP planning algorithm Upper Confidence Tree (UCT), Sparse UCT \cite{bjarnason2009lower}, into Sparse Particle Filter Tree (Sparse-PFT) and retain similar convergence guarantees for the POMDP \cite{kocsis2006bandit,bjarnason2009lower,shah2020nonasymptotic}.
This results in an algorithm that is simple to implement, and enjoys both theoretical guarantees and high performance in practice.

The entry point of Sparse-PFT is the \textsc{Plan} procedure which repeatedly calls the the \textsc{Simulate} procedure to construct the tree and choose an action.
Both of these procedures are defined in \cref{alg:spt}.
The set of global variables for Sparse-PFT includes the same global variables used for Sparse Sampling-$\omega$ with the addition of $n$, the number of tree search queries, and $c_\textsc{UCB}$ and $\beta_\textsc{UCB}$, the polynomial Upper Confidence Bound parameters that determine the amount of exploration in \cref{line:ucb} \citep{shah2020nonasymptotic}.
The \textsc{Simulate} function is analogous to the function of the same name from UCT \cite{kocsis2006bandit,bjarnason2009lower}, with the only difference being that Sparse-PFT manages particle belief sets through \textsc{GenPF} (defined in \cref{alg:ssw}) rather than states directly.

In the above algorithm definition, $C(\cdot)$ represents the list of children nodes, $N(\cdot)$ the number of visits to the node, $Q(\cdot)$ the estimated $Q$-value at the node, and $c_{\textsc{UCB}}$ the Upper Confidence Bound exploration parameter.
These lists are all implicitly initialized to 0 or $\emptyset$.
The \textsc{Rollout} procedure is an optional heuristic that runs a simulation with a heuristic rollout policy for $d$ steps to estimate the value, while avoiding building a large computation tree at each step of simulation.

With the introduction of Sparse-PFT, we can view the recent POMDP algorithms as practical extensions of Sparse-PFT.
For instance, PFT-DPW \cite{sunberg2018pomcpow} is a simple modification of Sparse-PFT by utilizing the double progressive widening (DPW) technique to additionally handle continuous action spaces, and POMCPOW \cite{sunberg2018pomcpow} is a further extension that plans based on particle trajectory that allows for flexible particle number representations of a given belief node.
However, further theoretical analyses of these algorithms would most likely require more sophisticated techniques and further assumptions.

\section{Numerical Experiments}
\label{sec:numerical}
\begin{figure}[t]
    \centering
    \subfigure[Laser Tag POMDP]{\includegraphics[width=0.49\textwidth]{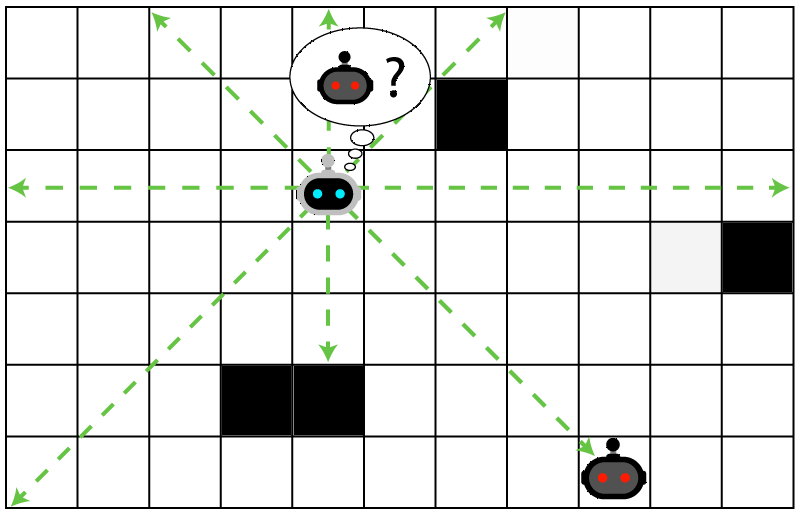}} 
    \subfigure[Light Dark POMDP]{\includegraphics[width=0.49\textwidth]{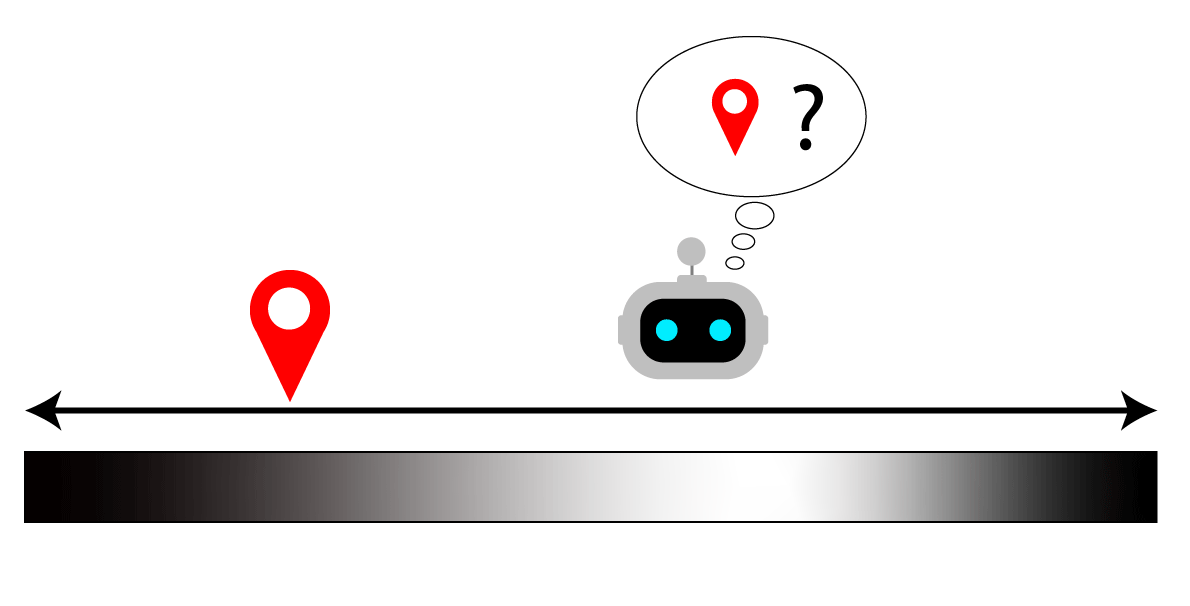}} 
    \subfigure[Sub Hunt POMDP]{\includegraphics[width=0.49\textwidth]{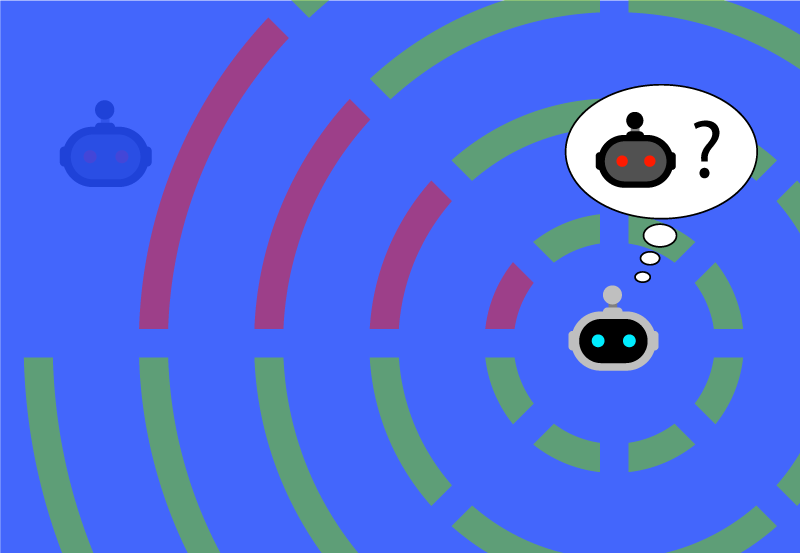}} 
    \subfigure[VDP Tag POMDP]{\includegraphics[width=0.49\textwidth]{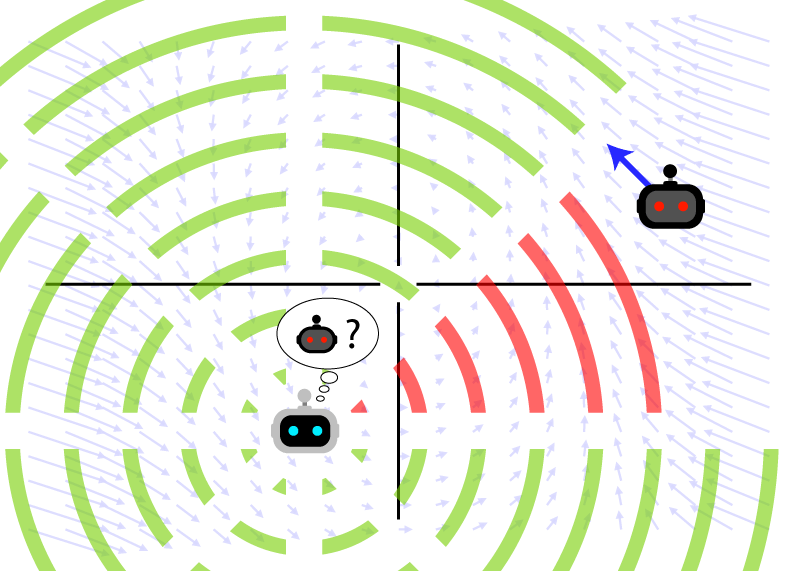}} 
    \caption{Illustration of the four environments we test our algorithms on: (a) Laser Tag, (b) Light Dark, (c) Sub Hunt, and (d) VDP Tag (discrete and continuous versions).}
    \label{fig:problems}
\end{figure}
Numerical simulation experiments were conducted in order to evaluate and compare the performances of our new simple algorithm, Sparse-PFT, along with other solvers.
In particular, we also ran experiments for Adaptive online packing-guided search (AdaOPS) \cite{wu2021adaptive}, a recent solver with practical performance and partial theoretical guarantees.
We also show performances of other hallmark algorithms like QMDP and POMCP along with random policy to demonstrate the need for continuous observation POMDP solvers that can handle more general assumptions.

The following sections contain descriptions of the evaluation problems along with discussion of solver performance.
In all five of the numerical experiments shown in \cref{fig:problems}, the POMDP solvers were limited to at most 1 second of planning time per step.
For closed-loop planning, whenever an observation was received from the environment, the belief was updated with a particle filter independent of the particle filter used in planning, and no part of the planning tree was saved for re-use on subsequent steps as done by \citet{silver2010pomcp}.
Since the observations received from the environment in this outer simulation loop were not generated from state particles in the filter, a larger number particles compared to \textsc{GenPF} are used to ensure likely states are present.
A total of 5000 simulation experiments were conducted for each configuration combination of solver and environment in order to obtain the Monte Carlo mean and standard error estimates for the Laser Tag and VDP tag environments, and 1000 simulation experiments for the Light Dark and Sub Hunt problems since planners typically yielded more consistent performances for these problems.
The tabular summary of all results is given in \cref{tab:experimental_results}, and corresponding figure summary of all results for different planning time allotments is given in \cref{fig:all-plots}.
We also vary belief particle count $C$ with \textsc{Simulate} calls held constant to demonstrate the effect of particle belief approximation resolution on the quality of the resulting policy in \cref{fig:Sparse-PFTParticleSweep} by using the optimized hyperparameters from \cref{tab:hyper} with 1000 simulation experiments for Laser Tag, Light Dark and Sub Hunt, and 100 for VDP Tag and Discrete VDP Tag.
The open source code for the experiments is built on the POMDPs.jl framework \cite{egorov2017pomdps}, and is available at: \texttt{github.com/WhiffleFish/PFTExperiments}.
The hyperparameter values used for the experiments are shown in \cref{app:experiments}.

\begin{figure}[p]
    \centering
    \includegraphics[width=0.95\textwidth]{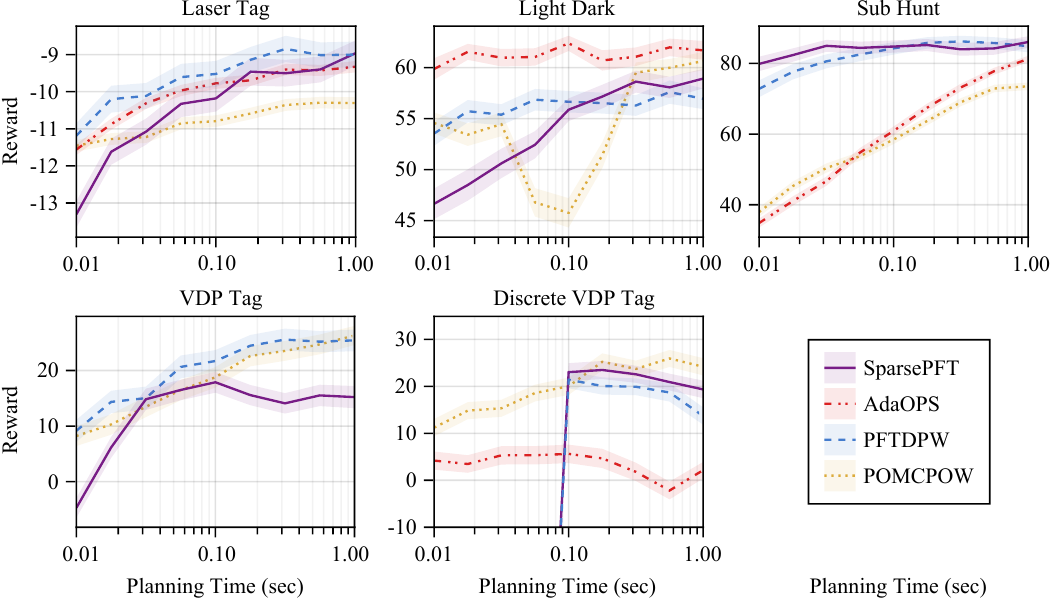}
    \caption{Comparative benchmark results summary over different planning time allotments. 
    Ribbons surrounding the Monte Carlo benchmark mean estimates indicate two standard error confidence bands. 
    Planning times are given in log scale.}
    \label{fig:all-plots}
\end{figure}
\begin{table}[p]
    \small
    \centering
    \begin{tabularx}{\linewidth}{lrlrlrl}
        \toprule
        & Laser Tag \makebox[0pt][l]{(D, D, D)} & & Light Dark \makebox[0pt][l]{(D, D, C)} & & Sub Hunt \makebox[0pt][l]{(D, D, C)}\\
        \midrule
        \textbf{Sparse-PFT} & \resultbest{-8.96}{0.17}{100}{1.0}     & \result{58.9}{0.5}{96}{0.96}      & \resultbest{86.0}{0.8}{100}{1.00}\\
        PFT-DPW   & \resultbest{-8.99}{0.07}{99}{0.99}     & \result{56.9}{0.5}{93}{0.93}      & \result{84.9}{0.9}{98}{0.98}\\
        POMCPOW   & \result{-10.3}{0.08}{81}{0.81}     & \result{60.6}{0.4}{98}{0.98}      & \result{73.5}{0.7}{78}{0.78}\\
        AdaOPS  & \resultbest{\dag-9.33}{0.08}{90}{0.90}  & \resultbest{61.7}{0.5}{100}{1.00} & \result{81.3}{0.6}{92}{0.92}\\
        QMDP               & \result{-10.4}{0.08}{80}{0.80}     & \result{3.28}{0.5}{24}{0.24}     & \result{28.0}{0.6}{1}{0.01}\\
        POMCP              & \result{-16.0}{0.09}{1}{0.01}      & \result{-14.86}{2.3}{1}{0.01}     & \result{30.0}{0.8}{3}{0.03}\\
        \textcolor{gray}{Random Policy} & \resultbad{-51.0}{0.18} & \resultbad{-85.0}{0.72}         & \resultbad{4.20}{0.27}\\
        
        \midrule
        & VDP Tag \makebox[0pt][l]{(C, C, C)}  & & VDP Tag$^D$ \makebox[0pt][l]{(C, D, C)} & \\
        \midrule
        \textbf{Sparse-PFT}              &  \result{15.2}{1.0}{58}{0.58}     & \result{19.3}{0.9}{78}{0.78}      \\
        PFT-DPW                &  \resultbest{25.4}{1.0}{97}{0.97}          & \result{13.7}{0.9}{53}{0.53}      \\
        POMCPOW                &  \resultbest{26.3}{0.9}{100}{1.00}         & \resultbest{24.2}{0.9}{100}{1.00} \\
        AdaOPS                          & \noresult{}                       & \result{2.0}{0.4}{1}{0.01}       \\
        \textcolor{gray}{Random Policy} & \resultbad{-66.8}{0.24}           & \resultbad{-66.6}{0.25}           \\
        \bottomrule
    \end{tabularx}
    \caption{Tabular comparative benchmark summary. 
    All results are given as (mean $\pm$ standard error) for a 1 second planning time allotment.
    The algorithm(s) with the best average performance within one standard error is shown in boldface for each experiment.
    The three letters after each problem name indicate whether the state, action, and observation spaces are continuous or discrete, respectively. 
    $\dag$ For AdaOPS, the result obtained by \citet{wu2021adaptive} on Laser Tag was $-8.31 \pm 0.18$, so this entry was also bolded.
    }
    \label{tab:experimental_results}
\end{table}

\subsection{Laser Tag}

\begin{wrapfigure}[15]{r}{0.5\textwidth}
    \centering
    \vspace{-1.4cm}
    \includegraphics[width=0.95\linewidth]{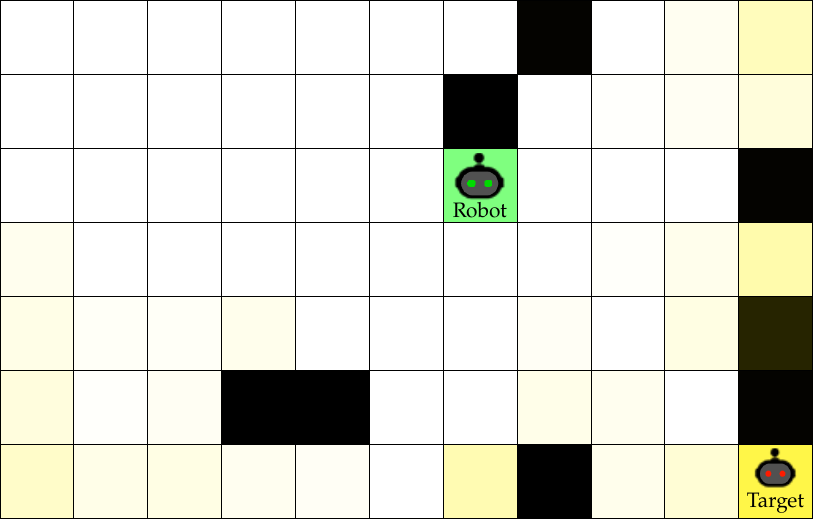}
    \caption{Example Sparse-PFT operating in Laser Tag: as indicated by the belief (yellow), the agent (green) has roughly located the evading robot at bottom right corner.}
    \label{fig:lasertag}
\end{wrapfigure}

The Laser Tag POMDP (\cref{fig:lasertag}) is taken from the DESPOT benchmarks~\cite{ye2017despot} wherein a robot is required to use laser sensors to localize with the ultimate goal of catching an evading robot. 
The agent’s laser sensors extend radially in 8 evenly spaced directions and each return a rounded sensed distance sampled from a normal distribution given by $\mathcal{N}(d, 2.5)$ where $d$ is the true distance to the nearest obstacle. 
Although the observation space is not continuous, it is sufficiently large (on the order of $10^6$) that most online solvers would have to treat this as close to continuous.

From the results, we find that the PFT methods outperform both POMCPOW and AdaOPS: PFT-DPW consistently outperforms both planners across different planning times, and Sparse-PFT outperforms all other planners with increased planning time.
This suggests that for Laser Tag, having a full particle belief approximation rather than dynamically varying particle size is helpful for keeping track of likely particle hypotheses.
Furthermore, this also suggests that the double progressive widening has diminishing returns when the action space has a fixed small size.

We note that POMCP particularly struggles on this problem compared to all other algorithms.
The large observation space forces the trees constructed by POMCP to become extremely shallow due to each unique sampled observation resulting in a new leaf node~\cite{sunberg2018pomcpow}. 
This hinders POMCP’s ability to develop a non-myopic multi-step plan and yield accurate action values, empirically showing the importance of particle weighting.
On the other hand, QMDP exhibits performance similar to the modern solvers. 
While the agent does not initially know its own location, it has sufficient information to localize using the laser sensor observations after some steps, and the evading robot behavior leads it reliably to the corners.
Thus, since this problem requires less active information gathering, the crude QMDP approximation performs well.

\subsection{Light Dark}
\begin{wrapfigure}[14]{R}{0.5\textwidth}
    \centering
    \vspace{-0.5cm}
    \includegraphics[width=\linewidth]{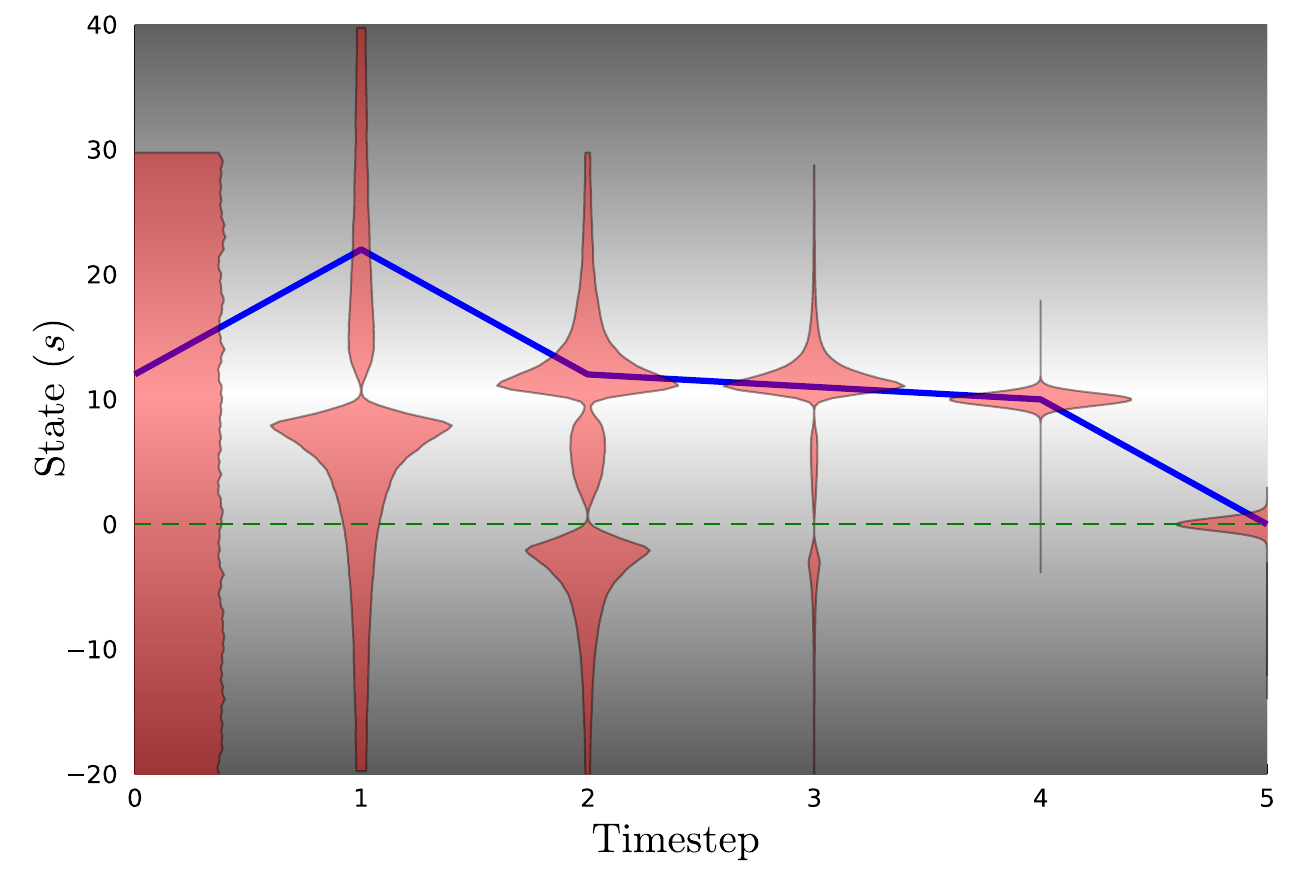}
    \vspace{-0.7cm}
    \caption{Example Sparse-PFT trajectory for Light Dark, successfully localizing in the light region to reach the goal in the dark region.}
    \vspace{-0.7cm}
    \label{fig:lightdark}
\end{wrapfigure}

The 1-dimensional Light Dark POMDP is designed to require active information gathering. The state is an integer representing the position of the agent and the action space is $\mathcal{A} = \{-10,-1,0,1,10\}$. Deterministic transitions are given by  $s' = s + a$. The reward, 
\begin{align}
R(s,a) = \begin{cases} 
    +100 \quad & \text{if } s =0,a=0 \\ 
    -100 \quad & \text{if }s \ne 0,a=0 \\ 
    -1 \quad & \text{otherwise} \\\end{cases},
\end{align}
dictates that the optimal policy drive the state to the origin as quickly as possible. 
Because the state is not immediately known, inferences over the true state must be made over noisy observations that grow in variance proportional to the agent’s distance from the light location at $s=10$. 
The observation distribution is $\mathcal{N}(s, |s-10|+\epsilon)$, where $\epsilon$ is some small constant included to prevent observation weights from reaching $+\infty$ due to a collapse to a Dirac distribution when the agent arrives at the light location.

The planners that yield the highest expected reward in the Light Dark domain roughly follow a 2-step plan: first localizing at the light location, then traveling down to the goal location. Essentially, the light location becomes a necessary subgoal. 
We can demonstrate this by creating a heuristic policy that initially steers towards the light region via certainty-equivalent control, and then takes action $a=-10$ down to the goal. 
This heuristic policy yields an expected reward of $62.0 \pm 0.19$ which is as good or better than any planner shown in \cref{tab:experimental_results}.

Surprisingly, higher planning times do not necessarily correspond to increasing expected rewards in the Light Dark domain. 
For AdaOPS, the solver converges to its peak expected reward with a planning time as low as 0.01 seconds leading to marginal improvement with further increases in planning time. 
Within a planning time interval of $[0.03,0.1]$ seconds, the performance of POMCPOW decreases, indicating that the planner becomes increasingly confident in a suboptimal plan. 
One possible source of this overconfidence is beliefs represented by a single particle. 
This same behavior becomes evident in PFT planners when the PFT planner is supplied with a single rollout value estimation. 
However, by increasing the number of sampled particles that are chosen as the true state in belief-based rollouts, the belief value estimate is granted lower variance and greater accuracy, effectively reducing the time spent exploring suboptimal branches of the constructed tree.

Because Light Dark requires costly information gathering, QMDP performs suboptimally. Specifically, regardless of belief distribution entropy, QMDP myopically steers directly towards the goal location but rarely commits to taking action 0 within the simulation horizon due to high state uncertainty.
POMCP also performs poorly due to the high branching factor introduced by the continuous observation space~\cite{sunberg2018pomcpow}.

\subsection{Sub Hunt}
\begin{wrapfigure}[16]{R}{0.5\textwidth}
    \centering
    \vspace{-0.5cm}
    \includegraphics[width=\linewidth]{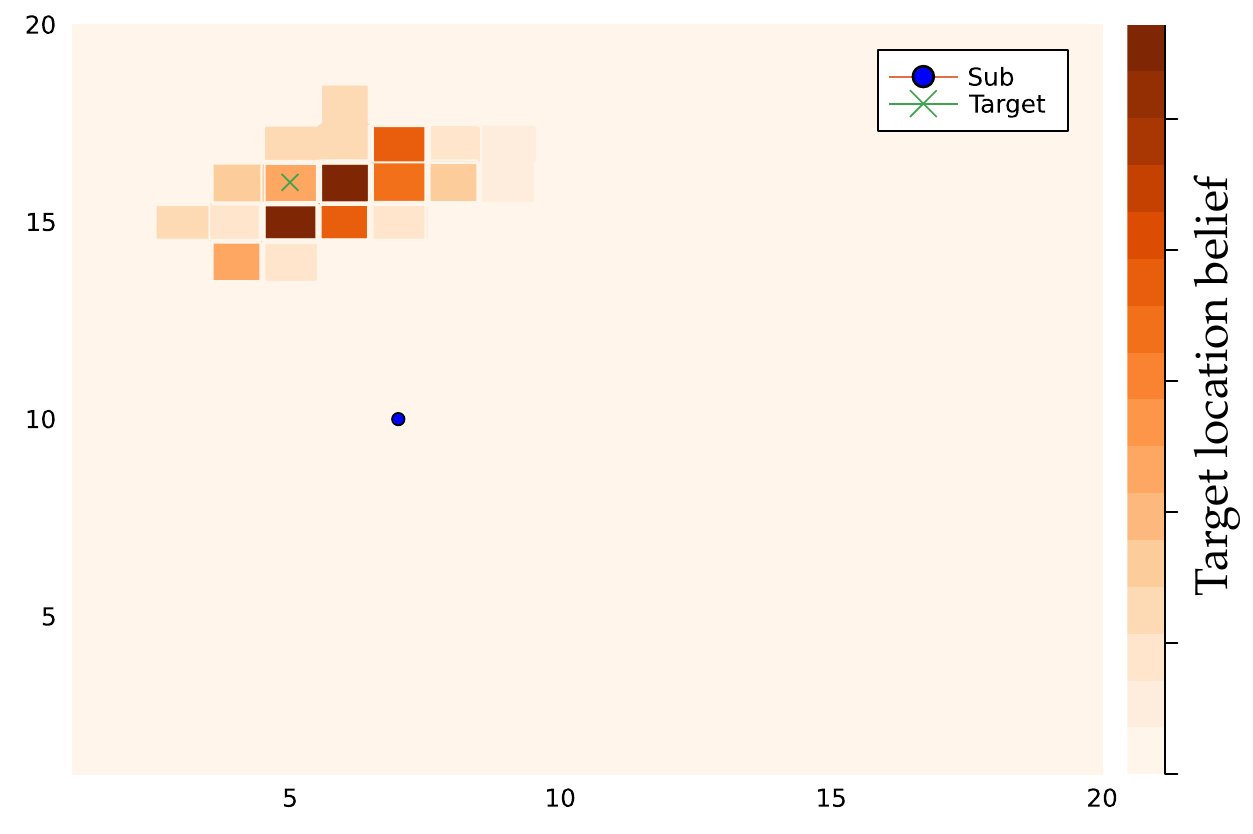}
    \vspace{-1.0cm}
    \caption{Example Sparse-PFT policy behavior for Sub Hunt, where the agent's submarine pursues the target with active information gathering. 
    }
    \label{fig:subhunt}
\end{wrapfigure}

In the Sub Hunt POMDP, from the POMCPOW benchmark~\cite{sunberg2018pomcpow}, the agent controls a submarine with the goal of finding and destroying an opposing submarine. The state space consists of the grid locations of both the agent and enemy submarines, a Boolean determining whether or not the enemy is aware of the agent’s presence, and the enemy’s goal direction $(\{1,\dots,20\}^4 \times \{\text{aware, unaware}\} \times \{N,S,E,W\})$. The agent is given the option to move three steps in any of the four cardinal directions, attack the enemy, or ping the enemy with active sonar while the enemy randomly chooses between taking two steps forward or one step diagonally forward.

In the Sub Hunt domain, PFT methods dominate all other planners over all planning times, with Sparse-PFT having a slight edge over all other planners. 
Because the state space is discrete, value iteration can be used to calculate $Q$-values for the fully observable MDP, and QMDP can be used for the rollout policy. With this strong belief-based rollout policy, both PFT-DPW and Sparse-PFT are able to construct nearly-optimal policies with planning times as low as 0.01 seconds, leading to no noticeable further improvement over longer planning times. 
Conversely, POMCPOW and AdaOPS have gradually increasing planning curves in \cref{fig:all-plots}, with AdaOPS nearly reaching the performance of PFT planners at 1 second and POMCPOW reaching an earlier inflection point, resulting in a final performance lower than the other three planners.
POMCP and QMDP perform poorly due to the large branching factor~\cite{sunberg2018pomcpow} and inability to perform costly information gathering, respectively.

\subsection{VDP Tag}
\begin{figure}[t]
    \centering
    \includegraphics[width=0.24\textwidth]{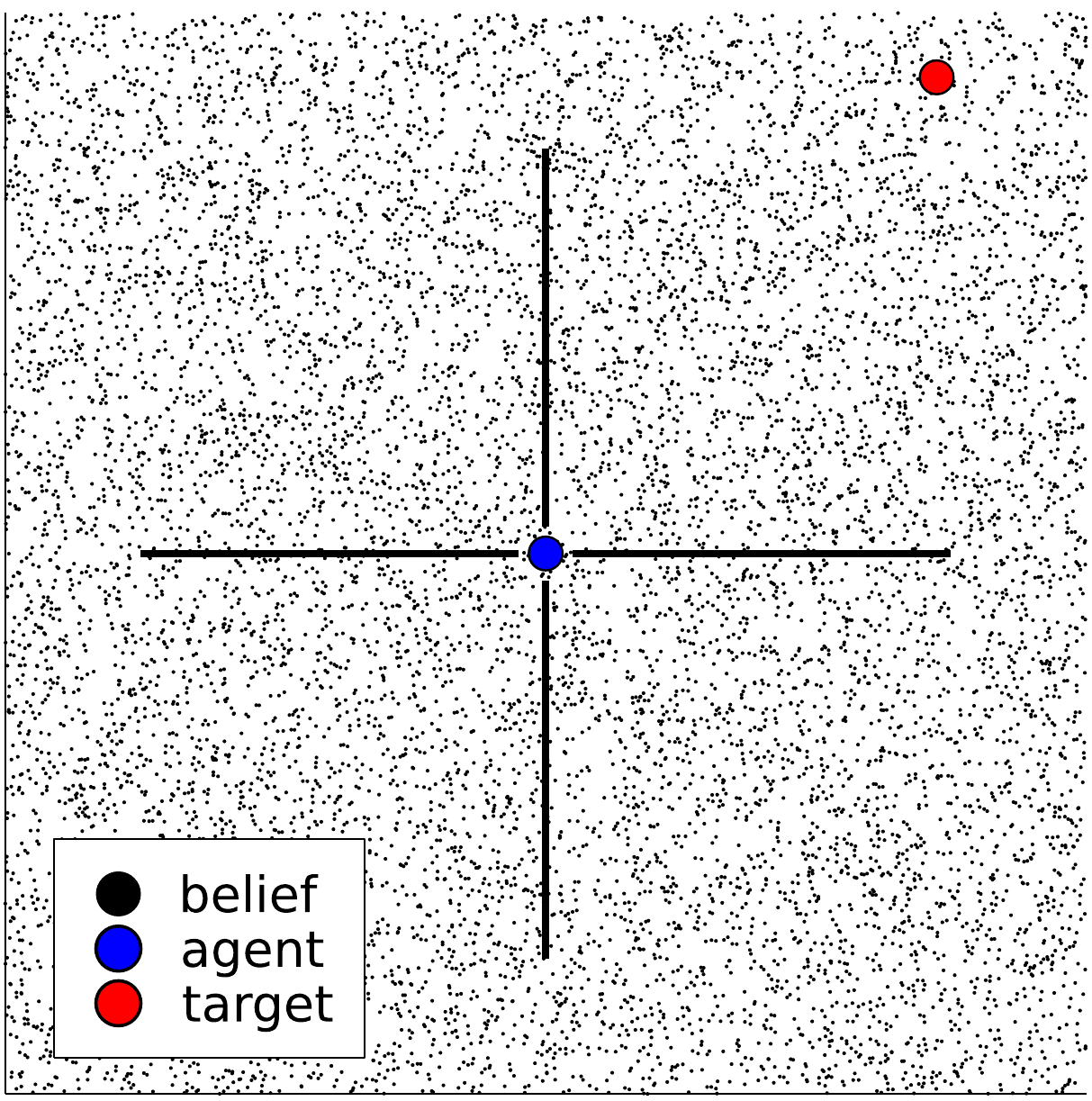}
    \includegraphics[width=0.24\textwidth]{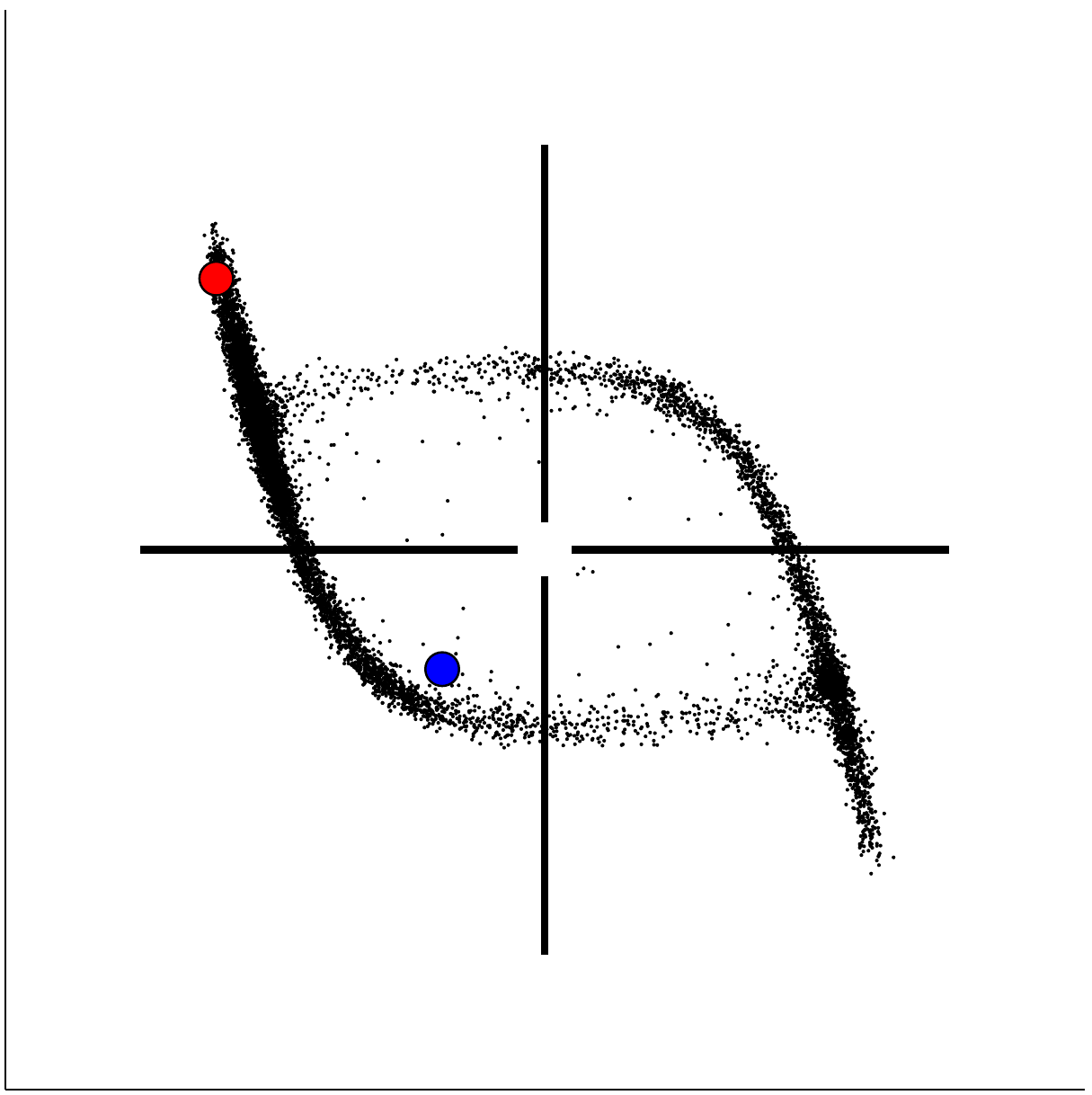}
    \includegraphics[width=0.24\textwidth]{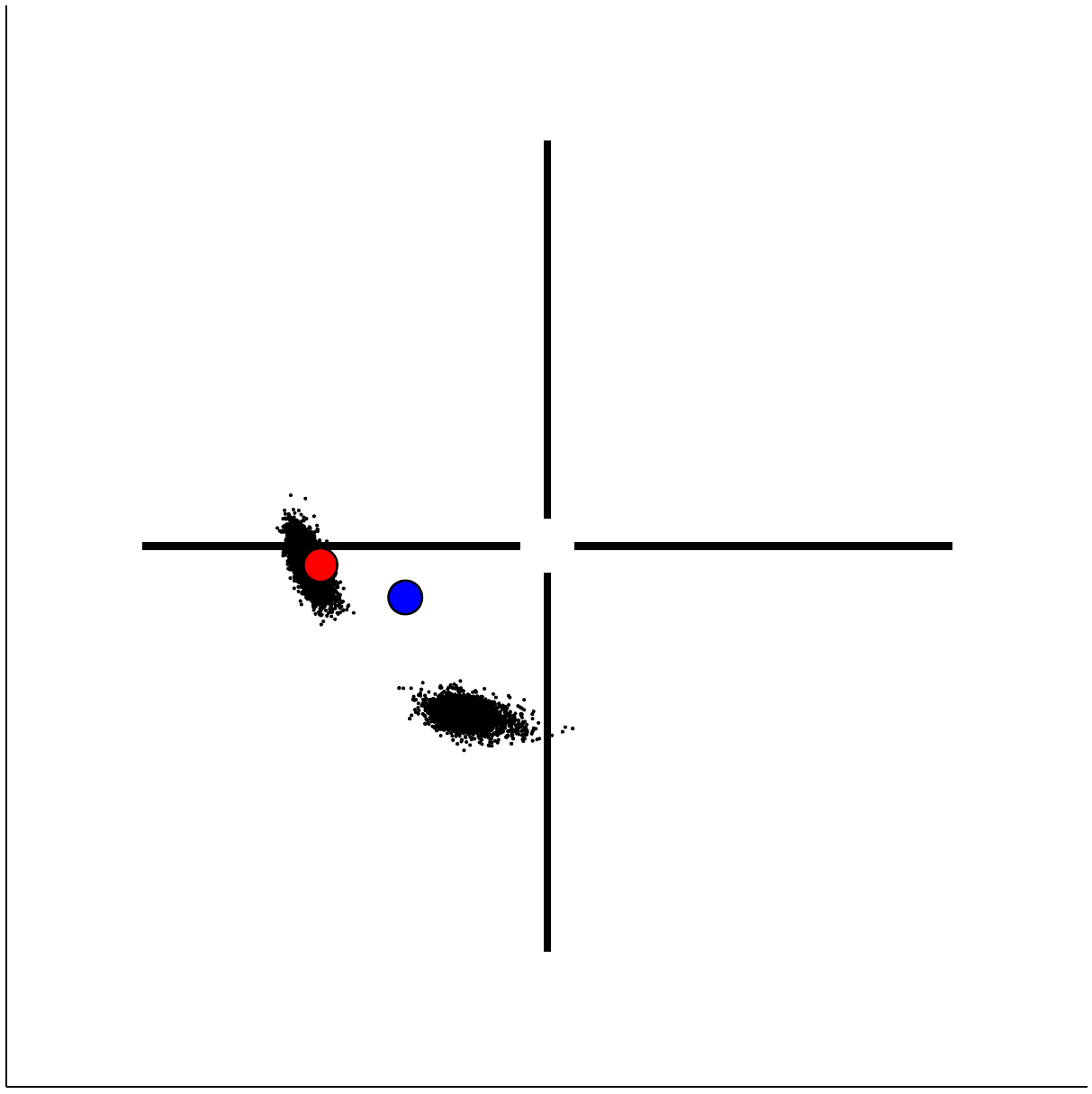}
    \includegraphics[width=0.24\textwidth]{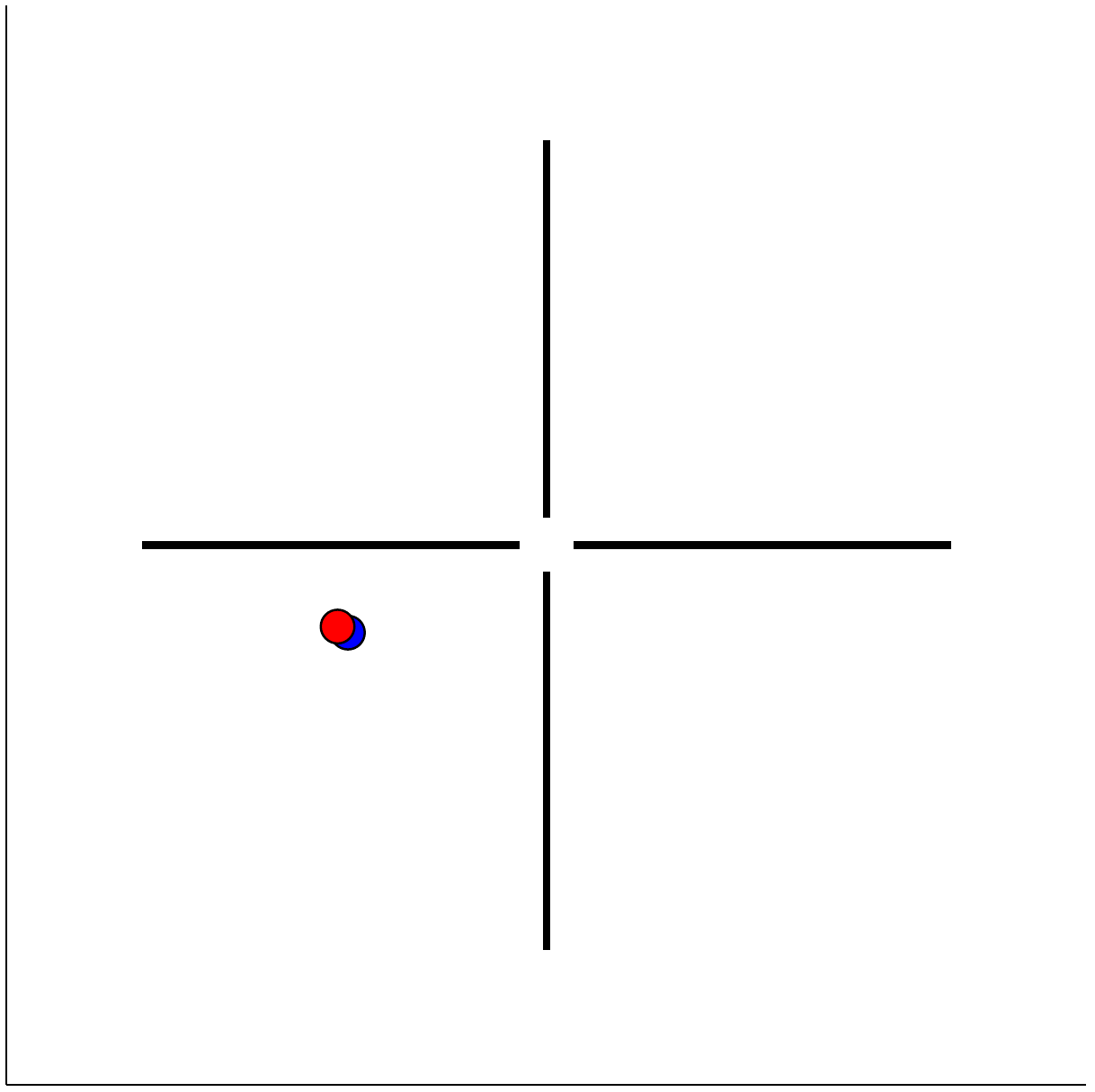}
    \caption{Example Sparse-PFT policy behavior for VDP Tag, which shows a successful localization of the target following the Van Der Pol dynamics, then a successful capture. }
    \label{fig:vdp_traj}
\end{figure}

\begin{wrapfigure}[19]{R}{0.5\textwidth}
    \centering
    \vspace{-0.5cm}
    \begin{tikzpicture}
    \begin{axis}[
        xmin = -4, xmax = 4,
        ymin = -4, ymax = 4,
        zmin = 0, zmax = 1,
        axis equal image,
        view = {0}{90},
    ]
        \addplot3[
            quiver = {
                u = {0.5*(x - x^3/3 - y)},
                v = {x/0.5},
                scale arrows = 0.10,
            },
            -stealth,
            domain = -4:4,
            domain y = -4:4,
        ] {0};
    \end{axis}
    \end{tikzpicture}
    \vspace{-0.3cm}
    \caption{An example Van Der Pol vector field $(\mu=0.5)$ which defines the target dynamics. 
    Unlike the agent, the target is not blocked by the barriers. }
    \label{fig:vanderpol}
\end{wrapfigure}
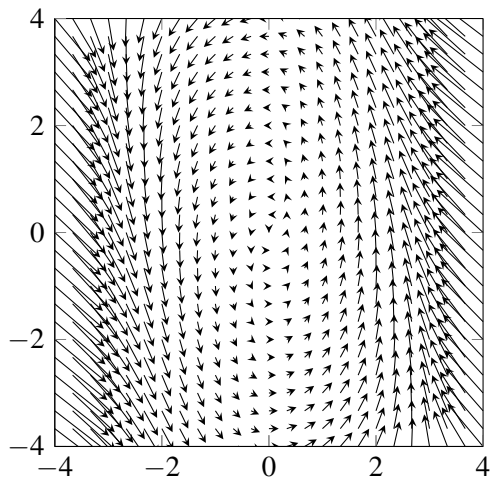


The Van Der Pol Tag (VDP Tag) POMDP formulation tasks the agent with moving through a two-dimensional space to catch an opponent whose dynamics are governed by the Van Der Pol differential equations,
\begin{align}
    \dot{x} = \mu\left(x-\frac{x^3}{3} - y\right), \quad \dot{y} = \frac{x}{\mu} ,   
\end{align}
for which we use scaling constant $\mu=2$. Because this problem has a continuous state space $(\mathcal{S} = \mathbb{R}^4)$, a continuous action space $(\mathcal{A}=[0,2\pi)\times\{0,1\})$ and a continuous observation space $(\mathcal{O} = \mathbb{R}^8)$, discrete value iteration is no longer admissible as input for a value estimation policy. 
AdaOPS is unable to handle continuous action spaces thus it is omitted from this benchmark. 
For the action-discretized VDP Tag, the available movement directions are 20 evenly spaced angles from $0^\circ$ to $360^\circ$.

The continuous VDP Tag domain is the first in which there exists a noticeable performance gap between Sparse-PFT and PFT-DPW, indicating that action progressive widening offers some utility over fixed widening in continuous action space problems.
Furthermore, PFT-DPW and POMCPOW have similar performances across different planning times, suggesting that the main challenge of VDP Tag is being able to handle continuity of state, action, and observation spaces, while the particle belief approximation resolution does not affect the performance as much.

For discrete VDP Tag, we come across two new peculiarities: AdaOPS performance decreases with increased planning time, and PFT methods perform orders of magnitude worse than other planners at very low allotted planning times. 
This is likely attributable to a large action space $(|\mathcal{A}| = 20)$ and a relatively expensive simulation function (RK4 integration). 
Because PFT methods propagate a collection of particles upon tree expansion, belief value estimates tend to be more accurate at the cost of added computation scaling linearly with the number of particles representing each belief. 
Thus, expanding all possible actions while using a computationally expensive simulator on all particles takes a long time, leading to very shallow trees.

\subsection{Experimental Validation of Particle Belief Approximation Convergence}
\begin{figure}[t]
    \centering
    \includegraphics[width=\textwidth]{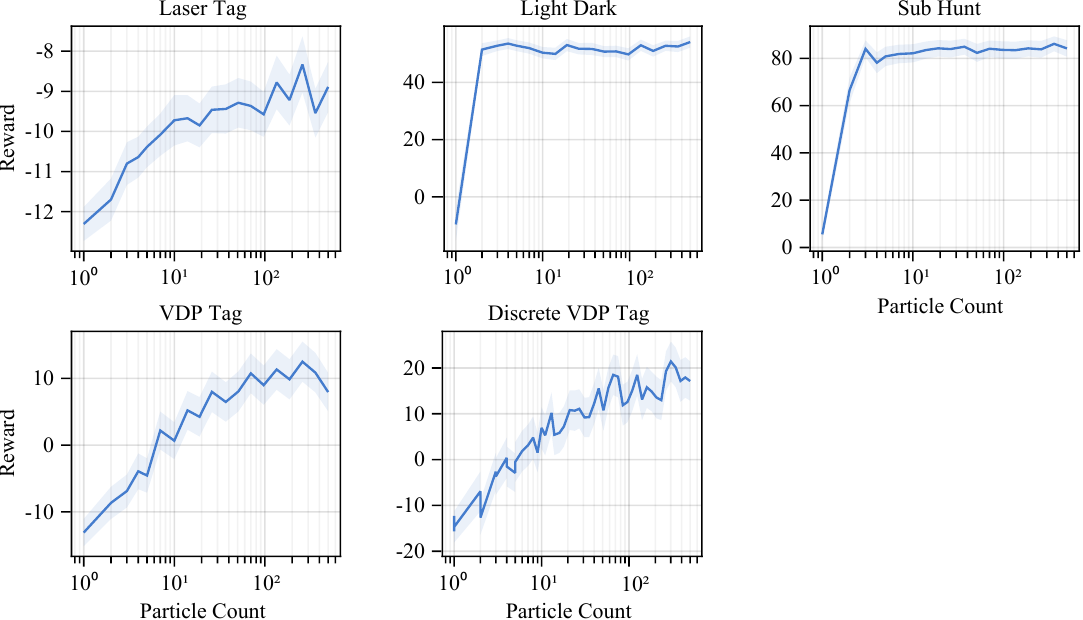}
    \caption{Sparse-PFT policy mean rewards with two standard error confidence band, varying belief particle count while holding number of tree search \textsc{Simulate} queries constant.}
    \label{fig:Sparse-PFTParticleSweep}
\end{figure}
In order to test the effect of particle belief approximation resolution, or the number of belief particles $C$, on planner performance, we vary $C$ while fixing the number of \textsc{Simulate} calls and using optimal hyperparameters found in \cref{app:experiments} for Sparse-PFT planner.
By increasing the the number of belief particles, the particle belief becomes a more accurate representation of the actual belief function, and should lead to a better optimal $Q$-value estimation.

Across all five problem domains, increasing the number of particles $C$ results in a roughly monotonic non-decreasing performance gain as shown in \cref{fig:Sparse-PFTParticleSweep}.
In particular, we see a gradual performance increase for Laser Tag, VDP Tag, and Discrete VDP Tag as we increase the number of particles, while Light Dark and Sub Hunt problems reach their performance capacity rather quickly at less than 10 particles.
However, when applying this principle to promote MDP algorithms into PB-MDP algorithms, the particle filtering transition generative model requires an extra $\mathcal{O}(C)$ computation and memory factor, so it is important to balance computational resource needs and value estimation accuracy when deploying these algorithms in practice.

These results offer two valuable insights.
First, the experiment outcomes are consistent with the general trend suggested by the theoretical analysis: increasing the number of particles results in improved performance, presumably because the $Q$-value estimates are more likely to be accurate as established in \cref{sec:pomdp-guarantees}.
Second, not all problems benefit the same way from increasing the number of particles.
Light Dark and Sub Hunt have rather simple state spaces, and adding more particles did not significantly improve the policy performance.
In contrast, the other three problems continually benefited from having increased resolution of belief approximation.
The belief approximation resolution is not the only factor contributing to the problem difficulty, but also other factors like action space cardinality and existence of simple rollout policies that perform well will contribute to the problem difficulty.


\section{Conclusion}
\label{sec:conclusion}
In this work, we formally show that optimality guarantees in a finite sample particle belief MDP (PB-MDP) approximation of a POMDP yields optimality guarantees in the original POMDP as well, which allows for simple yet powerful adaptations of MDP algorithms to solve POMDPs.
By proving that the Sparse Sampling-$\omega$ $Q$-value estimates are close to both optimal $Q$-values of the POMDP and PB-MDP with high probability, we conclude that the optimal $Q$-values of the POMDP and PB-MDP themselves are close with high probability.
This fundamental bridge between PB-MDPs and POMDPs allows us to adapt any sampling-based MDP algorithm of choice to a POMDP by solving the corresponding particle belief MDP approximation and to preserve the convergence guarantees in the POMDP.
The transformation only increases the computational complexity of transition generation by a factor of $\mathcal{O}(C)$ by using particle filtering-based generative models.
Our convergence result is not directly dependent on the size of the state space nor the observation space, but rather dependent on the R\'enyi divergence that links the probabilities concerning state and observation trajectories.
This motivates particle belief-based POMDP algorithms such as Sparse Particle Filter Tree (Sparse-PFT), which enjoys algorithmic simplicity, theoretical guarantees, and practicality.

There are many interesting avenues for future research.
First, the broader theoretical justification of more complex algorithms, such as POMCPOW and DESPOT-$\alpha$, still do not exist. 
Showing theoretical validity of these algorithms would help to close the gap between theory and practice even further.
In addition, as seen in our numerical experiments, the best performing algorithm varies across different types of benchmarks.
Further theoretical and empirical characterization of which algorithms are most effective for which problems could greatly aid practitioners.
Also, the particle number sweep suggests a method to characterize the difficulty of a POMDP problem, which may be of interest for both practitioners and researchers alike.
Lastly, while the algorithms presented here perform well in low dimensional continuous observation spaces, tree search for more difficult POMDPs, such as those with high dimensional observations~\cite{deglurkar2023compositional} and continuous/hybrid action spaces \cite{seiler2015online,lim2021voronoicdc,mern2021bayesian} is more difficult, and further analytical and empirical research is warranted.

\acks{This material is based upon work supported by a DARPA Assured Autonomy Grant, the SRC CONIX program, NSF CPS Frontiers, NSF National Robotics Initiative Grant No. 2133141, and the National Science Foundation Graduate Research Fellowship Program under Grant No. DGE 1752814 and DGE 2146752.
Any opinions, findings, and conclusions or recommendations expressed in this material are those of the authors and do not necessarily reflect the views of any sponsoring organizations.
}

\appendix
\setcounter{equation}{0}
\setcounter{lemma}{0}
\setcounter{theorem}{0}

\section{Proof of Theorem 1 - SN $d_{\infty}$-Concentration Bound}\label{app:theorem1}
\begin{theorem}[SN $d_{\infty}$-Concentration Bound]
  Let $\mathcal{P}$ and $\mathcal{Q}$ be two probability measures on the measurable space $(\mathcal{X},\mathcal{F})$ with $\mathcal{P}$ absolutely continuous w.r.t. $\mathcal{Q}$ and $d_{\infty}(\mathcal{P}||\mathcal{Q}) < +\infty$. Let $x_1,\ldots, x_N$ be $N$ independent identically distributed random variables with distribution $\mathcal{Q}$, and $f: \mathcal{X} \to \R$ be a bounded function ($\norm{f}_\infty < + \infty$). Then, for any $\lambda > 0$ and $N$ large enough such that $\lambda > \norm{f}_{\infty}  d_{\infty}(\mathcal{P}||\mathcal{Q})/\sqrt{N}$, the following bound holds with probability at least $1-3\exp(-N\cdot t^2(\lambda,N))$:
  \begin{align} 
    &|\E_{x\sim \mathcal{P}}[f(x)] - \tilde{\mu}_{\mathcal{P}/\mathcal{Q}}| \leq \lambda,\\
    &t(\lambda,N) \equiv \frac{\lambda}{\norm{f}_{\infty}d_{\infty}(\mathcal{P}||\mathcal{Q})}-\frac{1}{\sqrt{N}}.
  \end{align}
\end{theorem} 

\begin{proof} This proof follows similar proof steps as in Metelli \textit{et al.}~\shortcite{Metelli2018}. Since we have upper bounds on the infinite R\'enyi divergence $d_{\inft}(\mathcal{P}||\mathcal{Q})$, we can start from Hoeffding's inequality for bounded random variables applied to the regular IS estimator $\hat{\mu}_{\mathcal{P}/\mathcal{Q}} = \frac{1}{N}\sum_{i=1}^N w_{\mathcal{P}/\mathcal{Q}}(x_i)f(x_i)$, which is unbiased. While applying Hoeffding's inequality, we can view importance sampling on $f(x)$ weighted by $w_{\mathcal{P}/\mathcal{Q}}(x)$ as Monte Carlo sampling on $g(x) = w_{\mathcal{P}/\mathcal{Q}}(x)f(x)$, which is a function bounded by $\norm{g}_\inft = d_{\inft}(\mathcal{P}||\mathcal{Q})\norm{f}_{\inft}$:
\begin{align}
  \P\lrp{\hat{\mu}_{\mathcal{P}/\mathcal{Q}} - \E_{x \sim P}[f(x)] \geq \lambda} &= \P\lrp{\hat{\mu}_{\mathcal{P}/\mathcal{Q}} - \E_{x \sim Q}[\hat{\mu}_{\mathcal{P}/\mathcal{Q}}(x)f(x)] \geq \lambda}\\
  &\leq \exp\lrp{-\frac{2N^2\lambda^2}{\sum_{i=1}^N 2(d_{\inft}(\mathcal{P}||\mathcal{Q})\norm{f}_{\inft})^2}}\\
  &\leq \exp\lrp{-\frac{N\lambda^2}{d_{\inft}^2(\mathcal{P}||\mathcal{Q})\norm{f}_{\inft}^2}} \equiv \delta\\
  \P\lrp{|\hat{\mu}_{\mathcal{P}/\mathcal{Q}} - \E_{x \sim P}[f(x)]| \geq \lambda} & \leq 2\exp\lrp{-\frac{N\lambda^2}{d_{\inft}^2(\mathcal{P}||\mathcal{Q})\norm{f}_{\inft}^2}} = 2\delta
\end{align}
We prove a similar bound for the SN estimator $\tilde{\mu}_{\mathcal{P}/\mathcal{Q}} = \sum_{i=1}^N \tilde{w}_{\mathcal{P}/\mathcal{Q}}(x_i) f(x_i)$, which is a biased estimator. However, we need to take a step further and analyze the absolute difference, requiring us to split the difference up into two terms:
\begin{align}
  \P(|&\E_{x\sim \mathcal{P}}[f(x)] - \tilde{\mu}_{\mathcal{P}/\mathcal{Q}}|\geq \lambda) \\
  &\leq \P(\tilde{\mu}_{\mathcal{P}/\mathcal{Q}} - \E_{x\sim \mathcal{P}}[f(x)] \geq \lambda) + \P(\E_{x\sim \mathcal{P}}[f(x)] - \tilde{\mu}_{\mathcal{P}/\mathcal{Q}}\geq \lambda)\\
  &\leq \P(\tilde{\mu}_{\mathcal{P}/\mathcal{Q}} - \E_{x\sim \mathcal{Q}}[\tilde{\mu}_{\mathcal{P}/\mathcal{Q}}]\geq \tilde{\lambda}) + \P(\E_{x\sim \mathcal{Q}}[\tilde{\mu}_{\mathcal{P}/\mathcal{Q}}] - \tilde{\mu}_{\mathcal{P}/\mathcal{Q}}\geq \tilde{\lambda})\\
  &\leq \tilde{\delta} + \P(\E_{x\sim \mathcal{P}}[f(x)] - \tilde{\mu}_{\mathcal{P}/\mathcal{Q}}\geq \lambda)
\end{align}
The first term is bounded by $\tilde{\delta}$ from the above bound and recasting $\lambda$ to $\tilde{\lambda}$ to account for the bias of the SN estimator:
\begin{align}
  \tilde{\lambda} &= \lambda - \left|\E_{x\sim \mathcal{P}}[f(x)] - \E_{x\sim \mathcal{Q}}[\tilde{\mu}_{\mathcal{P}/\mathcal{Q}}]\right|\\
  \tilde{\delta} &= \exp\lrp{-\frac{N\tilde{\lambda}^2}{d_{\inft}^2(\mathcal{P}||\mathcal{Q})\norm{f}_{\inft}^2}}
\end{align}
Note that the bias term in the SN estimator is bounded by following through Cauchy-Schwarz inequality, closely following steps from \citeauthor{Metelli2018}~(\citeyear{Metelli2018}):
\begin{align}
  |\E_{x\sim \mathcal{P}}&[f(x)] - \E_{x\sim \mathcal{Q}}[\tilde{\mu}_{\mathcal{P}/\mathcal{Q}}]| = \left|\E_{x\sim \mathcal{Q}}[\hat{\mu}_{\mathcal{P}/\mathcal{Q}}-\tilde{\mu}_{\mathcal{P}/\mathcal{Q}}] \right| \leq \E_{x\sim \mathcal{Q}}[|\hat{\mu}_{\mathcal{P}/\mathcal{Q}}-\tilde{\mu}_{\mathcal{P}/\mathcal{Q}}|]\\
  &\leq \E_{x\sim \mathcal{Q}}\left| \frac{\sum_{i=1}^N w_{\mathcal{P}/\mathcal{Q}}(x_i) f(x_i)}{\sum_{i=1}^N w_{\mathcal{P}/\mathcal{Q}}(x_i) }-\frac{1}{N}\sum_{i=1}^N w_{\mathcal{P}/\mathcal{Q}}(x_i) f(x_i) \right|\\
  &= \E_{x\sim \mathcal{Q}}\lrs{\left| \frac{\sum_{i=1}^N w_{\mathcal{P}/\mathcal{Q}}(x_i) f(x_i)}{\sum_{i=1}^N w_{\mathcal{P}/\mathcal{Q}}(x_i) } \right| \left| 1-\frac{\sum_{i=1}^N w_{\mathcal{P}/\mathcal{Q}}(x_i)}{N} \right|}\\
  &\leq \E_{x\sim \mathcal{Q}}\lrs{ \lrp{\frac{\sum_{i=1}^N w_{\mathcal{P}/\mathcal{Q}}(x_i) f(x_i)}{\sum_{i=1}^N w_{\mathcal{P}/\mathcal{Q}}(x_i) } }^2}^{1/2} \E_{x\sim \mathcal{Q}}\lrs{\lrp{1-\frac{\sum_{i=1}^N w_{\mathcal{P}/\mathcal{Q}}(x_i)}{N}}^2 }^{1/2}\\
  &\leq \norm{f}_{\inft}  \sqrt{\frac{d_{2}(\mathcal{P}||\mathcal{Q}) - 1}{N}} \leq \norm{f}_{\inft}  \frac{d_{\inft}(\mathcal{P}||\mathcal{Q})}{\sqrt{N}}
\end{align}
In the last step, the first term is bounded by $\norm{f}_{\inft}$ as the function is bounded, and the second term is bounded by the fact that we can bound the square root of variance with the supremum squared, where we square it for the convenience of the definition of $t(\lambda,N)$ later on such that the $1/\sqrt{N}$ factor is nicely separated. We use the assumption in the theorem that $N$ is chosen large enough that $\lambda > \norm{f}_{\inft}  d_{\inft}(\mathcal{P}||\mathcal{Q})/\sqrt{N}$ to bound the $\tilde{\delta}$ term:
\begin{align}
  \tilde{\delta}&\leq \exp\lrp{-\frac{N(\lambda - \norm{f}_{\inft}  d_{\inft}(\mathcal{P}||\mathcal{Q})/\sqrt{N})^2}{d_{\inft}^2(\mathcal{P}||\mathcal{Q})\norm{f}_{\inft}^2}}\\
  &= \exp\lrp{-N\lrp{\frac{\lambda - \norm{f}_{\inft}  d_{\inft}(\mathcal{P}||\mathcal{Q})/\sqrt{N}}{\norm{f}_{\inft}d_{\inft}(\mathcal{P}||\mathcal{Q})}}^2 }\\
  &\equiv \exp\lrp{-N\cdot t^2(\lambda,N) }
\end{align}
Here, we define $t(\lambda,N) \equiv \frac{\lambda}{\norm{f}_{\inft}d_{\inft}(\mathcal{P}||\mathcal{Q})} - \frac{1}{\sqrt{N}}$, which satisfies $0 < t(\lambda,N) \leq \frac{\lambda}{\norm{f}_{\inft}d_{\inft}(\mathcal{P}||\mathcal{Q})}$. The second term can be bounded similarly by rebounding the bias term with $\tilde{\lambda}$, using symmetry and Hoeffding's inequality:
\begin{align}
  \P(\E_{x\sim \mathcal{P}}[f(x)] - \tilde{\mu}_{\mathcal{P}/\mathcal{Q}}\geq \lambda) &\leq \P(\E_{x\sim \mathcal{Q}}[\tilde{\mu}_{\mathcal{P}/\mathcal{Q}}] - \tilde{\mu}_{\mathcal{P}/\mathcal{Q}}\geq \tilde{\lambda})\\
  &\leq \P(|\E_{x\sim \mathcal{Q}}[\tilde{\mu}_{\mathcal{P}/\mathcal{Q}}] - \tilde{\mu}_{\mathcal{P}/\mathcal{Q}}|\geq \tilde{\lambda}) \leq 2\tilde{\delta}
\end{align}
Thus, we obtain the following bound:
\begin{align}
  \P(|\E_{x\sim \mathcal{P}}[f(x)] - \tilde{\mu}_{\mathcal{P}/\mathcal{Q}}|\geq \lambda) &\leq 3\exp(-N\cdot t^2(\lambda,N))
\end{align}
\end{proof}

\section{Proof of Lemma 1 (Continued) - Particle Likelihood SN Estimator Convergence}\label{app:lemma1}
In the main paper, we show that $\tilde{\mu}_{\bar{b}_d}[f]$ is an SN estimator of $\E_{s\sim {b_d}}[f(s)]$. We apply the concentration inequality proven in \cref{thm:sn} to finish the proof of \cref{lemma:pwl_sn}.
\begin{lemma}[Particle Likelihood SN Estimator Convergence]
  Suppose a function $f$ is bounded by a finite constant $\norm{f}_{\infty} \leq f_{\max}$, and a particle belief state $\bar{b}_d=\{(s_{d,i},w_{d,i})\}_{i=1}^C$ at depth $d$ represents $b_d$ with particle likelihood weighting that is recursively updated as $w_{d,i} = w_{d-1,i} \cdot \mathcal{Z}(o_{d}\mid~a,s_d)$.
  Then, for all $d = 0,\ldots,D-1$, the following weighted average is the SN estimator of $f$ under the belief $b_d$ corresponding to the actions $\{a_n\}_{n=0}^{d-1}$ and observations $\{o_n\}_{n=1}^d$, for all beliefs $b_d \in B$ that are realizable given the initial belief $b_0$:
  \begin{align}
      \tilde{\mu}_{\bar{b}_d}[f] &= \frac{\sum_{i = 1}^C w_{d,i}f(s_{d,i})}{\sum_{i = 1}^C w_{d,i}},
  \end{align}
  and the following concentration bound holds with probability at least $1-3\exp(-C\cdot t_{\max}^2(\lambda,C))$,
  \begin{align}
    &|\E_{s\sim {b_d}}[f(s)] - \tilde{\mu}_{\bar{b}_d}[f] | \leq \lambda,\\
    &t_{\max}(\lambda,C) \equiv \frac{\lambda}{f_{\max} d_{\inft}^{\max}} - \frac{1}{\sqrt{C}}.
  \end{align}
\end{lemma}

\begin{proof} 
In this proof, we will take advantage of the fact that the state particles trajectories $\{s_n\}_1,\ldots$ $\{s_n\}_C$ of depth $d$ are independent of each other, as \textsc{GenPF} independently generates each state sequence $i$ according to the transition density $\mathcal{T}$.

In the subsequent analysis, we abbreviate some terms of interest with the following notation:
\begin{align}
  \mathcal{T}_{1:d}^i &\equiv \prod_{n=1}^d \mathcal{T}(s_{n,i}\mid s_{n-1,i},a_n);\;\; \mathcal{Z}_{1:d}^{i} \equiv \prod_{n=1}^d \mathcal{Z}(o_{n}\mid a_{n},s_{n,i}).
\end{align}
Here $d$ denotes the depth, $i$ denotes the index of the state sample.
Intuitively, $\mathcal{T}_{1:d}^i$ is the transition density of state sequence $i$ from the root node to depth $d$, and $\mathcal{Z}_{1:d}^{i}$ is the conditional density of observation sequence state sequence $i$ from the root node to depth $d$.
Additionally, $b_d^i$ denotes $b_d(s_{d,i})$ and $w_{d,i}$ the weight of $s_{d,i}$.

First, we show that $\tilde{\mu}_{\bar{b}_d}[f]$ is an SN estimator of $\E_{s\sim {b_d}}[f(s)]$. By following the recursive belief update, the belief term can be fully expanded:
\begin{align}
  b_{D-1}(s_{D-1}) &= \frac{\int_{S^{D-1}} (\mathcal{Z}_{1:D-1})(\mathcal{T}_{1:D-1}) b_{0} ds_{0:D-2}}{\int_{S^{D}} (\mathcal{Z}_{1:D-1})(\mathcal{T}_{1:D-1}) b_{0} ds_{0:D-1}}
\end{align} 
Then, $\E_{s\sim {b_d}}[f(s)]$ is equal to the following:
\begin{align}
  &\E_{s\sim {b_d}}[f(s)] = \int_S f(s_{D-1}) b_{D-1} ds_{D-1} = \frac{\int_{S^{D}} f(s_{D-1}) (\mathcal{Z}_{1:D-1})(\mathcal{T}_{1:D-1}) b_{0} ds_{0:D-1}}{\int_{S^{D}} (\mathcal{Z}_{1:D-1})(\mathcal{T}_{1:D-1}) b_{0} ds_{0:D-1}}
\end{align}
We approximate the $\E_{s\sim {b_d}}[f(s)]$ function with importance sampling by using problem requirement \ref{req:generate}, where the target density is $b_{D-1}$. 
First, we sample the sequences $\{s_{n,i}\}$ according to the joint probability $(\mathcal{T}_{1:D-1}) b_{0}$.
Afterwards, we weight the sequences by the corresponding observation density $\mathcal{Z}_{1:D-1}$, obtained from the generated observation sequences $\{o_{n}\}$.
Normally, these generated observation sequences through \textsc{GenPF} will be correlated.
For now, we treat the observation sequences $\{o_{n}\}$ as fixed.

Applying the importance sampling formalism to our system for all depths $d=0,\ldots,D-1$, $\mathcal{P}^d$ is the normalized measure incorporating the probability of the observation sequence conditioned on the state sequence $i$ and action sequence until the node at depth $d$, and $\mathcal{Q}^d$ is the measure of the state sequence. We can think of $\mathcal{P}^d$ corresponding to the observation sequence $\{o_{n}\}$.
\begin{align}
  &\mathcal{P}^d = \mathcal{P}_{\{a_n,o_{n}\}}^d(\{s_{n,i}\}) = \frac{(\mathcal{Z}_{1:d}^{i}) (\mathcal{T}_{1:d}^i) b_{0}^i}{\int_{S^{d+1}} (\mathcal{Z}_{1:d}) (\mathcal{T}_{1:d}) b_{0} ds_{0:d}}\\
  &\mathcal{Q}^d = \mathcal{Q}_{\{a_n\}}^d(\{s_{n,i}\}) = (\mathcal{T}_{1:d}^i) b_{0}^i\\
  &w_{\mathcal{P}^d/\mathcal{Q}^d}(\{s_{n,i}\}) = \frac{(\mathcal{Z}_{1:d}^{i}) }{\int_{S^{d+1}} (\mathcal{Z}_{1:d}) (\mathcal{T}_{1:d}) b_{0} ds_{0:d}} 
\end{align}
Here, the integral to calculate the normalizing constant is taken over $S^{d+1}$, the Cartesian product of the state space $S$ over $d+1$ steps.

The weighing step is done by updating the self-normalized weights given in \textsc{GenPF} algorithm. 
We define $w_{d,i}$ and $r_{d,i}$ as the weights and rewards obtained at step $d$ for state sequence $i$ from \textsc{GenPF} simulation. With our recursive definition of the empirical weights, we obtain the full weight of each state sequence $i$ for a fixed observation sequence:
\begin{align}
  w_{d,i} &= w_{d-1,i}\cdot \mathcal{Z}(o_{d}\mid a_{d},s_{d,i}) \propto \mathcal{Z}_{1:d}^{i}.
\end{align}

Realizing that the marginal observation probability is independent of indexing by $i$, we show that $\tilde{\mu}_{\bar{b}_d}[f]$ is an SN estimator of $\E_{s\sim {b_d}}[f(s)]$:
\begin{align}
  \tilde{\mu}_{\bar{b}_d}[f] &= \frac{\sum_{i = 1}^C (\mathcal{Z}_{1:d}^{i}) f(s_{d,i})}{\sum_{i = 1}^C (\mathcal{Z}_{1:d}^{i}) } = \frac{\sum_{i = 1}^C \frac{(\mathcal{Z}_{1:d}^{i})}{\int_{S^{D}} (\mathcal{Z}_{1:d}) (\mathcal{T}_{1:d}) b_{0} ds_{0:d}} f(s_{d,i})}{\sum_{i = 1}^C \frac{(\mathcal{Z}_{1:d}^{i})}{\int_{S^{D}} (\mathcal{Z}_{1:d}) (\mathcal{T}_{1:d}) b_{0} ds_{0:d}} }\\
  &= \frac{\sum_{i = 1}^C w_{\mathcal{P}^{d}/\mathcal{Q}^{d}}(\{s_{n,i}\}) f(s_{d,i})}{\sum_{i = 1}^C w_{\mathcal{P}^{d}/\mathcal{Q}^{d}}(\{s_{n,i}\}) } = \sum_{i = 1}^C \tilde{w}_{\mathcal{P}^{d}/\mathcal{Q}^{d}}(\{s_{n,i}\}) f(s_{d,i})
\end{align}
Since $\{s_n\}_1,\ldots, \{s_n\}_C$ are independent identically distributed random variable sequences of depth $d$, and $f$ is a bounded function, we can apply the SN concentration bound in Theorem $\ref{thm:sn}$ to obtain the concentration inequality. 
Since $d_{\inft}(\mathcal{P}^{d}||\mathcal{Q}^{d})$ is bounded by $d_{\inft}^{\max}$ a.s., we can bound the resulting $t_{d}(\lambda,C)$ by $t_{\max}(\lambda,C)$ a.s.:
\begin{align}
  t_{d}(\lambda,C) &= \frac{\lambda}{f_{\max}d_{\inft}(\mathcal{P}^{d}||\mathcal{Q}^{d})} - \frac{1}{\sqrt{C}} \geq \frac{\lambda}{f_{\max}d_{\inft}^{\max}} - \frac{1}{\sqrt{C}} \equiv t_{\max}(\lambda,C)
\end{align}
This means that for all $d$, we can bound $t_d(\lambda,C) \geq t_{\max}(\lambda,C)$. 
Thus, bounding the concentration inequality probability with $t_{\max}(\lambda,C)$ for any step $d$ is justified when we prove \cref{lemma:induction} later. 
This probabilistic bound holds for any choice of $\{o_{n}\}$, where $\{o_{n}\}$ could be a sequence of random variables correlated with any elements of $\{s_{n,i}\}$. Thus, for any $\{o_{n}\}$,
\begin{align}\label{eq:leafnode}
    |\E_{s\sim {b_d}}[f(s)] - \tilde{\mu}_{\bar{b}_d}[f] | \leq \lambda
\end{align}
holds with probability at least $1-3\exp(-C\cdot t_{\max}^2(\lambda,C))$.
\end{proof}

\section{Proof of Lemma 2 (Continued) - Sparse Sampling-$\omega$ $Q$-Value Coupled Convergence}\label{app:lemma2}
\begin{lemma}[Sparse Sampling-$\omega$ Estimator $Q$-Value Coupled Convergence]
  For all $d = 0,\ldots,D-1$ and $a$, the following bounds hold with probability at least $1 - 6|A|(4|A|C)^{D}\exp(-C\cdot \tilde{t}^2)$:
  \begin{align} 
    &|Q^*_{\mathbf{P},d}(b_d,a) - \hat{Q}_{\omega,d}^*(\bar{b}_d,a)| \leq \alpha_d,\; \alpha_{d} = \lambda + \gamma\alpha_{d+1},\; \alpha_{D-1} = \lambda,\label{eq:p_to_ssw_app}\\
    &|Q^*_{\mathbf{M}_{\mathbf{P}},d}(\bar{b}_d,a)- \hat{Q}_{\omega,d}^*(\bar{b}_d,a)| \leq \beta_d,\; \beta_{d} = \gamma(\lambda + \beta_{d+1}),\; \beta_{D-1} = 0, \label{eq:ssw_to_m_app}\\
    &t_{\max}(\lambda,C) = \frac{\lambda}{4V_{\max} d_{\inft}^{\max}} - \frac{1}{\sqrt{C}},\; \tilde{t} = \min\{t_{\max},\lambda/4\sqrt{2}V_{\max}\}
  \end{align}
\end{lemma}
Before we proceed with the proof, note that in our definition of $t_{\max}$, we set the maximum of the $f_{\max}$ to be equal to $4V_{\max}$.
While this may seem very conservative to bound most reasonable functions resulting from reward and value estimation with 4 times the $V_{\max}$, it serves to uniformly bound the probability for each of the SN estimator terms with convenient coefficients.
Furthermore, individual concentration bounds may be adjusted to account for this generous upper bound by multiplying a factor in front of $\lambda$.

\textbf{POMDP Value Convergence:} We split the difference between the SN estimator and $Q_{\mathbf{P}}^*$ into two terms, the reward estimation error (A) and the next-step value estimation error (B):

\begin{align}
  |Q_{\mathbf{P},d}^*(b_d,a) - \hat{Q}^*_{\omega,d}(\bar{b}_d,a) | &\leq \undb{\left| \E_{\mathbf{P}}[R(s_d,a)\mid b_d] - \frac{\sum_{i = 1}^C w_{d,i}r_{d,i}}{\sum_{i = 1}^C w_{d,i}}\right|}{(A)} \\ \notag
  &\;\;\;\;+ \gamma \undb{\left| \E_{\mathbf{P}}[V_{\mathbf{P},d+1}^*(b_dao)\mid b_d] - \frac{1}{C}\sum_{i=1}^{C} \hat{V}^*_{\omega,d+1}(\bar{b}_{d+1}'^{[I_i]}) \right|}{(B)}
\end{align}

Here, the $I_i$ notation represents that random variables $I_i$ are sampled $C$ times from the finite discrete distribution $p_{w,d}$ with probability mass $p_{w,d}(I=i) = (w_{d,i}/\sum_j w_{d,j})$, and particle belief state $\bar{b}_{d+1}'^{[I_i]}$ is updated by an observation generated from $s_{d,I_i}$.
This reflects the fact that \textsc{GenPF} randomly selects a state particle $s_o$ with probability $w_{d,o}/\sum_j w_{d,j}$ $C$ times independently to generate a new observation for the next step particle belief state.
Similarly, a particle belief state $\bar{b}_{d+1}'^{[i]}$ is updated by an observation generated from $s_{d,i}$, which is a notation we will use to represent beliefs that are generated through iterating upon each state particle $s_{d,i}$.

To prove the base case $d=D-1$, we note that we only need to bound the first term (A) since $d=D-1$ corresponds to the leaf node of Sparse Sampling-$\omega$ tree and no further next step value estimation is performed:
\begin{align}
  |Q^*_{\mathbf{P},D-1}(b_{D-1},a) - \hat{Q}^*_{\omega,D-1}(\bar{b}_{D-1},a) | \leq \undb{\left| \E_{\mathbf{P}}[R(s_{D-1},a)\mid b_{D-1}] -  \frac{\sum_{i = 1}^C w_{D-1,i}r_{D-1,i}}{\sum_{i = 1}^C w_{D-1,i}}\right|}{(A)}.
\end{align}
This term is simply a particle likelihood weighted average estimation term where the function is $R(\cdot, a)$, and does not need any inductive step. 
Below, we will show how to bound both terms (A) and (B), so the base case proof naturally follows from the proof of concentration bound for (A).

For (A), we use the particle likelihood SN concentration bound in \cref{lemma:pwl_sn} to obtain the bound $\frac{R_{\max}}{4V_{\max}}\lambda$; rather than bounding $R$ with $4V_{\max}$ in this step, we instead bound $R$ with $R_{\max}$ and then augment $\lambda$ to $\frac{R_{\max}}{4V_{\max}}\lambda$ in order to obtain the same uniform $t_{\max}$ factor as the other steps. This choice of bound is made to effectively combine the $\lambda$ terms when we add (A) and (B).
This also covers the base case since $\alpha_{D-1} = \lambda \geq \frac{R_{\max}}{4V_{\max}}\lambda$.

For (B), we use the triangle inequality repeatedly to separate it into four terms; (1) the importance sampling error bounded by $\lambda/4$, (2) the Monte Carlo weighted sum approximation error bounded by $\lambda/4$, (3) the Monte Carlo next-step integral approximation error bounded by $\lambda/2$, and (4) the inductive function estimation error bounded by $\alpha_{d+1}$:
{\small
\begin{align}
    &\undb{\left| \E_{\mathbf{P}}[V_{\mathbf{P},d+1}^*(b_dao)\mid b_d] - \frac{1}{C}\sum_{i=1}^{C} \hat{V}^*_{\omega,d+1}(\bar{b}_{d+1}'^{[I_i]}) \right|}{(B)}\\
    &\leq \undb{\left| \E_{\mathbf{P}}[V_{\mathbf{P},d+1}^*(b_dao)\mid b_d] - \frac{\sum_{i = 1}^C w_{d,i}\mathbf{V}_{\mathbf{P},d+1}^*(b_d,a)^{[i]}}{\sum_{i = 1}^C w_{d,i}} \right|}{(1) Importance sampling error}  + \undb{\left| \frac{\sum_{i = 1}^C w_{d,i}\mathbf{V}_{\mathbf{P},d+1}^*(b_d,a)^{[i]}}{\sum_{i = 1}^C w_{d,i}} - \frac{1}{C}\sum_{i=1}^{C}\mathbf{V}_{\mathbf{P},d+1}^*(b_d,a)^{[I_i]} \right|}{(2) MC weighted sum approximation error} \notag\\
    &\quad + \undb{\left| \frac{1}{C}\sum_{i=1}^{C}\mathbf{V}_{\mathbf{P},d+1}^*(b_d,a)^{[I_i]} - \frac{1}{C}\sum_{i=1}^{C}V_{\mathbf{P},d+1}^*(b_dao^{[I_i]}) \right|}{(3) MC next-step integral approximation error} + \undb{\left|\frac{1}{C}\sum_{i=1}^{C}V_{\mathbf{P},d+1}^*(b_dao^{[I_i]}) - \frac{1}{C}\sum_{i=1}^{C} \hat{V}^*_{\omega, d+1}(\bar{b}_{d+1}'^{[I_i]}) \right|}{(4) Inductive function estimation error}. \notag\\
    &\leq \undb{\frac{1}{4}\lambda}{(1)} + \undb{\frac{1}{4}\lambda}{(2)} + \undb{\frac{1}{2}\lambda}{(3)} + \undb{\alpha_{d+1}}{(4)}.
\end{align}
}
The following subsections justify how each error term is bounded. 
\begin{enumerate}[label=(\arabic*)]
    \item \textbf{Importance Sampling Error:} Before we analyze the first term, note that the conditional expectation of the optimal value function at step $d+1$ given $b_d,a$ is calculated by the following, where we introduce $\mathbf{V}_{\mathbf{P},d+1}^*(b_d,a,s_{d,i}) \equiv \mathbf{V}_{\mathbf{P},d+1}^*(b_d,a)^{[i]}$ as a shorthand for the next-step integration over $(s_{d+1},o)$ conditioned on $(b_d,a,s_{d,i})$.
    Once again, we denote $[i]$ to indicate that $s_{d,i}$ was the particle chosen to generate the observation $o$, and if we are conditioning on a generic particle $s_d$, then we simply denote all the variables $\mathbf{V}_{\mathbf{P},d+1}^*(b_d,a,s_d)$:
    \begin{align}
        \mathbf{V}_{\mathbf{P},d+1}^*(b_d,a)^{[i]} &\equiv \int_S \int_O\ V_{\mathbf{P},d+1}^*(b_dao)\mathcal{Z}(o\mid a,s_{d+1})\mathcal{T}(s_{d+1}\mid s_{d,i},a)ds_{d+1}do\\
        \E_{\mathbf{P}}[V_{\mathbf{P},d+1}^*(b_dao)\mid b_d] &= \int_S \int_S \int_O V_{\mathbf{P},d+1}^*(b_dao)(\mathcal{Z}_{d+1}) (\mathcal{T}_{d,d+1})b_d \cdot ds_{d:d+1}do \\
        &= \int_S \mathbf{V}_{\mathbf{P},d+1}^*(b_d,a,s_d) b_d \cdot ds_{d}\\
        &= \frac{\int_{S^{d+1}} \mathbf{V}_{\mathbf{P},d+1}^*(b_d,a,s_d)(\mathcal{Z}_{1:d}) (\mathcal{T}_{1:d}) b_{0} ds_{0:d}}{\int_{S^{d+1}} (\mathcal{Z}_{1:d}) (\mathcal{T}_{1:d}) b_{0} ds_{0:d}}.
    \end{align}
    Noting that the term (1) is then the difference between the SN estimator and the conditional expectation, and that $\norm{\mathbf{V}_{\mathbf{P},d+1}^*}_{\inft} \leq V_{\max}$, we can apply the SN inequality for the second time in Lemma 2 to bound it by the augmented $\lambda/4$. 
    Thus, with our definition of $t_{\max}$, the bound holds with probability at least $1-3\exp(-C\cdot t_{\max}^2(\lambda,C))$.

    \item \textbf{Monte Carlo Weighted Sum Approximation Error:} The second term is the error resulting from estimating the sum with a Monte Carlo sum, which can be bounded by a Hoeffding-type bound.
    First, we assume that all the variables except $I$ are given, which are $\{s_{d,i}, w_{d,i}\}, b_d, a$.
    Then, we note that $\mathbf{V}_{\mathbf{P},d+1}^*(b_d,a,\cdot)$ is a function bounded by $V_{\max}$.
    For convenience of notation and conceptual clarity, we will denote $\mathbf{V}_{\mathbf{P},d+1}^*(b_d,a)^{[i]} \equiv \mathbf{V}(i)$, which means the value estimate realization for the $i$-th state index.
    Noting that the probability mass is $p_{w,d}(I=i) = (w_{d,i}/\sum_j w_{d,j})$, the Monte Carlo summation error can be simplified as the following:
    \begin{align}
        \left| \frac{\sum_{i = 1}^C w_{d,i}\mathbf{V}_{\mathbf{P},d+1}^*(b_d,a)^{[i]}}{\sum_{i = 1}^C w_{d,i}} - \frac{1}{C}\sum_{i=1}^{C}\mathbf{V}_{\mathbf{P},d+1}^*(b_d,a)^{[I_i]} \right| & \Longrightarrow \left| \sum_{i=1}^C p_{w,d}(I=i) \cdot \mathbf{V}(i) - \frac{1}{C}\sum_{i = 1}^C \mathbf{V}(I_i) \right|.
    \end{align}
    The first term in the difference is the expectation of $\mathbf{V}(\cdot)$ under the probability measure $p_{w,d}$:
    \begin{align}
        \left| \sum_{i=1}^C p_{w,d}(I=i) \cdot \mathbf{V}(i) - \frac{1}{C}\sum_{i = 1}^C \mathbf{V}(I_i) \right| &= \left| \E_{p_{w,d}}[\mathbf{V}(I)] - \frac{1}{C}\sum_{i = 1}^C \mathbf{V}(I_i) \right|.
    \end{align}
    This is precisely the form of the double-sided Hoeffding-type bound on the function values $\mathbf{V}(I)$, where a Monte Carlo summation, or the Monte Carlo average in this case, attempts to approximate the expected value.
    Therefore, we can choose $\lambda$ such that the absolute difference is bounded by $\lambda$ with probability at least $1-2\exp(-C\lambda^2/2V_{\max}^2)$ for an arbitrary fixed set of $\{s_{d,i}, w_{d,i}\}, b_d, a$:
    \begin{align}
        \left| \frac{\sum_{i = 1}^C w_{d,i}\mathbf{V}_{\mathbf{P},d+1}^*(b_d,a)^{[i]}}{\sum_{i = 1}^C w_{d,i}} - \frac{1}{C}\sum_{i=1}^{C}\mathbf{V}_{\mathbf{P},d+1}^*(b_d,a)^{[I_i]} \right| &\leq \lambda.
    \end{align}
    
    The previous calculation was done by conditioning on $\{s_{d,i}, w_{d,i}\}, b_d, a$.
    However, this bound does not depend on the specific values of these weights nor the particle belief sets, since Hoeffding bound only takes advantage of the fact that the random variables $I_i$ are sampled i.i.d. and the corresponding $\mathbf{V}(I_i)$ are bounded.
    Thus, we can revert this back into a general statement by applying the Tower property, and noting that the expectation of an indicator random variable is the probability of the associated event.
    By denoting the difference as $\Delta(\{s_{d,i}, w_{d,i}\}, b_d, a,\{I_i\})$, we obtain the unconditional Hoeffding-type bound:
    \begin{align}
        \P\lrc{\Delta(\{s_{d,i}, w_{d,i}\}, b_d, a,\{I_i\}) \leq \lambda} &= \E[\mb{1}_{\{\Delta(\{s_{d,i}, w_{d,i}\}, b_d, a,\{I_i\}) \leq \lambda\}}]\\
        &= \E\lrs{\E\lrs{\mb{1}_{\{\Delta(\{s_{d,i}, w_{d,i}\}, b_d, a,\{I_i\}) \leq \lambda\}}\mid \{s_{d,i}, w_{d,i}\}, b_d, a}}\\
        &= \E\lrs{\P\lrc{\Delta(\{s_{d,i}, w_{d,i}\}, b_d, a,\{I_i\}) \leq \lambda \mid \{s_{d,i}, w_{d,i}\}, b_d, a}}\\
        &\geq \E[1-2\exp(-C\lambda^2/2V_{\max}^2)]\\
        &=1-2\exp(-C\lambda^2/2V_{\max}^2).
    \end{align}
    Here, we use the factor augmentation once again to choose $\lambda/4$ such that the absolute difference is bounded by $\lambda/4$ with probability at least $1-2\exp(-C\lambda^2/32V_{\max}^2)$, which gets us our desired result:
    \begin{align}
        \left| \frac{\sum_{i = 1}^C w_{d,i}\mathbf{V}_{\mathbf{P},d+1}^*(b_d,a)^{[i]}}{\sum_{i = 1}^C w_{d,i}} - \frac{1}{C}\sum_{i=1}^{C}\mathbf{V}_{\mathbf{P},d+1}^*(b_d,a)^{[I_i]} \right| &\leq \frac{\lambda}{4}.
    \end{align}
    
    \item \textbf{Monte Carlo Next-Step Integral Approximation Error:} The third term can be thought of as Monte Carlo next-step integral approximation error. 
    To estimate $\mathbf{V}_{\mathbf{P},d+1}^*(b_d,a)^{[I_i]}$, we can simply use the quantity $V^*_{\mathbf{P},d+1}(b_dao^{[I_i]})$, as the random vector $(s_{d+1,{I_i}},o_{I_i})$ is jointly generated using $G$ according to the correct probability $\mathcal{Z}(o\mid a,s_{d+1})\mathcal{T}(s_{d+1}\mid s_{d,I_i},a)$ given $s_{d,I_i}$ in the simulation realized in the tree. 
    Consequently, the quantity $V^*_{\mathbf{P},d+1}(b_dao^{[I_i]})$ for a given $(s_{d,{I_i}},b_d,a)$ is an unbiased 1-sample MC estimate of $\mathbf{V}_{\mathbf{P},d+1}^*(b_d,a)^{[I_i]}$. 
    We define the difference between these two quantities as $\Delta_{d+1}$, which is implicitly a function of random variables $(s_{d+1,{I_i}},o_{I_i})$:
    \begin{align}
      \Delta_{d+1}(b_d,a)^{[I_i]} &\equiv \mathbf{V}_{\mathbf{P},d+1}^*(b_d,a)^{[I_i]} - V^*_{\mathbf{P},d+1}(b_dao^{[I_i]}).
    \end{align}
    Then, we note that $\norm{\Delta_{d+1}}_{\inft} \leq 2V_{\max}$ and $\E\Delta_{d+1} = 0$ by the Tower property conditioning on $(s_{d,{I_i}},b_d,a)$ (which is implicitly conditioning on $I_i$, but this does not matter greatly as everything cancels out) and integrating over $(s_{d+1,{I_i}},o_{I_i})$ first, which holds for any choice of well-behaved sampling distributions on $\{s_{0:d}\}_i$.
    Using this fact, we can then consider this term as a Monte Carlo estimator for the bias $\E\Delta_{d+1} = 0$, and use another Hoeffding bound. 
    Since $\norm{\Delta_{d+1}}_{\inft} \leq 2V_{\max}$, our $\lambda$ factor is then augmented by 1/2 to once again obtain probability at least $1-2\exp(-C\lambda^2/32V_{\max}^2)$:
    \begin{align}
      &\left| \frac{1}{C}\sum_{i=1}^{C}\mathbf{V}_{\mathbf{P},d+1}^*(b_d,a)^{[I_i]} - \frac{1}{C}\sum_{i=1}^{C}V^*_{\mathbf{P},d+1}(b_dao^{[I_i]}) \right| \\
      &= \left| \frac{1}{C}\sum_{i=1}^{C}(\mathbf{V}_{\mathbf{P},d+1}^*(b_d,a)^{[I_i]} - V^*_{\mathbf{P},d+1}(b_dao^{[I_i]})) - 0 \right| \notag\\
      &= \left| \frac{1}{C}\sum_{i=1}^{C}\Delta_{d+1}(b_d,a)^{[I_i]} - \E\Delta_{d+1} \right| \leq \frac{\lambda}{2}.
    \end{align}
    
    \item \textbf{Inductive Function Estimation Error:} The fourth term is bounded by the inductive hypothesis, since each $i$-th absolute difference of the $Q$-function and its estimate at step $d+1$, and furthermore the value function and its estimate at step $d+1$, are all bounded by $\alpha_{d+1}$.
    
\end{enumerate}

Thus, each of the error terms are bound by $(A) \leq \frac{R_{\max}}{4V_{\max}}\lambda$ and $(B)\leq \frac{1}{4}\lambda + \frac{1}{4}\lambda + \frac{1}{2}\lambda + \alpha_{d+1}$, which uses the SN concentration bound 2 times and Hoeffding bound 2 times. 
Combining (A) and (B), we can obtain the desired bound:
\begin{align}
   |Q_{\mathbf{P},d}^*(b_d,a) - \hat{Q}^*_d(\bar{b}_d,a) | &\leq \frac{R_{\max}}{4V_{\max}}\lambda + \gamma \lrs{\frac{1}{4}\lambda + \frac{1}{4}\lambda + \frac{1}{2}\lambda + \alpha_{d+1}} \\
   &\leq \frac{1-\gamma}{4}\lambda + \gamma \lrs{\frac{1}{4}\lambda + \frac{3}{4\gamma}\lambda + \alpha_{d+1}} \\
   &=  \lambda + \gamma\alpha_{d+1} = \alpha_d.
\end{align}

Now, we derive the worst case union bound probability.
First, we want to ensure that the SN concentration inequality holds with probability $1-3\exp(-C\cdot t_{\max}^2(\lambda,C))$ whenever it is used at any given step $d$ and action $a$.
Similarly, we also want to ensure that the Hoeffding-type inequality holds with probability at least $1-2\exp(-C\lambda^2/32V_{\max}^2)$ whenever it is used at any given step $d$ and action $a$.
This means we can bound the worst case probability of using either bound by 
\begin{align}
    \max(3\exp(-C\cdot t_{\max}^2(\lambda,C))&, 2\exp(-C\lambda^2/32V_{\max}^2)) \\
    &\leq 3\exp(-C\cdot t_{\max}^2(\lambda,C)) + 2\exp(-C\lambda^2/32V_{\max}^2) \\
    &\leq 5\exp(-C\cdot\tilde{t}^2).
\end{align}
Furthermore, we multiply the worst-case union bound factor $(4|A|C)^{D}$, since we want the function estimates to be within their respective concentration bounds for all the actions $|A|$ and child nodes $C$ at each step $d=0,\ldots,D-1$, for the 2 times we use SN concentration bound and 2 times we use the double-sided Hoeffding-type bound in the induction step. 
We once again multiply the final probability by $|A|$ to account for the root node $Q$-value estimates also satisfying their respective concentration bounds for all actions. 
Thus, the worst case union bound probability of all bad events is bounded by probability $5|A|(4|A|C)^{D}\exp(-C\cdot\tilde{t}^2)$.
Therefore, we have shown that the concentration bounds for both the particle likelihood SN estimator and Monte Carlo estimator components converge with probability at least $1-5|A|(4|A|C)^{D}\exp(-C\cdot\tilde{t}^2)$ for all levels $d$:
\begin{align}
    |Q_{\mathbf{P},d}^*(b_d,a) - \tilde{Q}^*_d(\bar{b}_d,a)| \leq \tilde{\alpha}_{d}.
\end{align}

\textbf{PB-MDP Value Convergence:} Once again, we split the difference between the SN estimator and the $Q_{\mathbf{M}_{\mathbf{P}}}^*$ function into two terms, the reward estimation error (A) and the next-step value estimation error (B):
{\small
\begin{align}
  |Q^*_{\mathbf{M}_{\mathbf{P}},d}(\bar{b}_d,a)- \hat{Q}_{\omega,d}^*(\bar{b}_d,a) | &\leq \undb{\left| \rho(\bar{b}_d,a) - \rho(\bar{b}_d,a)\right|}{(A) = 0} + \gamma \undb{\left| \E_{\mathbf{M}_{\mathbf{P}}}[V_{\mathbf{M}_{\mathbf{P}},d+1}^*(\bar{b}_{d+1})\mid \bar{b}_{d}, a] - \frac{1}{C}\sum_{i=1}^{C} \hat{V}^*_{\omega,d+1}(\bar{b}_{d+1}'^{[I_i]}) \right|}{(B)}.
\end{align}
}

Since our particle belief MDP induces no reward estimation error, the term (A) is always 0 and proving the base case $d=D-1$ is trivial as (A) and (B) are both 0.

We now prove that the difference (B) is bounded for all $d=0,\ldots,D-1$.
We use the triangle inequality repeatedly to separate it into two terms; (1) the MC transition approximation error bounded by $\lambda$, and (2) the inductive function estimation error bounded by $\beta_{d+1}$:
\begin{align}
    &\undb{\left| \E_{\mathbf{M}_{\mathbf{P}}}[V_{\mathbf{M}_{\mathbf{P}},d+1}^*(\bar{b}_{d+1})\mid \bar{b}_{d}, a] - \frac{1}{C}\sum_{i=1}^{C} \hat{V}^*_{\omega,d+1}(\bar{b}_{d+1}'^{[I_i]}) \right|}{(B)}\\
    &\leq \undb{\left| \E_{\mathbf{M}_{\mathbf{P}}}[V_{\mathbf{M}_{\mathbf{P}},d+1}^*(\bar{b}_{d+1})\mid \bar{b}_{d}, a] - \frac{1}{C}\sum_{i=1}^{C} V_{\mathbf{M}_{\mathbf{P}},d+1}^*(\bar{b}_{d+1}'^{[I_i]}) \right|}{(1) MC transition approximation error} + \undb{\left| \frac{1}{C}\sum_{i=1}^{C} V_{\mathbf{M}_{\mathbf{P}},d+1}^*(\bar{b}_{d+1}'^{[I_i]}) - \frac{1}{C}\sum_{i=1}^{C} \hat{V}^*_{\omega,d+1}(\bar{b}_{d+1}'^{[I_i]}) \right|}{(2) Inductive function estimation error}\notag\\
    &\leq \undb{\lambda}{(1)} + \undb{\beta_{d+1}}{(2)}.
\end{align}
We justify how each error term is bounded.
\begin{enumerate}[label=(\arabic*)]
    \item \textbf{MC Transition Approximation Error:} The Monte Carlo summation over the next step particle belief state samples $\{\bar{b}_{d+1}'^{[I_i]}\}$ given $(\bar{b}_d, a)$ is essentially approximating the integration over the transition density $\tau(\bar{b}_{d+1}\mid \bar{b}_{d},a)$. 
    Since the value function and its estimate are both bounded by $V_{\max}$, we can invoke Hoeffding bound here to obtain the following exponential probabilistic bound on the difference:
    \begin{align}
      \P\lrc{\left| \E_{\mathbf{M}_{\mathbf{P}}}[V_{\mathbf{M}_{\mathbf{P}},d+1}^*(\bar{b}_{d+1})\mid \bar{b}_{d}, a] - \frac{1}{C}\sum_{i=1}^{C} V_{\mathbf{M}_{\mathbf{P}},d+1}^*(\bar{b}_{d+1}'^{[I_i]}) \right| \leq \lambda} &\geq 1-2\exp(-C\lambda^2/2V_{\max}^2).
    \end{align}
    \item \textbf{Inductive Function Estimation Error:} The second term is bounded by the inductive hypothesis, since each $i$-th absolute difference of the $Q$-function and its estimate at step $d+1$, and furthermore the value function and its estimate at step $d+1$, are all bounded by $\beta_{d+1}$.
\end{enumerate}

By applying similar logic of ensuring that every particle belief state node and action pairs can satisfy the concentration inequality, we note that the particle belief MDP approximation concentration bound is satisfied with probability at least $1 - |A|(|A|C)^{D}\exp(-C\lambda^2/2V_{\max}^2)$.
Thus, since (A) is 0 and (B) is bounded by $\lambda + \beta_{d+1}$, the $Q$-value estimation error with respect to $\mathbf{M}_{\mathbf{P}}$ is bounded as desired:
\begin{align}
  |Q^*_{\mathbf{M}_{\mathbf{P}},d}(\bar{b}_d,a)- \hat{Q}_{\omega,d}^*(\bar{b}_d,a) | &\leq \gamma(\lambda + \beta_{d+1}) = \beta_d.
\end{align}

\textbf{Combining both concentration bounds:} In order to enable simultaneous satisfaction of the two concentration inequalities, we bound the worst case union probability by using the definition of $\tilde{t}$ and combining the upper bounding terms together:
\begin{align}
  5|A|(4|A|C)^{D}\exp(-C\cdot \tilde{t}^2) +& |A|(|A|C)^{D}\exp(-C\lambda^2/2V_{\max}^2) \\
  &\leq 5|A|(4|A|C)^{D}\exp(-C\cdot \tilde{t}^2) + |A|(|A|C)^{D}\exp(-C\cdot \tilde{t}^2)\\
  &\leq 5|A|(4|A|C)^{D}\exp(-C\cdot \tilde{t}^2) + |A|(4|A|C)^{D}\exp(-C\cdot \tilde{t}^2)\\
  &= 6|A|(4|A|C)^{D}\exp(-C\cdot \tilde{t}^2).
\end{align}
Therefore, we conclude that the $Q$-value concentration inequalities for both POMDP approximation error and particle belief approximation error are bounded by $\alpha_d,\beta_d$ at every node, respectively, with probability at least $1 - 6|A|(4|A|C)^{D}\exp(-C\cdot \tilde{t}^2)$.

\section{Proof of Theorem 2 - Sparse Sampling-$\omega$ Coupled Optimality}\label{app:theorem2}
We reiterate the conditions and \cref{thm:ssw} below:
\begin{enumerate}[label=(\roman*)]
    \item $S$ and $O$ are continuous spaces, and the action space has a finite number of elements, $|A|<+\inft$.
    \item The densities $\mathcal{Z},\mathcal{T},b_0$ have the property that, for any observation sequence $\{o_{n}\}_{n=1}^d$, the R\'enyi divergence of the target distribution $\mathcal{P}^d$ and sampling distribution $\mathcal{Q}^d$ (\cref{eq:pqdef,eq:pqdef2}) is bounded above by $d_{\inft}^{\max}$ for all $d = 0,\ldots,D-1$:
    \begin{align}
        d_{\inft}(\mathcal{P}^d||\mathcal{Q}^d) = \text{ess sup}_{x \sim \mathcal{Q}^d}\, w_{\mathcal{P}^d/\mathcal{Q}^d}(x) \leq d_{\inft}^{\max}
    \end{align}
    \item The reward function $R$ is bounded by a finite constant $R_{\max}$, and hence the value function is bounded by $V_{\max} \equiv \frac{R_{\max}}{1-\gamma}$.
    \item We can sample from the generating function $G$ and evaluate the observation density $\mathcal{Z}$. 
    \item The POMDP terminates after no more than $D < \infty$ steps. 
    \item We restrict our analysis to all the beliefs $b \in B$ that are realizable from the initial belief $b_0$ through Bayesian updates with action sequences $\{a_n\}$ and observation sequences $\{o_n\}$.
\end{enumerate}
\begin{theorem}[Sparse Sampling-$\omega$ Coupled Optimality]
  Suppose conditions \ref{req:space}-\ref{req:realizable} are satisfied. Then, for any $\lambda >0$ and $0<\delta\leq 1$, choosing particle count constant $C$ that satisfies:
  \begin{align}
    C &= \max\lrc{\lrp{\frac{4V_{\max} d_{\inft}^{\max}}{\lambda}}^2, \frac{64V_{\max}^2 }{\lambda^2}\lrp{D\log\frac{24|A|^{\frac{D+1}{D}}V_{\max}^2D }{\lambda^2} + \log\frac{1}{\delta}}},
  \end{align}
  the $Q$-function estimates $\hat{Q}^*_{\omega,d}(\bar{b}_d,a)$ obtained for all depths $d=0,\ldots,D-1$, realized beliefs or histories $b_d$ encountered in the Sparse Sampling-$\omega$ tree, and actions $a$ are jointly near-optimal with respect to $Q_{\mathbf{P},d}^*$ and $Q_{\mathbf{M}_{\mathbf{P}},d}^*$ with probability at least $1 - \delta$: 
  \begin{align}
    |Q_{\mathbf{P},d}^*(b_d,a) - \hat{Q}^*_{\omega,d}(\bar{b}_d,a) | & \leq \frac{\lambda}{1-\gamma}, \\ 
    |Q_{\mathbf{M}_{\mathbf{P}},d}^*(\bar{b}_d,a) - \hat{Q}^*_{\omega,d}(\bar{b}_d,a) | & \leq \frac{\lambda}{1-\gamma}.
  \end{align}
\end{theorem}

\begin{proof}
This proof has two parts.
First, we show that the choice of $C$ is valid given the assumptions in \cref{lemma:induction}.
Then, we use \cref{lemma:induction,lemma:value} to prove the $Q$-value estimate claim.

The conditions necessary for $C$ from \cref{lemma:induction} are the following:
\begin{align}
    t_{\max}&(\lambda,C) = \frac{\lambda}{4V_{\max} d_{\inft}^{\max}} - \frac{1}{\sqrt{C}} > 0\\
    \delta &\geq 6|A|(4|A|C)^{D}\exp(-C\cdot\tilde{t}^2) \label{eq:delta-bound}\\
    \tilde{t}_{\max}&(\lambda,C) \equiv \max\lrc{t_{\max}(\lambda,C), \lambda/4\sqrt{2}V_{\max}}
\end{align}
Note that the constraint on $t_{\max}$ implies that the following must be true:
\begin{align}
    \frac{\lambda}{4V_{\max} d_{\inft}^{\max}} - \frac{1}{\sqrt{C}} > 0 \Longrightarrow C > \lrp{\frac{4V_{\max} d_{\inft}^{\max}}{\lambda}}^2,
\end{align}
which gives us the first option of $C$ in the maximum.

For the next option of $C$, we show that substituting the formula yields condition \cref{eq:delta-bound}.
We note that due to the definition of $\tilde{t}_{\max}$, the following is true:
\begin{align}
    \tilde{t}_{\max}&(\lambda,C) \geq \lambda/4\sqrt{2}V_{\max}.
\end{align}
Let us denote $T \equiv (\lambda/4\sqrt{2}V_{\max})^2$ for convenience.
Then, since $T$ is upper-bounded by $\tilde{t}_{\max}$, 
\begin{align}
    6|A|(4|A|C)^{D}\exp(-C\cdot\tilde{t}^2) &\leq 6|A|(4|A|C)^{D}\exp(-C\cdot T)\\
    &\leq (24|A|^{\frac{D+1}{D}}C)^{D}\exp(-C\cdot T) \label{eq:boundbydelta}
\end{align}
Consequently, if we show that \cref{eq:boundbydelta} is bounded by $\delta$, then we automatically show that the original quantity is bounded by $\delta$ as well.
By defining $X \equiv 24|A|^{\frac{D+1}{D}}$, we want to show that this simplified formula is bounded above by $\delta$:
\begin{align}
    \delta &\geq (X\cdot C)^{D}\exp(-C\cdot T).
\end{align}
We will show that our second option of $C$ satisfies the following, where the simplified formula equals:
\begin{align}
    \frac{64V_{\max}^2 }{\lambda^2}\lrp{D\log\frac{24|A|^{\frac{D+1}{D}}V_{\max}^2D }{\lambda^2} + \log\frac{1}{\delta} } &\Longrightarrow \frac{2}{T}\lrp{D\log\frac{XD}{T} + \log\frac{1}{\delta}}
\end{align}
Substituting in the second option of $C$:
\begin{align}
    (X\cdot C)^{D}\exp(-C\cdot T) &= \delta^2 \frac{\lrp{\frac{2XD}{T}\log\frac{XD}{T} + \frac{2X}{T}\log\frac{1}{\delta}}^{D}} {\lrp{\frac{XD}{T}}^{2D}}\\
    &= \delta^2 \frac{\lrp{\frac{XD}{T}\log\lrp{\frac{XD}{T}}^2 + \frac{XD}{T}\log\lrp{\frac{1}{\delta}}^{2/D}}^{D}} {\lrp{\frac{XD}{T}}^{2D}}\\
    &= \delta^2 \frac{\lrp{\log\lrp{\frac{XD}{T\delta^{1/D}}}^2}^{D}} {\lrp{\frac{XD}{T}}^{D}}= \delta \lrp{\frac{\log\lrp{\frac{XD}{T\delta^{1/D}}}^2} {\lrp{\frac{XD}{T\delta^{1/D}}}}}^{D}
\end{align}
Note that the function $f(x) = \log x^2/ x$ is less than 1 for $x > 0$ (in fact, the maximum value of $f(x)$ is exactly $2/e$, attained by setting $x=e$).
This means that the quantity inside the parentheses is less than 1, which lets us obtain our desired result
\begin{align}
    (X\cdot C)^{D}\exp(-C\cdot T) &\leq \delta.
\end{align}
Therefore, each of our option of $C$ satisfies the respective conditions, and taking the maximum of the options will yield valid results for both inequality constraints:
\begin{align}
    C &= \max\lrc{\lrp{\frac{4V_{\max} d_{\inft}^{\max}}{\lambda}}^2, \frac{64V_{\max}^2 }{\lambda^2}\lrp{D\log\frac{24|A|^{\frac{D+1}{D}}V_{\max}^2D }{\lambda^2} + \log\frac{1}{\delta}}}.
\end{align}
This concludes the first part of the proof; we have shown that $C$ is a valid choice.
Next, we prove the value bounds.

With our choice of $C$, from \cref{lemma:induction}, the error in estimating $Q^*$ with our Sparse Sampling-$\omega$ policy is bounded by $\alpha_d$ for all $d,a$ with probability at least $1-\delta$. 
Since $\alpha_d \leq \alpha_0$, the following holds for all $d=0,\ldots, D-1$ with probability at least $1-\delta$ through \cref{lemma:pwl_sn,lemma:induction}:
\begin{align}
  |Q_{\mathbf{P},d}^*(b_d,a) - \hat{Q}^*_d(\bar{b}_d,a) | & \leq \alpha_0  \leq \sum_{d=0}^{D-1} \gamma^{d} \lambda \leq \frac{\lambda}{1-\gamma},\\
  |Q_{\mathbf{M}_{\mathbf{P}},d}^*(\bar{b}_d,a) - \hat{Q}^*_d(\bar{b}_d,a) | & \leq \beta_0  \leq \sum_{d=1}^{D} \gamma^{d} \lambda \leq \frac{\lambda}{1-\gamma}.
\end{align}
\end{proof}

\section{Proof of Theorem 4 - Particle Belief MDP Approximate Policy Convergence}\label{app:theorem4}
Before we prove Theorem \ref{thm:pomdp-grand}, we first prove the following lemma, which is an adaptation of Kearns \textit{et al.}~(\citeyear{kearns2002sparse}) and Singh and Yee~(\citeyear{Singh1994}) for belief states $b$.

\renewcommand{\thelemma}{\arabic{lemma}A}%
\begin{lemma}\label{lemma:value}
    Consider a POMDP with a finite horizon of $D$ steps and policy $\pi(b) = \arg\max_a\, \hat{Q}(b, a)$ where $\hat{Q}$ is a stochastic value function approximator with errors bounded by a positive constant $\xi$: $|Q^*(b, a) - \hat{Q}(b, a)| \leq \xi$.
    Let $V^\pi(b_0)$ denote the value of executing $\pi$ starting at belief $b_0$ with an exact Bayesian belief update, $b_{t+1} = b_tao$, between each call to the policy.
    Then
    \begin{equation}
        V^*(b_0) - V^\pi(b_0) \leq \frac{2\xi}{1-\gamma}\text{.}
    \end{equation}
\end{lemma}

\begin{proof}
First, note that if an action is chosen by $\pi$, it must appear better than $\pi^*$ according to $\hat{Q}$, i.e. $\hat{Q}(b, \pi(b)) \geq \hat{Q}(b, \pi^*(b))$. The worst case is when $\hat{Q}(b, \pi(b)) = Q^*(b, \pi(b)) + \xi$ and $\hat{Q}(b, \pi^*(b)) = Q^*(b, \pi^*(b)) - \xi$. Thus, for any $t$, we have the bound
\begin{align}
    Q^*\left(b_t, \pi^*(b_t)\right) - \E[Q^*\left(b_t, \pi(b_t)\right)] &\leq 2 \xi \text{.} \label{eq:2beta}
\end{align}
Next, we prove that $V^*(b_t) - V^\pi(b_t) \leq \sum_{d=t}^{D-1} \gamma^{d-t} 2 \xi$ using induction from $t=D-1$ to $t=0$. We verify the base case, $t=D-1$, by observing that both $Q^\pi(b_{D-1}, \pi(b_{D-1}))$ and $Q^*(b_{D-1}, \pi(b_{D-1}))$ are equal to $R(b_{D-1}, \pi(b_{D-1}))$ since no further reward can be accumulated and using \cref{eq:2beta}:
\begin{align}
    V^*(b_{D-1}) - V^\pi(b_{D-1}) &= Q^*(b_{D-1}, \pi^*(b_{D-1})) - \E[Q^\pi(b_{D-1}, \pi(b_{D-1}))]\\
    &= Q^*(b_{D-1}, \pi^*(b_{D-1})) - \E[Q^*(b_{D-1}, \pi(b_{D-1}))]\\
    & \leq 2\xi \text{.}
\end{align}
The inductive step is verified by subtracting and adding $\E[Q^*\left(b, \pi(b)\right)]$, using the bound in \cref{eq:2beta}, and applying the inductive hypothesis:
\begin{align}
    V^*(b_t) - V^\pi(b_t) &= Q^*\left(b, \pi^*(b)\right) - \E[Q^\pi(b, \pi(b))]\\
    &= \undb{Q^*\left(b, \pi^*(b)\right) - \E[Q^*\left(b, \pi(b)\right)]}{Bounded by \cref{eq:2beta}} + \E[Q^*\left(b, \pi(b)\right)] - \E[Q^\pi(b, \pi(b))]\\
    &\leq 2\xi + \E[Q^*\left(b, \pi(b)\right)] - \E[Q^\pi(b, \pi(b))]\\
    &= 2\xi + \E[R\left(b, \pi(b)\right)] + \gamma \E[V^*(b_{t+1}))] - \E[R(b, \pi(b))] - \gamma \E[V^\pi(b_{t+1}))]\\
    &= 2\xi + \gamma \E[V^*(b_{t+1}))- V^\pi(b_{t+1}))]\\
    &= 2\xi + \gamma \sum_{d=t+1}^{D-1} \gamma^{d-t} 2 \xi = \sum_{d=t}^{D-1} \gamma^{d-t} 2 \xi \text{.}
\end{align}
Now, by applying the result above to $t=0$, we prove the lemma:
\begin{equation}
    V^*(b_0) - V^\pi(b_0) \leq \sum_{d=0}^{D-1} \gamma^d 2\xi \leq \frac{2\xi}{1-\gamma} \text{.}\notag
\end{equation}%
\end{proof}

\setcounter{theorem}{3}
\begin{theorem}[Particle Belief MDP Approximate Policy Convergence]
    Suppose a near-optimal MDP planning algorithm $\mathcal{A}$ is used to plan with particle belief MDP $\mathbf{M}_{\mathbf{P}}$ repeatedly in a closed loop with POMDP environment $\mathbf{P}$ and an exact Bayesian belief updater to process observations from the environment. Further assume that regularity conditions \ref{req:space}-\ref{req:realizable} are met for $\mathbf{M}_{\mathbf{P}}$ and that $\mathcal{A}$ can approximate the $Q$-values of $\mathbf{M}_{\mathbf{P}}$ with arbitrary precision $\epsilon_\mathcal{A}$ with probability at least $1-\delta_\mathcal{A}$.
    Then, for any $\epsilon > 0$, we can choose $C$ such that the value obtained by planning with $\mathcal{A}$ in $\mathbf{M}_{\mathbf{P}}$ is within $\epsilon$ of the optimal POMDP value function at $b_0$:
    \begin{align}
        V_{\mathbf{P}}^*(b_0)-V_{\mathbf{M}_{\mathbf{P}}}^{\mathcal{A}}(b_0) &\leq \epsilon.
    \end{align}
\end{theorem}

\begin{proof}
First, we choose $\lambda$ and $\delta_{\mathbf{M}_{\mathbf{P}}}$ for the particle belief MDP approximation to be the following, with $\epsilon_{\mathbf{M}_{\mathbf{P}}} = \frac{2\lambda}{1-\gamma}$ as per the definition in the proof of \cref{thm:pomdp-qvalue}:
\begin{align}
    \lambda &= \frac{(1-\gamma)^2}{8}\epsilon - \frac{1-\gamma}{2}\epsilon_{\mathcal{A}} \Longrightarrow \epsilon_{\mathbf{M}_{\mathbf{P}}} = \frac{1-\gamma}{4}\epsilon - \epsilon_{\mathcal{A}},\\
    \delta_{\mathbf{M}_{\mathbf{P}}} &= \frac{\epsilon_{\mathbf{M}_{\mathbf{P}}} + \epsilon_{\mathcal{A}}}{V_{\max}D(1-\gamma)} - \delta_{\mathcal{A}}.
\end{align}
In this context, we mathematically mean a near-optimal MDP planning algorithm $\mathcal{A}$ to be one that can obtain arbitrarily small values of $\epsilon_{\mathcal{A}},\delta_{\mathcal{A}}$ that would satisfy $\lambda > 0$ and $0<\delta_{\mathbf{M}_{\mathbf{P}}}\leq1$.
Then, we can choose $C$ through \cref{thm:pomdp-qvalue} such that we can invoke \cref{cor:planning} to obtain $Q$-value estimation accuracy $\xi = \epsilon_{\mathbf{M}_{\mathbf{P}}} + \epsilon_{\mathcal{A}}$ with worst case probability $\delta' = \delta_{\mathbf{M}_{\mathbf{P}}} + \delta_{\mathcal{A}}$.
Consequently, with our choice of $\lambda$ and $\delta_{\mathbf{M}_{\mathbf{P}}}$ above, $\xi$ and $\delta'$ are equal to the following:
\begin{align}
    \xi &= \frac{1-\gamma}{4}\epsilon,\\ 
    \delta' &= \frac{\xi}{V_{\max}D(1-\gamma)}.
\end{align}

During policy execution, we create a new independent tree and choose an action based on the estimated $Q$-values.
This means with algorithm $\mathcal{A}$ equipped with particle belief states, there is at most $\delta'$ probability that $|Q_{\mathbf{P},0}^*(b_0,a) - \hat{Q}_{\mathbf{M}_{\mathbf{P}},0}^{\mathcal{A}}(\bar{b}_0,a)| > \xi$ at each of the $D$ steps of the POMDP.

Thus, with probability at least $1-D\delta'$, we execute a policy that meets the assumptions of \cref{lemma:value} with $\xi$, and hence by \cref{lemma:value}, the difference between the optimal value and the average accumulated reward for this case is at most $\frac{2\xi}{1-\gamma}$.
In the other case, which occurs with at most probability $D\delta'$, an arbitrarily bad policy can be executed, resulting in an accumulated reward difference of up to $2V_{\max}$ from the optimal policy.
Combining these two cases, we have
\begin{align}
    V_{\mathbf{P}}^*(b_0)-V_{\mathbf{M}_{\mathbf{P}}}^{\mathcal{A}}(b_0) &\leq (1-D\delta') \frac{2\xi}{1-\gamma} + D\delta' (2 V_{\max})\\
    &\leq \frac{2\xi}{1-\gamma} + D\delta' (2 V_{\max})\\
    &= \frac{2\xi}{1-\gamma} + \frac{2\xi}{1-\gamma}\\
    & = \epsilon \text{.}
\end{align}
\end{proof}

\section{Experiment Details}\label{app:experiments}
\begin{table}[H]
    \centering
    \begin{tabular}{lcccccccccc}
        \toprule
        & $c_\textsc{UCB}$ & $\beta_\textsc{UCB}$ & $k_a$ & $\alpha_a$ & $k_o$ & $\alpha_o$ & $m_{\min}$ & $\delta$ & $C$ & Depth\\
        \midrule
        Laser Tag (D, D, D) & & & & & & & & & \\
        \midrule
        Sparse-PFT & 15 & 0.22 & - & - & 15 & -    & -  & -   & 96 & 37 \\
        PFT-DPW    & 25 & 0.09 & - & - & 5  & 0.33 & -  & -   & 25 & 48 \\
        POMCPOW    & 26 & -    & - & - & 4  & 0.03 & -  & -   & -  & 50 \\
        AdaOPS     & -  & -    & - & - & -  & -    & 30 & 0.1 & -  & 90 \\
        POMCP      & 26 & -    & - & - & -  & -    & -  & -   & -  &  50 \\
        QMDP       & -  & -    & - & - & -  & -    & -  & -   & -  & -  \\
        \midrule
        Light Dark (D, D, C) & & & & & & & & & \\
        \midrule
        Sparse-PFT& 95  & 0.39 & - & - & 24 & -    & -  & -   & 134& 28 \\
        PFT-DPW   & 93  & 0.30 &  - & - & 13 & 0.08 & -  & -   & 33 & 20 \\
        POMCPOW   & 90  & - &  - & - & 5  & 0.07 & -  & -   & -  & 20 \\
        AdaOPS    &  -  & - &  - & - & -  & -    & 30 & 0.1 & -  & 90  \\
        POMCP     &  83 & - &  - & - & -  & -    & -  & -   & -  & 20 \\
        QMDP      &  -  & - &  - & - & -  & -    & -  & -   & -  & -  \\
        \midrule
        Sub Hunt (D, D, C) & & & & & & & & & \\
        \midrule
        Sparse-PFT & 20 & 0.25 &  - & - & 27 & -   & -  & -   & 23 & 20 \\
        PFT-DPW   & 85 & 0.08 &  - & - & 10 & 0.08 & -  & -   & 79 & 20 \\
        POMCPOW   & 17  & - &  - & - & 6 & 0.01 & -  & -   & -  & 50 \\
        AdaOPS    &  -  & - &  - & - & - & -    & 30 & 0.1 & -  & 90  \\
        POMCP     & 17  & - &  - & - & - & -    & -  & -   & -  & 84 \\
        QMDP      & -   & - &  - & - & - & -    & -  & -   & -  & -  \\
        \midrule
        VDP Tag (C, C, C) & & & & & & & & & \\
        \midrule
        Sparse-PFT & 16  & 0.12 &  28 & -    & 28 & -     & - & - & 385 & 33 \\
        PFT-DPW    & 23  & 0.25 &  22 & 0.32 & 21 & 0.04  & - & - & 132 & 44 \\
        POMCPOW    & 110 & - &  30 & 0.03 & 5  & 0.01  & - & - & -   & 10 \\
        \midrule
        VDP Tag$^D$ (C, D, C) & & & & & & & & & \\
        \midrule
        Sparse-PFT & 76 & 0.08 &  -  & - & 25 & -    & -  & -    & 444 & 46 \\
        PFT-DPW   & 10  & 0.18 &  -  & - & 9  & 0.11 & -  & -    & 330 & 22 \\
        POMCPOW   & 31  & - &  -  & - & 5  & 0.05 & -  & -    & -   & 10 \\
        AdaOPS    &  -  & - &  -  & - & -  & -    & 40 & 0.25 & -   & 90  \\
        \bottomrule
    \end{tabular}
    \caption{Summary of hyperparameters used in experiments.}
    \label{tab:hyper}
\end{table}

For UCT methods, we vary $c_\textsc{UCB}$, the UCB exploration parameter, and $\beta_{UCB}$, the polynomial UCB factor.
$k_a$ and $\alpha_a$ are action progressive widening parameters, where new actions are added if widening criterion $|C(h)| \le k_aN(h)^{\alpha_a}$ is met.
Similarly, $k_o$ and $\alpha_o$ are observation progressive widening parameters, where new actions are added if widening criterion $|C(ha)| \le k_oN(ha)^{\alpha_o}$ is met.
Sparse-PFT uses $\alpha_a = \alpha_o = 0$.
$C$ is the number of particles constituting internal tree beliefs for PFT methods.
$m_{min}$ is the minimum number of particles required to approximate a belief for AdaOPS.
Finally, $\delta$ is the maximum distance distance between beliefs resulting from observation branches required to merge the branches for AdaOPS.


\begin{table}[H]
    \centering
    \begin{tabular}{lccc}
        \toprule
        & $\hat{V}$ & $L_0$ & $U_0$\\
        \midrule
        Laser Tag (D, D, D) & & &\\
        \midrule
        Sparse-PFT & QMDP PO-Rollout & - & -  \\
        PFT-DPW   & QMDP PO-Rollout & - & -  \\
        POMCPOW   & FO-Value  & - & -  \\
        AdaOPS    &  -  & Random Rollout & QMDP \\
        POMCP     & Random Rollout  & - & -  \\
        QMDP      & -   & - & - \\
        \midrule
        Light Dark (D, D, C) & & &\\
        \midrule
        Sparse-PFT & QMDP PO-Rollout & - & -  \\
        PFT-DPW   & QMDP PO-Rollout & - & -  \\
        POMCPOW   & FO-Value  & - & -  \\
        AdaOPS    &  -  & Random Rollout & QMDP \\
        POMCP     & Random Rollout  & - & -  \\
        QMDP      & -   & - & - \\
        \midrule
        Sub Hunt (D, D, C) & & &\\
        \midrule
        Sparse-PFT & QMDP PO-Rollout & - & -  \\
        PFT-DPW   & QMDP PO-Rollout & - & -  \\
        POMCPOW   & FO-Value  & - & -  \\
        AdaOPS    &  -  & Random Rollout & QMDP \\
        POMCP     & Random Rollout  & - & -  \\
        QMDP      & -   & - & - \\
        \midrule
        VDP Tag (C, C, C) & & &\\
        \midrule
        Sparse-PFT & Random Rollout & - & -  \\
        PFT-DPW   & Random Rollout & - & -  \\
        POMCPOW   & Random Rollout  & - & -  \\
        \midrule
        VDP Tag$^D$ (C, D, C) & & &\\
        \midrule
        Sparse-PFT & Random Rollout & - & -  \\
        PFT-DPW   & Random Rollout & - & -  \\
        POMCPOW   & Random Rollout  & - & -  \\
        AdaOPS    &  -  & Random Rollout & $10^6$  \\
        POMCP     & Random Rollout  & - & -  \\
        \bottomrule
    \end{tabular}
    \caption{Summary of leaf node value estimators used in experiments.}
\end{table}

QMDP PO-Rollout corresponds to sampling a ``true state'' from a leaf node particle belief and simulating the state/belief dynamics with a particle filter as a belief updater and QMDP as a policy. The returns following the trajectory of the sampled ``true state'' are taken as a value estimate for the leaf node.
FO-Value ("Fully Observable Value") corresponds to using the MDP value for the state representation of a particle.
QMDP corresponds to using the belief value estimate given by a QMDP policy.
Random Rollout corresponds to sampling a state from a particle belief and simulating it forward using a random policy. The returns of this simulation are used as the initial leaf node value estimate.
A constant number (e.g. $10^6$) indicates a belief-independent static initial value estimate.
UCT solvers only require a single value estimate ($\hat{V}$), whereas AdaOPS requires lower and upper bounds on belief value, $L_0$ and $U_0$ respectively.


\vskip 0.2in


\printbibliography

@String { aaai        = {AAAI Conference on Artificial Intelligence (AAAI)} }

@String { aamas       = {International Conference on Autonomous Agents and Multiagent Systems (AAMAS)} }

@String { cdc         = {IEEE Conference on Decision and Control (CDC)} }

@String { ecml        = {European Conference on Machine Learning (ECML)} }

@String { icaps       = {International Conference on Automated Planning and Scheduling (ICAPS)} }

@String { icra        = {IEEE International Conference on Robotics and Automation (ICRA)} }

@String { ieeeciaig   = {IEEE Transactions on Computational Intelligence and AI in Games} }

@String { ijcai       = {International Joint Conference on Artificial Intelligence (IJCAI)} }

@String { nips        = {Advances in Neural Information Processing Systems (NIPS)} }

@String { neurips     = {Advances in Neural Information Processing Systems (NeurIPS)} }

@String { uai         = {Conference on Uncertainty in Artificial Intelligence (UAI)} }

@inproceedings{sunberg2018pomcpow,
  title={Online Algorithms for {POMDP}s with Continuous State, Action, and Observation Spaces},
  author={Zachary Sunberg and Mykel J. Kochenderfer},
  booktitle=icaps,
  year={2018}
}

@inproceedings{garg2019despotalpha,
  title={{DESPOT}-alpha: Online {POMDP} Planning With Large State And Observation Spaces},
  author={Neha P. Garg and David Hsu and Wee Sun Lee},
  booktitle={Robotics: Science and Systems},
  year={2019}
}

@inproceedings{lim2020sparse,
  title={Sparse Tree Search Optimality Guarantees in {POMDP}s with Continuous Observation Spaces},
  author={Lim, Michael H. and Tomlin, Claire J. and Sunberg, Zachary N.},
  year={2020},
  booktitle=ijcai
}

@Article{kearns2002sparse,
    author="Kearns, Michael and Mansour, Yishay and Ng, Andrew Y.",
    title="A Sparse Sampling Algorithm for Near-Optimal Planning in Large {M}arkov Decision Processes",
    journal="Machine Learning",
    year="2002",
    month="11",
    day="01",
    volume="49",
    number="2",
    pages="193--208",
    issn="1573-0565",
}

@inproceedings{silver2010pomcp,
    title = {{M}onte-{C}arlo Planning in Large {POMDP}s},
    author = {Silver, David and Veness, Joel},
    booktitle = {Advances in Neural Information Processing Systems},
    pages = {2164--2172},
    year = {2010},
    url = {http://papers.nips.cc/paper/4031-monte-carlo-planning-in-large-pomdps.pdf}
}

@article{ye2017despot,
  title={{DESPOT}: Online {POMDP} Planning with Regularization},
  author={Ye, Nan and Somani, Adhiraj and Hsu, David and Lee, Wee Sun},
  journal={Journal of Artificial Intelligence Research},
  volume={58},
  pages={231--266},
  year={2017}
}

@inproceedings{seiler2015online,
  title={An online and approximate solver for {POMDP}s with continuous action space},
  author={Seiler, Konstantin M. and Kurniawati, Hanna and Singh, Surya P. N.},
  booktitle=icra,
  pages={2290--2297},
  year={2015},
}

@article{browne2012survey,
  title={A survey of {M}onte {C}arlo tree search methods},
  author={Browne, Cameron B. and Powley, Edward and Whitehouse, Daniel and Lucas, Simon M. and Cowling, Peter I. and Rohlfshagen, Philipp and Tavener, Stephen and Perez, Diego and Samothrakis, Spyridon and Colton, Simon},
  journal=ieeeciaig,
  volume={4},
  number={1},
  pages={1--43},
  year={2012},
}

@ARTICLE{sunberg2017value,
  author={Sunberg, Zachary and Kochenderfer, Mykel J.},
  journal={IEEE Transactions on Intelligent Transportation Systems}, 
  title={Improving Automated Driving Through {POMDP} Planning With Human Internal States}, 
  year={2022},
  volume={},
  number={},
  pages={1-11},
  doi={10.1109/TITS.2022.3182687}}

@inproceedings{couetoux2011double,
  address = {Rome, Italy},
  author = {Cou\"{e}toux, A. and Hoock, J.-B. and Sokolovska, N. and Teytaud, O. and Bonnard, N.},
  booktitle = {Learning and Intelligent Optimization},
  title = {Continuous Upper Confidence Trees},
  year = {2011}
}

@Article{holland2013optimizing,
  Title                    = {Optimizing the Next Generation Collision Avoidance System for Safe, Suitable, and Acceptable Operational Performance},
  Author                   = {Jessica E. Holland and Mykel J. Kochenderfer and Wesley A. Olson},
  Journal                  = {{A}ir {T}raffic {C}ontrol {Q}uarterly},
  Year                     = {2013},
  Number                   = {3},
  Pages                    = {275-297},
  Volume                   = {21}
}

@article{ayer2012mammography,
  title={A {POMDP} approach to personalize mammography screening decisions},
  author={Ayer, Turgay and Alagoz, Oguzhan and Stout, Natasha K},
  journal={Operations Research},
  volume={60},
  number={5},
  pages={1019--1034},
  year={2012},
  publisher={INFORMS}
}

@article{papadimitriou1987complexity,
  title={The Complexity of {M}arkov Decision Processes},
  author={Papadimitriou, Christos H. and Tsitsiklis, John N.},
  journal={Mathematics of Operations Research},
  volume={12},
  number={3},
  pages={441--450},
  year={1987},
  publisher={{INFORMS}}
}

@article{egorov2017pomdps,
  author  = {Maxim Egorov and Zachary N. Sunberg and Edward Balaban and Tim A. Wheeler and Jayesh K. Gupta and Mykel J. Kochenderfer},
  title   = {{POMDP}s.jl: A Framework for Sequential Decision Making under Uncertainty},
  journal = {Journal of Machine Learning Research},
  year    = {2017},
  volume  = {18},
  number  = {26},
  pages   = {1-5},
  url     = {http://jmlr.org/papers/v18/16-300.html}
}

@book{kochenderfer2015decision,
  title =         {Decision Making Under Uncertainty: Theory and Application},
  publisher =     {{MIT} Press},
  year =          {2015},
  author =        {Kochenderfer, Mykel J}
}

@article{ross2008online,
  title={Online planning algorithms for {POMDP}s},
  author={Ross, St{\'e}phane and Pineau, Joelle and Paquet, S{\'e}bastien and Chaib-Draa, Brahim},
  journal={Journal of Artificial Intelligence Research},
  volume={32},
  pages={663--704},
  year={2008}
}

@inproceedings{mern2021bayesian,
  title={Bayesian Optimized {M}onte {C}arlo Planning}, 
  author={John Mern and Anil Yildiz and Zachary Sunberg and Tapan Mukerji and Mykel J. Kochenderfer},
  author+an={3=me},
  year={2021},
  booktitle=aaai
}

@article{kaelbling1998planning,
  title={Planning and Acting in Partially Observable Stochastic Domains},
  author={Kaelbling, Leslie Pack and Littman, Michael L and Cassandra, Anthony R},
  journal={Artificial intelligence},
  volume={101},
  number={1-2},
  pages={99--134},
  year={1998},
  publisher={Elsevier}
}

@book{thrun2005probabilistic,
  title={Probabilistic Robotics},
  author={Thrun, Sebastian and Burgard, Wolfram and Fox, Dieter},
  year={2005},
  publisher={MIT Press}
}

@article{frew2020field,
  title={Field observation of tornadic supercells by multiple autonomous fixed-wing unmanned aircraft},
  author={Frew, Eric W and Argrow, Brian and Borenstein, Steve and Swenson, Sara and Hirst, C Alexander and Havenga, Henno and Houston, Adam},
  journal={Journal of Field Robotics},
  volume={37},
  number={6},
  pages={1077--1093},
  year={2020},
  publisher={Wiley Online Library}
}

@article{memarzadeh2018adaptive,
  title={Adaptive Management of Ecological Systems under Partial Observability},
  author={Memarzadeh, Milad and Boettiger, Carl},
  journal={Biological Conservation},
  volume={224},
  pages={9--15},
  year={2018},
  publisher={Elsevier}
}

@article{wu2021adaptive,
  title={Adaptive Online Packing-guided Search for {POMDP}s},
  author={Wu, Chenyang and Yang, Guoyu and Zhang, Zongzhang and Yu, Yang and Li, Dong and Liu, Wulong and Hao, Jianye},
  journal=neurips,
  volume={34},
  year={2021}
}

@article{smallwood1973optimal,
  title={The optimal control of partially observable {M}arkov processes over a finite horizon},
  author={Smallwood, Richard D and Sondik, Edward J},
  journal={Operations Research},
  volume={21},
  number={5},
  pages={1071--1088},
  year={1973},
  publisher={INFORMS}
}

@article{shani2013survey,
  title={A survey of point-based {POMDP} solvers},
  author={Shani, Guy and Pineau, Joelle and Kaplow, Robert},
  journal=aamas,
  volume={27},
  number={1},
  pages={1--51},
  year={2013},
  publisher={Springer}
}

@inproceedings{deglurkar2023compositional,
    author={Sampada Deglurkar and Michael H. Lim and Johnathan Tucker and Zachary N. Sunberg and Aleksandra Faust and Claire J. Tomlin},
    author+an={4=zach;2=student;3=student},
    title={Compositional Learning-based Planning for Vision {POMDP}s},
    booktitle={Learning for Dynamics \& Control (L4DC)},
    year={2023},
    url={https://arxiv.org/abs/2112.09456}
}

@inproceedings{bresina2002planning,
  title={Planning under Continuous Time and Resource Uncertainty: a challenge for {AI}},
  author={Bresina, J. and Dearden, Richard and Meuleau, N. and Ramakrishnan, S. and Smith, D. and Washington, R. and Darwiche, A. and Friedman, N.},
  booktitle=uai,
  year={2002}
}

@inproceedings{Metelli2018,
    title = {Policy Optimization via Importance Sampling},
    author = {Metelli, Alberto Maria and Papini, Matteo and Faccio, Francesco and Restelli, Marcello},
    booktitle = nips,
    pages = {5442--5454},
    year = {2018},
}

@Article{Singh1994,
    author="Singh, Satinder P. and Yee, Richard C.",
    title="An upper bound on the loss from approximate optimal-value functions",
    journal="Machine Learning",
    year="1994",
    day="01",
    volume="16",
    number="3",
    pages="227--233",
    issn="1573-0565",
}

@incollection{kurniawati2016online,
  title={An online {POMDP} solver for uncertainty planning in dynamic environment},
  author={Kurniawati, Hanna and Yadav, Vinay},
  booktitle={Robotics Research},
  pages={611--629},
  year={2016},
  publisher={Springer}
}

@inproceedings{bjarnason2009lower,
  title={Lower bounding {K}londike solitaire with {M}onte-{C}arlo planning},
  author={Bjarnason, Ronald and Fern, Alan and Tadepalli, Prasad},
  booktitle=icaps,
  year={2009}
}

@article{bai2014integrated,
  title={Integrated perception and planning in the continuous space: A {POMDP} approach},
  author={Bai, Haoyu and Hsu, David and Lee, Wee Sun},
  journal={International Journal of Robotics Research},
  volume={33},
  number={9},
  pages={1288--1302},
  year={2014},
  address={London, England}
}

@inproceedings{kocsis2006bandit,
  title={Bandit based {M}onte-{C}arlo planning},
  author={Kocsis, Levente and Szepesv{\'a}ri, Csaba},
  booktitle=ecml,
  pages={282--293},
  year={2006},
  organization={Springer}
}

@inproceedings{lim2021voronoicdc,
      title={Voronoi Progressive Widening: Efficient Online Solvers for Continuous State, Action, and Observation {POMDP}s}, 
      author={Michael H. Lim and Claire J. Tomlin and Zachary N. Sunberg},
      year={2021},
      booktitle=cdc,
}

@inproceedings{hoerger2020,
      title={An On-Line {POMDP} Solver for Continuous Observation Spaces}, 
      author={Marcus Hoerger and Hanna Kurniawati},
      year={2021},
      booktitle=icra
}

@article{shah2020nonasymptotic,
author = {Shah, Devavrat and Xie, Qiaomin and Xu, Zhi},
title = {Nonasymptotic Analysis of {M}onte {C}arlo Tree Search},
journal = {Operations Research},
volume = {0},
number = {0},
year = {2022},
doi = {10.1287/opre.2021.2239},

URL = { 
        https://doi.org/10.1287/opre.2021.2239
    
},
eprint = { 
        https://doi.org/10.1287/opre.2021.2239
    
}
}

@inproceedings{du2021particle,
  title={When is particle filtering efficient for planning in partially observed linear dynamical systems?},
  author={Du, Simon S and Hu, Wei and Li, Zhiyuan and Shen, Ruoqi and Song, Zhao and Wu, Jiajun},
  booktitle={Uncertainty in Artificial Intelligence},
  pages={728--737},
  year={2021},
  organization={PMLR}
}

@article{luo2019importance,
author = {Luo, Yuanfu and Bai, Haoyu and Hsu, David and Lee, Wee Sun},
doi = {10.1177/0278364918780322},
isbn = {0278364918},
issn = {17413176},
journal = {International Journal of Robotics Research},
number = {2-3},
pages = {162--181},
title = {{Importance sampling for online planning under uncertainty}},
volume = {38},
year = {2019}
}

@article{crisan2002survey,
  title={A survey of convergence results on particle filtering methods for practitioners},
  author={Crisan, Dan and Doucet, Arnaud},
  journal={{IEEE} Transactions on signal processing},
  volume={50},
  number={3},
  pages={736--746},
  year={2002},
  publisher={IEEE}
}

\end{document}